\documentclass[journal]{IEEEtran}
\usepackage{comment}
\usepackage{cite}
\usepackage{amsmath,amssymb,amsfonts,mathrsfs}
\usepackage{algorithm,algorithmic,setspace}
\usepackage{graphicx}
\usepackage{textcomp}
\usepackage{xcolor}
\usepackage{comment}
\usepackage{mathtools}
\usepackage{dsfont}
\usepackage{subfig}
\usepackage{algorithm}
\usepackage{algorithmic}
\usepackage{tikz}
\usepackage{epstopdf} 
\usepackage{graphics}
\usepackage{subfig}
\usepackage{epstopdf} 
\usepackage{url}
\usepackage{color,soul}
\usepackage{amsthm}
\usepackage[colorlinks=true,linkcolor=blue, citecolor=green, urlcolor=magenta]{hyperref}

\DeclarePairedDelimiter{\norm}{\lVert}{\rVert}

\theoremstyle{definition}

\newtheorem{theorem}{Theorem}
\newtheorem{assumption}{Assumption}

\newtheorem{lemma}{Lemma}
\newtheorem{remark}{Remark}

\DeclareMathOperator{\dom}{dom}

\allowdisplaybreaks

\title{
 N-dimensional Convex Obstacle Avoidance using Hybrid Feedback Control (Extended version)
} 
\author{Mayur Sawant, Ilia Polushin and Abdelhamid Tayebi \IEEEmembership{Fellow, IEEE} 
	\thanks{This work was supported by the National Sciences and Engineering Research Council of Canada (NSERC), under the grants RGPIN-2020-06270, RGPIN-2020-0644 and RGPIN-2020-04759. }
	\thanks{M. Sawant and A. Tayebi are with the Department of Electrical and Computer Engineering, Lakehead University, Thunder Bay, ON P7B 5E1, Canada. (e-mail: {\tt\small atayebi, msawant@lakeheadu.ca}). I. Polushin and A. Tayebi are with the Department of Electrical and Computer Engineering, Western University, London, ON N6A 3K7, Canada. (e-mail: {\tt\small ipolushi, atayebi@uwo.ca}).}%
}%

\begin{document}
\maketitle 

\begin{abstract}
This paper addresses the autonomous robot navigation problem in \textit{a priori} unknown $n$-dimensional environments containing disjoint convex obstacles of arbitrary shapes and sizes, with pairwise distances strictly greater than the robot's diameter. 
We propose a hybrid feedback control scheme that guarantees safe and global asymptotic convergence of the robot to a predefined target location. The proposed control strategy relies on a switching mechanism allowing the robot to operate either in the \textit{move-to-target} mode or the \textit{obstacle-avoidance} mode, based on its proximity to the obstacles and the availability of a clear straight path between the robot and the target. 
In the \textit{obstacle-avoidance} mode, the robot is constrained to move within a two-dimensional plane that intersects the obstacle being avoided and the target, preventing it from retracing its path. 
The effectiveness of the proposed hybrid feedback controller is demonstrated through simulations in two-dimensional and three-dimensional environments.
\end{abstract}

\section{Introduction}
Safe autonomous robot navigation consists in steering a robot to a target location while avoiding obstacles. One commonly used navigation technique is the Artificial Potential Field (APF) approach \cite{khatib1986real}, where a combination of attractive and repulsive vector fields guides the robot safely to the target location. However, this approach faces challenges with certain obstacle arrangements, leading to undesired stable local minima. 
The Navigation Function (NF) approach \cite{koditschek1992exact,verginis2021adaptive} is effective in sphere world environments, addressing the local minima issue by limiting the repulsive field's influence around the obstacles by means of a properly tuned parameter. However, this method ensures almost\footnote{Almost global convergence refers to the convergence from all initial conditions except a set of zero Lebesgue measure.} global convergence of the robot to the target location. 
To apply the NF approach to environments with general convex and star-shaped obstacles, one has to use diffeomorphic mappings from \cite{koditschek1992exact} and \cite{li2018navigation}, which require global knowledge of the environment for implementation. 
In \cite{paternain2017navigation}, the authors extended the NF approach to handle environments with convex obstacles with smooth boundaries that meet certain curvature conditions. 
However, this approach is limited to obstacles, which are not too flat and not too close to the target. In \cite{kumar2022navigation}, for the case of ellipsoidal worlds, the authors removed the flatness limitation in \cite{paternain2017navigation}, by providing a controller design, which locally transforms the region near the obstacle into a spherical region by using the Hessian information. However, similar to \cite{paternain2017navigation}, it is assumed that the entire shape of the obstacle becomes known when the robot visits its neighbourhood. 

Other approaches, such as \cite{ames2016control, wang2017safety, singletary2021comparative}, rely on the control Lyapunov function (CLF) and the control barrier function (CBF) to design feedback control laws that achieve (simultaneously) convergence to a target set and avoidance of an unsafe set. In \cite{wang2017safety}, the authors proposed a CBF-based method for multi-robot navigation in two-dimensional environments in the presence of circular obstacles. In \cite{singletary2021comparative}, a comparative analysis between the APF-based and CBF-based approaches has been provided.
However, the work in \cite{reis2020control} and \cite{TAN2024111359} demonstrates that the CLF-CBF-based navigation approach, similar to the NF approach, suffers from the undesired equilibria problem and provides (at best) almost global asymptotic stability guarantees in sphere worlds.

In \cite{berkane2021navigation}, the authors proposed a feedback controller based on Nagumo's theorem \cite[Theorem 4.7]{blanchini2008set} for autonomous navigation in environments with general convex obstacles. The forward invariance of the obstacle-free space is achieved by projecting the \textit{ideal} velocity control vector, which points towards the target, onto the tangent cone at the obstacle boundary whenever it points towards the obstacle. This approach was extended in \cite{smaili2024real} to guarantee almost global asymptotic stability of the target location in \textit{a priori} unknown environments containing strongly convex obstacles.

In \cite{arslan2019sensor}, the authors proposed a purely reactive autonomous navigation approach based on separating hyperplanes for robots operating in environments cluttered with unknown but sufficiently separated convex obstacles with smooth boundaries which satisfy curvature conditions similar to \cite{smaili2024real}. 
This approach was extended in \cite{vasilopoulos2018reactive} to address partially known environments with non-convex obstacles, where the robot has geometric information about the non-convex obstacles but lacks precise knowledge of their locations in the workspace.


The approaches discussed above provide, at best, almost global convergence guarantees due to the undesired equilibria that are generated when using continuous time-invariant vector fields \cite{koditschek1990robot}. This can be resolved by introducing discontinuities in the control, as shown in \cite{sanfelice2006robust, casau2019hybrid, poveda2021robust, braun2020explicit, berkane2021obstacle, matveev2011method, loizou2003closed}. 

In \cite{sanfelice2006robust} and \cite{casau2019hybrid}, hybrid control methods are employed to achieve robust global asymptotic stabilization in $\mathbb{R}^2$ for robots navigating towards a target location while avoiding collision with a single spherical obstacle. The approach in \cite{sanfelice2006robust} has been extended in \cite{poveda2021robust} to steer a group of planar robots in formation toward the source of an unknown but measurable signal while avoiding a single obstacle.
In \cite{braun2020explicit}, the authors proposed a hybrid control law to globally asymptotically stabilize a class of linear systems with drift while avoiding neighbourhoods of unsafe isolated points. In \cite{berkane2021obstacle}, hybrid control techniques were employed to achieve global stabilization of target locations in $n$-dimensional environments with sufficiently separated ellipsoidal obstacles. 

In \cite{matveev2011method}, the authors proposed a discontinuous feedback-based autonomous navigation scheme for nonholonomic robots operating in two-dimensional environments with non-convex obstacles, subject to restrictions on inter-obstacle arrangements. 
In \cite{loizou2003closed}, a discontinuous feedback control law was introduced for autonomous robot navigation in partially known two-dimensional environments. When encountering a known obstacle, the control vector aligns with the negative gradient of the navigation function, whereas near unknown obstacles, the robot follows the boundary, using local curvature information. 

In this paper, we proposed a hybrid feedback-based solution for autonomous navigation in $n$-dimensional environments with \textit{a priori} unknown convex obstacles.
The main contributions of the proposed research work are as follows:
\begin{enumerate}
   \item \textit{Global asymptotic stability:} The proposed autonomous navigation solution guarantees global asymptotic stabilization of the target location in unknown environments with convex obstacles of arbitrary shapes. Note that the few existing results in the literature achieving such strong stability results are of a hybrid type and are restricted to environments with ellipsoidal obstacles \cite{berkane2021obstacle}.
    

    \item \textit{$n$-dimensional convex obstacles:} The proposed hybrid feedback controller is applicable to $n$-dimensional workspaces containing convex obstacles of arbitrary shapes and sizes. In contrast, the autonomous navigation schemes in \cite{matveev2011method} and \cite{loizou2003closed} are limited to two-dimensional environments, while the methods in \cite{kumar2022navigation} and \cite{berkane2021navigation} apply only to $n$-dimensional environments with ellipsoidal obstacles. The navigation approaches in \cite{smaili2024real} and \cite{arslan2019sensor} are restricted to environments with strongly convex obstacles.

    \item \textit{Arbitrary interobstacle arrangements:} There are no restrictions on the arrangement of obstacles, unlike those imposed in \cite[Assumption 10]{matveev2011method} and \cite[Theorem 2]{berkane2021obstacle}, except for the widely accepted mild condition in Assumption \ref{3d_obstacle_separation}, which states that the robot can pass between any two obstacles while maintaining a positive distance.

    \item \textit{Applicable in \textit{a priori} unknown environments:} Unlike the approaches in \cite{paternain2017navigation} and \cite{kumar2022navigation}, which require global information about the obstacles, the proposed autonomous navigation method relies solely on range scanners and does not require \textit{a priori} global knowledge of the environment (sensor-based technique). 
\end{enumerate}

Compared to our earlier works \cite{sawant2023convex, sawant_nonconvex}, which are limited to two-dimensional settings, the novelty of the present work lies in its applicability to $n$-dimensional environments with arbitrarily shaped convex obstacles.
While the theoretical developments in \cite{sawant2023convex, sawant_nonconvex} assume complete knowledge of the obstacle geometries in 2D environments, the proposed approach in the present paper is designed from the outset to operate in \textit{a priori} unknown $n$-dimensional environments with arbitrarily-shaped convex obstacles.

The remainder of the paper is organized as follows. Section \ref{sec:preliminaries} introduces the notations and preliminaries used throughout the paper. The problem formulation is presented in Section \ref{section:problem_formulation}, followed by the proposed hybrid control algorithm in Section \ref{sec:hybrid_controller_design}. Stability and safety guarantees of the navigation control scheme are discussed in Section \ref{sec:stability}. 
In Section \ref{sec:sphere_world}, the hybrid feedback control law is modified to ensure a monotonic decrease in distance to the target in sphere worlds.
Section \ref{implementation_procedure} outlines the implementation procedure for the obstacle avoidance algorithm in \textit{a priori} unknown environments. Simulation results are provided in Section \ref{section:simulation} to demonstrate the algorithm’s effectiveness.
Section \ref{section:experimental_validation} presents the experimental validation of the proposed approach using the TurtleBot 4 Standard mobile robot, and concluding remarks are given in Section \ref{sec:conclusion}.

\section{Notations and Preliminaries}\label{sec:preliminaries}
\subsection{Notations}
The sets of real and natural numbers are denoted by $\mathbb{R}$ and $\mathbb{N}$, respectively. We identify vectors using bold lowercase letters. The Euclidean norm of a vector $\mathbf{p}\in\mathbb{R}^n$ is denoted by $\norm{\mathbf{p}}$, and an Euclidean ball of radius $r\geq0$ centered at $\mathbf{p}$ is represented by $\mathcal{B}_{r}(\mathbf{p}) = \{\mathbf{q}\in\mathbb{R}^n|\norm{\mathbf{q} - \mathbf{p}} \leq r\}.$ 
The set of $n-$dimensional unit vectors is given by $\mathbb{S}^{n-1} = \{\mathbf{p}\in\mathbb{R}^n|\norm{\mathbf{p}} = 1\}.$ The identity matrix of order $n$ is denoted by $\mathbf{I}_n$.
The $n-$dimensional \textit{Special Orthogonal group} is denoted by $SO(n):= \{\mathbf{R}\in\mathbb{R}^{n\times n}:\mathbf{R}^\top\mathbf{R} = \mathbf{I}_n, \det(\mathbf{R}) = 1\}.$  

For two sets $\mathcal{A}, \mathcal{B}\subset\mathbb{R}^n$, the relative complement of $\mathcal{B}$ with respect to $\mathcal{A}$ is denoted by $\mathcal{A}\setminus\mathcal{B} =\{\mathbf{a}\in\mathcal{A}|\mathbf{a}\notin \mathcal{B}\}$. The symbols $\partial\mathcal{A}, \mathcal{A}^{\circ}$, $\mathcal{A}^c$ and $\bar{\mathcal{A}}$ represent the boundary, interior, complement and the closure of the set $\mathcal{A}$, respectively, where $\partial\mathcal{A} = \bar{\mathcal{A}}\backslash\mathcal{A}^{\circ}$. The cardinality of a set $\mathcal{A}$ is denoted by $\mathbf{card}(\mathcal{A})$. 
The Minkowski sum of the sets $\mathcal{A}$ and $\mathcal{B}$ is denoted by $\mathcal{A} \oplus\mathcal{B} = \{\mathbf{a} + \mathbf{b}|\mathbf{a}\in\mathcal{A}, \mathbf{b}\in\mathcal{B}\}$. The dilated version of a set $\mathcal{A}\subset\mathbb{R}^n$ with $r \geq 0$ is represented by $\mathcal{D}_r(\mathcal{A}) = \mathcal{A} \oplus\mathcal{B}_r(\mathbf{0})$. The $r-$neighbourhood of a set $\mathcal{A}$ is denoted by $\mathcal{N}_r(\mathcal{A}) = \mathcal{D}_r(\mathcal{A})\setminus\mathcal{A}^\circ$, where $r$ is a strictly positive scalar.

\subsection{Projection on a set}\label{section:metric_projection} Given a closed set $\mathcal{A}\subset\mathbb{R}^n$ and a point $\mathbf{x}\in\mathbb{R}^n$, the Euclidean distance of $\mathbf{x}$ from the set $\mathcal{A}$ is evaluated as
\begin{equation}
    d(\mathbf{x}, \mathcal{A}) = \underset{\mathbf{q}\in\mathcal{A}}{\min}\norm{\mathbf{x} - \mathbf{q}}.\label{definition:distance_from_set}
\end{equation}
If $\mathcal{A}$ is convex, the unique closest point to any $\mathbf{x} \in \mathbb{R}^n$ on $\mathcal{A}$ is denoted by $\Pi(\mathbf{x}, \mathcal{A})$ and is defined as
\begin{equation}
\Pi(\mathbf{x}, \mathcal{O}_i) := \underset{\mathbf{q}\in\mathcal{A}}{\arg\min}\norm{\mathbf{x} - \mathbf{q}}.
\end{equation}


\subsection{Geometric subsets of \texorpdfstring{$\mathbb{R}^n$}{}}\label{section:geometric_subset}
\subsubsection{Line} The line passing through two points $\mathbf{p}\in\mathbb{R}^n$ and $\mathbf{q}\in\mathbb{R}^n\setminus\{\mathbf{p}\}$ is given by
\begin{equation}
    \mathcal{L}(\mathbf{p}, \mathbf{q}) := \{\mathbf{x}\in\mathbb{R}^n|\mathbf{x} = \lambda \mathbf{p} + (1 - \lambda) \mathbf{q}, \lambda \in\mathbb{R}\}.\label{line}
\end{equation}
\subsubsection{Line segment} The line segment joining two points $\mathbf{p}\in\mathbb{R}^n$ and $\mathbf{q}\in\mathbb{R}^n\setminus\{\mathbf{p}\}$ is given by
\begin{equation}
    \mathcal{L}_s(\mathbf{p}, \mathbf{q}) := \{\mathbf{x}\in\mathbb{R}^n|\mathbf{x} = \lambda \mathbf{p} + (1 - \lambda) \mathbf{q}, \lambda \in[0, 1]\}.\label{line_segment}
\end{equation}
\subsubsection{Plane} Given two linearly independent vectors $\mathbf{p}, \mathbf{q}\in\mathbb{R}^n$, the plane containing all vectors that can be expressed as a linear combination of $\mathbf{p}$ and $\mathbf{q}$ is given by
\begin{equation}
    \mathcal{P}(\mathbf{p}, \mathbf{q}):= \{\mathbf{x}\in\mathbb{R}^n|\exists k_1\in\mathbb{R}, \exists k_2\in\mathbb{R}, \mathbf{x} = k_1\mathbf{p} + k_2\mathbf{q}\}
\end{equation}
\subsubsection{Hyperplane} The hyperplane passing through $\mathbf{p}\in\mathbb{R}^n$ and orthogonal to $\mathbf{q}\in\mathbb{R}^n\backslash\{\mathbf{0}\}$ is given by
    \begin{align}
    \mathcal{H}(\mathbf{p}, \mathbf{q}) := \{\mathbf{x}\in\mathbb{R}^n| \mathbf{q}^\intercal(\mathbf{x} - \mathbf{p}) = 0\}.\label{equation_of_hyperplane}
    \end{align}
    The hyperplane divides the Euclidean space $\mathbb{R}^n$ into two half-spaces \textit{i.e.}, a closed positive half-space $\mathcal{H}_{\geq}(\mathbf{p}, \mathbf{q})$ and a closed negative half-space $\mathcal{H}_{\leq}(\mathbf{p}, \mathbf{q})$ which are obtained by substituting `$=$' with `$\geq$' and `$\leq$' respectively, in the right-hand side of \eqref{equation_of_hyperplane}. We also use the notations $\mathcal{H}_{>}(\mathbf{p}, \mathbf{q})$ and $\mathcal{H}_{<}(\mathbf{p} ,\mathbf{q})$ to denote the open positive and the open negative half-spaces such that $\mathcal{H}_{>}(\mathbf{p}, \mathbf{q}) = \mathcal{H}_{\geq}(\mathbf{p}, \mathbf{q})\backslash\mathcal{H}(\mathbf{p}, \mathbf{q})$ and $\mathcal{H}_{<}(\mathbf{p} ,\mathbf{q})= \mathcal{H}_{\leq}(\mathbf{p}, \mathbf{q})\backslash\mathcal{H}(\mathbf{p}, \mathbf{q})$.

\subsubsection{Cylinder}\label{definition:cylinder}
Given $\mathbf{p}\in\mathbb{R}^n$ and $w > 0$, the cylinder with width $w$ and the line segment $\mathcal{L}_s(\mathbf{p}, \mathbf{0})$ as its axis is denoted by
\begin{equation}
\mathcal{C}_{\mathcal{L}}(\mathbf{p}, w) := \mathcal{D}_{w}(\mathcal{L}_s(\mathbf{p}, \mathbf{0})) \cap \mathcal{H}_{\geq}(\mathbf{p}, -\mathbf{p}) \cap \mathcal{H}_{\geq}(\mathbf{0}, \mathbf{p}) \nonumber
\end{equation}

\subsection{Parallel projection operator}\label{definition_parallel_projection_operator}
Given two orthonormal vectors \(\mathbf{p}, \mathbf{q} \in \mathbb{S}^{n-1}\), the parallel projection operator is defined as
\[
\mathbf{P}(\mathbf{p}, \mathbf{q}) = \mathbf{p}\mathbf{p}^\top + \mathbf{q}\mathbf{q}^\top.
\]
For any vector \(\mathbf{s} \in \mathbb{R}^n\), the vector \(\mathbf{s}^\prime\), where \(\mathbf{s}^\prime = \mathbf{P}(\mathbf{p}, \mathbf{q})\mathbf{s}\), is the projection of \(\mathbf{s}\) onto the plane \(\mathcal{P}(\mathbf{p}, \mathbf{q})\). Additionally, it can be verified that \(\mathbf{s}^\top \mathbf{s}^\prime \geq 0\), indicating that the angle between \(\mathbf{s}\) and \(\mathbf{s}^\prime\) is either acute or $90^{\circ}$.

\subsection{n-Dimensional rotation matrix}
\label{n-D_rotation_matrix}
Consider two orthonormal vectors \(\mathbf{p}, \mathbf{q} \in \mathbb{S}^{n-1}\) and an angle \(\theta \in [0, 2\pi)\). The rotation matrix \(R(\theta, \mathbf{p}, \mathbf{q})\) is constructed as follows:
\begin{equation}
    R(\theta, \mathbf{p}, \mathbf{q}) := \mathbf{I}_{n} + \sin(\theta)\mathbf{Sk}(\mathbf{p}, \mathbf{q}) + (1 - \cos(\theta))(\mathbf{Sk}(\mathbf{p}, \mathbf{q}))^2,\nonumber
\end{equation}
where the skew-symmetric matrix \(\mathbf{Sk}(\mathbf{p}, \mathbf{q})\) is defined by
\[
\mathbf{Sk}(\mathbf{p}, \mathbf{q}) = \mathbf{q}\mathbf{p}^\top - \mathbf{p}\mathbf{q}^\top.
\]
The operator \(R(\theta, \mathbf{p}, \mathbf{q})\) performs a rotation by the angle \(\theta\) in the plane spanned by \(\mathbf{p}\) and \(\mathbf{q}\). 
Rotation angles are considered positive if they are performed from the vector $\mathbf{p}$ to the vector $\mathbf{q}$.


\subsection{Hybrid system framework}\label{section:hybrid_system}
A hybrid dynamical system
\cite{goebel2012hybrid} is represented using differential and difference inclusions for the state $\mathbf{\xi}\in\mathbb{R}^n$ as follows:
\begin{align}
    \begin{cases}\begin{matrix}\mathbf{\dot{\xi}} \in \mathbf{F}(\mathbf{\xi}) , & \mathbf{\xi} \in \mathcal{F}, \\
    \mathbf{\xi}^{+}\in \mathbf{J}(\mathbf{\xi}), & \mathbf{\xi}\in\mathcal{J},\end{matrix}\end{cases}\label{hybrid_system_general_model}
\end{align}
where the \textit{flow map} $\mathbf{F}:\mathbb{R}^n\rightrightarrows\mathbb{R}^n$ is the differential inclusion which governs the continuous evolution when $\mathbf{\xi}$ belongs to the \textit{flow set} $\mathcal{F}\subseteq\mathbb{R}^n$, where the symbol `$\rightrightarrows$' represents set-valued mapping. The \textit{jump map} $\mathbf{J}:\mathbb{R}^n\rightrightarrows\mathbb{R}^n$ is the difference inclusion that governs the discrete evolution when $\mathbf{\xi}$ belongs to the \textit{jump set} $\mathcal{J}\subseteq\mathbb{R}^n$. The vector $\xi^+$ represents the state of the hybrid system after a jump. The hybrid system \eqref{hybrid_system_general_model} is defined by its data and is denoted as $\mathcal{H}_{\mathcal{S}} = (\mathcal{F}, \mathbf{F}, \mathcal{J}, \mathbf{J}).$

A subset $\mathbb{T}\subset\mathbb{R}_{\geq}\times\mathbb{N}$ is a \textit{hybrid time domain} if it is a union of a finite or infinite sequence of intervals $[t_j, t_{j + 1}]\times \{j\},$ where the last interval (if existent) is possibly of the form $[t_j, T)$ with $T$ finite or $T = +\infty$. The ordering of points on each hybrid time domain is such that $(t, j)\preceq(t^{\prime}, j^{\prime})$ if $t \leq t^{\prime}$ and $j \leq j^{\prime}$. A \textit{hybrid solution} $\phi$ is maximal if it cannot be extended, and complete if its domain dom $\phi$ (which is a hybrid time domain) is unbounded.

\section{Problem Formulation}
\label{section:problem_formulation}

Let $\mathcal{W}$ be a closed subset of the $n-$dimensional Euclidean space that bounds the workspace. The workspace $\mathcal{W}$ consists of finite number of compact, convex obstacles $\mathcal{O}_i$, $i\in\{1, \ldots, b\}, b\in\mathbb{N}$. We define obstacle $\mathcal{O}_0:=(\mathcal{W}^{\circ})^c$ as the complement of the interior of the workspace. Collectively, the obstacle-occupied workspace is denoted by $\mathcal{O}_{\mathcal{W}} = \bigcup_{i\in\mathbb{I}}\mathcal{O}_i$, where $\mathbb{I} = \{0, \ldots, b\}.$ 

The robot is represented by a $n-$dimensional sphere with radius $r \geq 0$ and center $\mathbf{x}$. 
It is equipped with a range sensor that can identify the set $\eth_{\mathcal{W}}^{\mathbf{x}}$ which contains the locations on the boundaries of nearby obstacles, provided there is a clear line of sight to the center of the robot, up to a certain sensing range $R_s$, as shown in Fig. \ref{visible_boundary_diagram}. The set $\eth_{\mathcal{W}}$ is defined as follows:
\begin{equation}
    \eth_{\mathcal{W}}^{\mathbf{x}} := \{\mathbf{p}\in\partial\mathcal{O}_{\mathcal{W}}|\|\mathbf{p} - \mathbf{x}\|\leq R_s, \mathcal{L}_s(\mathbf{p}, \mathbf{x})\cap\mathcal{O}_{\mathcal{W}} = \mathbf{p}\}.\label{sensor_visible_boundary}
\end{equation}

\begin{figure}
\centering
    \includegraphics[width = 0.65\linewidth]{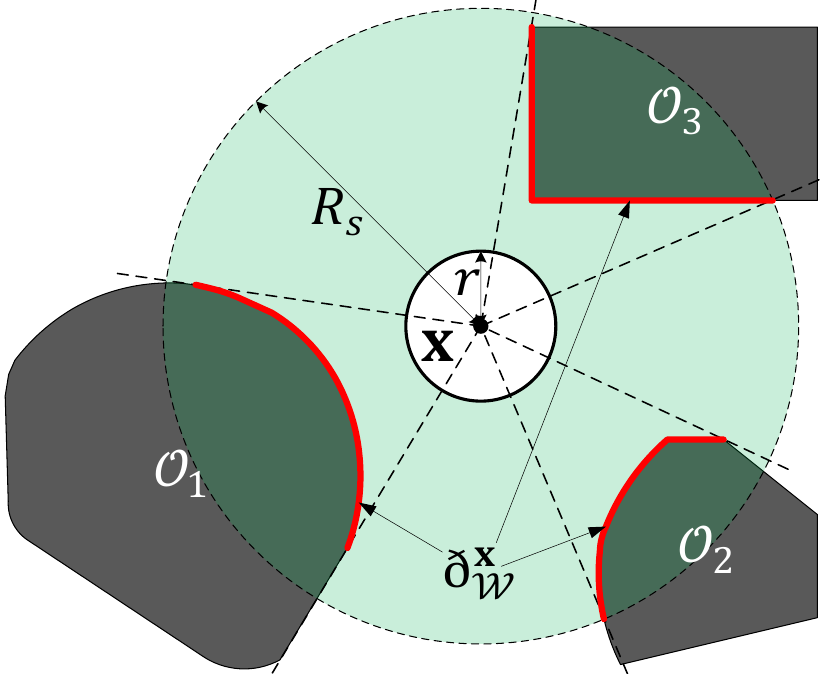}
    \caption{Illustration of the set $\eth_{\mathcal{W}}^{\mathbf{x}}$ as defined in \eqref{sensor_visible_boundary}.}
    \label{visible_boundary_diagram}
\end{figure}

To maintain the local convexity of the obstacle boundaries, we impose the following workspace feasibility assumption. 
\begin{assumption}
The minimum separation between any pair of obstacles should be greater than or equal to $2(r + \delta)$ \textit{i.e.}, for all $i, j\in\mathbb{I}, i\ne j,$ one has
\begin{equation}
    d(\mathcal{O}_i, \mathcal{O}_j):= \underset{\mathbf{p}\in\mathcal{O}_i, \mathbf{q}\in\mathcal{O}_j}{\text{min}} \|{\mathbf{p} - \mathbf{q}}\| \geq 2(r+\delta),
\end{equation}
where $\delta > 0$. \label{3d_obstacle_separation}
\end{assumption}

We then pick an arbitrarily small value $r_s\in(0, \delta)$ as the minimum distance that the robot should maintain to any obstacle. Therefore, the obstacle-free workspace with respect to the center of the robot is defined as follows:
\begin{equation}
\mathcal{W}_{r_a} := \mathcal{W}\setminus(\mathcal{D}_{r_a}(\mathcal{O}_{\mathcal{W}}))^{\circ},\label{r_a-eroded_obstacle-free_workspace}
\end{equation}
where $r_a = r + r_s$. According to Assumption \ref{3d_obstacle_separation}, the set $\mathcal{W}_{r_a}$ is pathwise connected and  $\mathbf{x}\in\mathcal{W}_{r_a}\iff\mathcal{B}_{r_a}(\mathbf{x})\subset\mathcal{W}\setminus\mathcal{O}_{\mathcal{W}}^{\circ}$.

The robot is governed by a single integrator dynamics 
\begin{equation}
    \dot{\mathbf{x}} = \mathbf{u},\label{first_order_system}
\end{equation}
where $\mathbf{u}\in\mathbb{R}^n$ is the control input. Given a workspace that satisfies Assumption \ref{3d_obstacle_separation}, and given that the robot can identify the set $\eth_{\mathcal{W}}^{\mathbf{x}}$ \eqref{sensor_visible_boundary}, the task is to design a feedback control law $\mathbf{u}$ to guarantee the following properties: 
\begin{enumerate}
    \item \textbf{Safety}: the obstacle-free workspace $\mathcal{W}_{r_a}$  with respect to the center of the robot is forward invariant,
    \item \textbf{Global asymptotic stability}: any target location $\mathbf{x^d}\in(\mathcal{W}_{r_a})^{\circ}$ is a globally asymptotically stable equilibrium for the closed-loop system. Without loss of generality, we will consider $\mathbf{x^d}=\mathbf{0}$.
\end{enumerate}

\begin{figure}
\centering
\includegraphics[width = 0.55\linewidth]{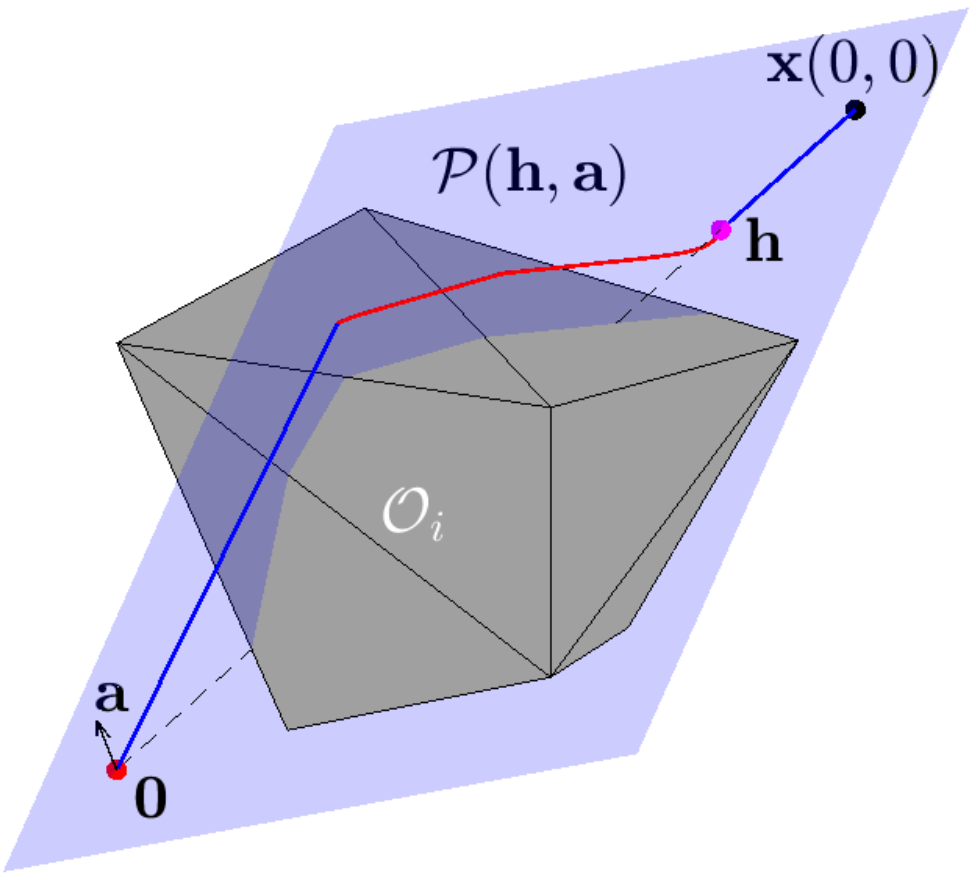}
\caption{Robot trajectory moving along the plane $\mathcal{P}(\mathbf{h}, \mathbf{a})$ while avoiding an obstacle.}
\label{plane_passing_through_obstacle}
\end{figure}

\section{Hybrid Control For Obstacle Avoidance}
\label{sec:hybrid_controller_design}
The proposed hybrid feedback controller operates in two different modes based on the mode indicator $m\in\mathbb{M}:=\{0, 1\}$. In the \textit{move-to-target} mode $(m = 0)$, the control input vector directs the robot along a straight line toward the target location. When the robot is in the neighbourhood of an obstacle that obstructs its direct path (\textit{i.e.,} the distance from the robot's centre to the closest point on the obstacle is less than or equal to $\gamma\in(0, \delta - r_s)$), the control law switches to the \textit{obstacle-avoidance} mode $(m = 1)$. 
In contrast to the two-dimensional case, where the robot can only move clockwise or counterclockwise around an obstacle, the $n$-dimensional environment offers an infinite number of safe paths around the obstacles. Thus, designing an obstacle-avoidance strategy that maintains the continuity of the control input vector and prevents the robot from retracing its path is crucial. The proposed strategy achieves this by confining the robot’s motion to a plane that passes through both the target location and the interior of the obstacle being avoided. The robot is guided along the obstacle’s boundary until it reaches a point where the obstacle no longer blocks the direct line-of-sight to the target location, as shown in Fig. \ref{plane_passing_through_obstacle}.


\subsection{Hybrid control design}\label{section:hybrid_control_design}
The proposed hybrid control $\mathbf{u}(\mathbf{x}, \mathbf{h}, \mathbf{a}, m, s)$ is given as
\begin{subequations}
\begin{align}
    &\mathbf{u}(\xi) = -\kappa_s(1 - m)\mathbf{x} + \kappa_r m\mathbf{v}(\mathbf{x}, \mathbf{h}, \mathbf{a}),\label{control_u}\\
    &{\begin{bmatrix}\mathbf{\dot{h}}\\\mathbf{\dot{a}}\\\dot{m}\\\dot{s}
    \end{bmatrix} =\begin{bmatrix}\mathbf{0}\\\mathbf{0}\\0\\1\end{bmatrix}},{\xi\in\mathcal{F}};\quad{\begin{bmatrix}
    \mathbf{h^+}\\\mathbf{a^+}\\m^+\\{s^+}
    \end{bmatrix} \begin{matrix}\;\in\mathbf{L}(\xi)\end{matrix}},{\xi\in\mathcal{J}}, \label{control_L}   
\end{align}\label{hybrid_control_law}
\end{subequations}
where $\kappa_s > 0$, $\kappa_r > 0$, and the composite state vector $\xi:=(\mathbf{x}, \mathbf{h}, \mathbf{a}, m, s)\in\mathcal{K}:=\mathcal{W}_{r_a}\times\mathcal{W}_{r_a}\times\mathbb{S}^{n-1}\times\mathbb{M}\times\mathbb{R}_{\geq 0}$
The variable $\mathbf{h}$ denotes the \textit{hit point}, which is the location where the robot switches from the \textit{move-to-target} mode to the \textit{obstacle-avoidance} mode. 
The unit vector $\mathbf{a}\in\mathbb{S}^{n-1}$ is updated to be orthogonal to $\mathbf{x}$ when the robot switches from the \textit{move-to-target} mode to the \textit{obstacle-avoidance} mode, and is instrumental for the construction of the avoidance control vector $\mathbf{v}(\mathbf{x}, \mathbf{h}, \mathbf{a})$, used in \eqref{control_u}.
The scalar variable $s\in\mathbb{R}_{\geq 0}$ allows the robot to switch once from the \textit{obstacle-avoidance} mode to the \textit{move-to-target} mode only when it is initialized in the \textit{obstacle-avoidance} mode. Details of this switching process are provided later in Section \ref{section:flow_jump_set_construction}.
The sets $\mathcal{F}$ and $\mathcal{J}$ are the flow and jump sets related to different modes of operation, respectively, whose constructions are provided in Section \ref{section:flow_jump_set_construction}. 
The update law $\mathbf{L}$, which allows the robot to update the values of the variables $\mathbf{h}$, $\mathbf{a}$, $m$ and $s$ based on the current location of the robot with respect to the obstacle being avoided and the target location, will be designed later in Section \ref{section:update_law}.  Next, we provide the design of the vector $\mathbf{v}(\mathbf{x}, \mathbf{h}, \mathbf{a})\in\mathbb{R}^n$.
The vector $\mathbf{v}(\mathbf{x}, \mathbf{h}, \mathbf{a})$, used in \eqref{control_u}, is defined as
\begin{equation}
    \mathbf{v}(\mathbf{x},\mathbf{h}, \mathbf{a}) = \left[\eta(\mathbf{x})\mathbf{I}_n + (1 - 
|\eta(\mathbf{x})
|)\mathbf{R}(\hat{\mathbf{h}}, \mathbf{a})\right]\mathbf{P}(\hat{\mathbf{h}}, \mathbf{a})\mathbf{x}_{\pi},\label{n_dimensional_obstacle-avoidance_vector}
\end{equation}
where $\hat{\mathbf{h}}$ denotes unit vector in the direction of $\mathbf{h}$, and $\mathbf{x}_{\pi} = \mathbf{x} - \Pi(\mathbf{x}, \mathcal{O}_{\mathcal{W}})$. The vector $\Pi(\mathbf{x}, \mathcal{O}_{\mathcal{W}})$ represents the point on the obstacle-occupied workspace closest to $\mathbf{x}$, as defined in Section \ref{section:metric_projection}. Notice that, since the obstacles $\mathcal{O}_i, i\in\mathbb{I}\setminus\{0\}$ are convex and the parameter $\gamma\in(0, \delta - r_s)$, according to Assumption \ref{3d_obstacle_separation}, the robot will have a unique closest point to the obstacles whenever its center is in the $(r_a + \gamma)-$neighbourhood of these obstacles.
On the other hand, since obstacle $\mathcal{O}_0$, where $\mathcal{O}_0= (\mathcal{W}^{\circ})^c$, is non-convex, there may be some locations in the $(r_a + \gamma)-$neighbourhood of the obstacle $\mathcal{O}_0$ for which the uniqueness of the closest point from the robot's center to the obstacle $\mathcal{O}_0$ cannot be guaranteed. 
However, since $\mathcal{W}$ is a convex subset of $\mathbb{R}^n$, the obstacle $\mathcal{O}_0$ does not obstruct the robot's straight-line path to the target at the origin as long as $\mathbf{x}\in\mathcal{N}_{\gamma}(\mathcal{D}_{r_a}(\mathcal{O}_0))$.
Consequently, as discussed later in Remark \ref{remark:not_executed_near_workspace_boundary}, the design of the flow sets and the jump sets guarantees that the obstacle-avoidance control vector $\mathbf{v}(\mathbf{x}, \mathbf{h}, \mathbf{a})$ is never activated in the region $\mathcal{N}_{r_a + \gamma}(\mathcal{O}_0).$

Note that the coordinates of the \textit{hit point} $\mathbf{h}$ and the unit vector $\mathbf{a}$ are updated when the robot switches from the \textit{move-to-target} mode to the \textit{obstacle-avoidance} mode using the update law $\mathbf{L}(\xi)$, whose design is provided later in Section \ref{section:update_law}. It is ensured that $\mathbf{a}^\top\mathbf{h} = 0$, which allows us to define the operators $\mathbf{P}$ and $\mathbf{R}$, used in \eqref{n_dimensional_obstacle-avoidance_vector}, as discussed next. 

Since $\mathbf{h}$ and $\mathbf{a}$ are ensured to be orthogonal to each other, the parallel projection operator $\mathbf{P}(\hat{\mathbf{h}}, \mathbf{a})$, as described in Section \ref{definition_parallel_projection_operator}, is defined as
\begin{equation}
    \mathbf{P}(\hat{\mathbf{h}}, \mathbf{a}):= \hat{\mathbf{h}}\hat{\mathbf{h}}^\top + \mathbf{a}\mathbf{a}^\top.\label{parallel_projection_operator}
\end{equation}
According to \eqref{parallel_projection_operator}, the vector $\mathbf{P}(\hat{\mathbf{h}}, \mathbf{a})\mathbf{x}_{\pi}$, used in \eqref{n_dimensional_obstacle-avoidance_vector},  is the projection of the vector $\mathbf{x}_{\pi}$ onto the plane $\mathcal{P}(\mathbf{h}, \mathbf{a})$. Next, we discuss the design of the operator $\mathbf{R}$.

The rotation matrix \(\mathbf{R}(\hat{\mathbf{h}}, \mathbf{a}) := R(\pi/2, \hat{\mathbf{h}}, \mathbf{a}) \in SO(n)\), where \(R(\theta, \hat{\mathbf{h}}, \mathbf{a})\) for \(\theta \in [0, 2\pi)\) is defined in Section \ref{n-D_rotation_matrix}. For \(\theta = \frac{\pi}{2}\), one gets
\begin{equation}
    \mathbf{R}(\hat{\mathbf{h}}, \mathbf{a}) = \mathbf{I}_n + (\mathbf{a}\hat{\mathbf{h}}^\top - \hat{\mathbf{h}}\mathbf{a}^\top) - (\hat{\mathbf{h}}\hat{\mathbf{h}}^\top + \mathbf{a}\mathbf{a}^\top).\label{rotational_operator}
\end{equation}
Note that, according to \eqref{parallel_projection_operator}, one has \(\mathbf{P}(\hat{\mathbf{h}}, \mathbf{a})\mathbf{x}_{\pi} \in \mathcal{P}(\mathbf{h}, \mathbf{a})\). Therefore, as per \eqref{rotational_operator}, the operator \(\mathbf{R}(\hat{\mathbf{h}}, \mathbf{a})\), used in \eqref{n_dimensional_obstacle-avoidance_vector}, rotates the vector \(\mathbf{P}(\hat{\mathbf{h}}, \mathbf{a})\mathbf{x}_{\pi}\) by \(\pi/2\) radians in the plane \(\mathcal{P}(\mathbf{h}, \mathbf{a})\). This rotation is performed from the vector \(\mathbf{h}\) to the vector \(\mathbf{a}\).

Finally, the scalar function $\eta(\mathbf{x}) \in[-1, 1]$ is given by
\begin{equation}
    \eta(\mathbf{x}) = \begin{cases} -1, & d(\mathbf{x}, \mathcal{O}_{\mathcal{W}}) - r_a \geq \gamma_s,\\
    1 -\frac{d(\mathbf{x},\mathcal{O}_{\mathcal{W}})- r_a - \gamma_a}{0.5(\gamma_s - \gamma_a)}, &\gamma_a< d(\mathbf{x}, \mathcal{O}_{\mathcal{W}}) - r_a < \gamma_s,\\
    1, & d(\mathbf{x}, \mathcal{O}_{\mathcal{W}}) - r_a\leq \gamma_a,
    \end{cases}\label{scalar_function_eta_definition}
\end{equation}
where $0<\gamma_a<\gamma_s<\gamma$.
The scalar function $\eta$ is designed to ensure that the center of the robot remains inside the $\gamma-$neighborhood of the $r_a-$dilated obstacle-occupied workspace $\mathcal{N}_{\gamma}(\mathcal{D}_{r_a}(\mathcal{O}_{\mathcal{W}}))$ when it operates in the \textit{obstacle-avoidance} mode in the set $\mathcal{N}_{\gamma}(\mathcal{D}_{r_a}(\mathcal{O}_{\mathcal{W}}))$. This feature allows for the design of the jump set of the \textit{obstacle-avoidance} mode, as discussed later in Section \ref{subsubsection_flowjump_obstacleavoidance}, and ensures convergence to the target location, as stated later in Theorem \ref{theorem:global_stability}.

\begin{remark}
    Consider a plane $\mathcal{P}(\mathbf{h}, \mathbf{a})$ passing through the interior of obstacle $\mathcal{O}_i$ and the origin, where $\mathbf{h}\in\mathcal{N}_{\gamma}(\mathcal{D}_{r_a}(\mathcal{O}_i))$ and $\mathbf{a}\in\mathbb{S}^{n-1}$. Since $\mathcal{O}_i$ is a general convex obstacle, one cannot guarantee that $\mathbf{x}_{\pi}\in\mathcal{P}(\mathbf{h}, \mathbf{a})$ for all $\mathbf{x}\in\mathcal{P}(\mathbf{h}, \mathbf{a})\cap\mathcal{N}_{\gamma}(\mathcal{D}_{r_a}(\mathcal{O}_i))$. Therefore, the operator $\mathbf{P}(\hat{\mathbf{h}}, \mathbf{a})$ is used in \eqref{n_dimensional_obstacle-avoidance_vector} to ensure that $\mathbf{v}(\mathbf{x}, \mathbf{h}, \mathbf{a})\in\mathcal{P}(\mathbf{h}, \mathbf{a})$ for all $\mathbf{x}\in\mathcal{P}(\mathbf{h}, \mathbf{a})\cap\mathcal{N}_{\gamma}(\mathcal{D}_{r_a}(\mathcal{O}_i))$. 
    The operator $\mathbf{R}(\hat{\mathbf{h}}, \mathbf{a})$ aids in steering the robot along the boundary of obstacle $\mathcal{O}_i$ when it operates in the \textit{obstacle-avoidance} mode.
    However, because $\mathcal{O}_i$ is a general convex obstacle, the vector $\mathbf{R}(\hat{\mathbf{h}}, \mathbf{a})\mathbf{P}(\hat{\mathbf{h}}, \mathbf{a})\mathbf{x}_{\pi}$ is not necessarily tangential to the set $\partial\mathcal{D}_{\beta}(\mathcal{O}_i)$ at $\mathbf{x}\in\mathcal{P}(\mathbf{h}, \mathbf{a})\cap\mathcal{N}_{\gamma}(\mathcal{D}_{r_a}(\mathcal{O}_i))$, where $\beta = d(\mathbf{x}, \mathcal{O}_i)$. 
    Therefore, it may drive the robot either closer to obstacle $\mathcal{O}_i$ or away from it. 
    Thus, the scalar function $\eta(\mathbf{x})$ in included to ensure that $\mathbf{x}$ remains in the set $\mathcal{P}(\mathbf{h}, \mathbf{a})\cap\mathcal{N}_{\gamma}(\mathcal{D}_{r_a}(\mathcal{O}_i))$ as long as the robot operates in \textit{obstacle-avoidance} mode.
\end{remark}

Next, we provide the construction of the flow set $\mathcal{F}$ and the jump set $\mathcal{J}$ used in \eqref{hybrid_control_law}.

\subsection{Geometric construction of the flow and jump sets}\label{section:flow_jump_set_construction}
When the robot is located at a distance larger than $\gamma$ from the obstacles, it moves straight towards the target location in the \textit{move-to-target} mode. The robot's distance from the nearby obstacles is obtained by evaluating $d(\mathbf{x}, \eth_{\mathcal{W}}^{\mathbf{x}})$, as discussed in Section \ref{section:metric_projection}, where the set $\eth_{\mathcal{W}}^{\mathbf{x}}$ is defined in \eqref{sensor_visible_boundary}. Upon entering in the $\gamma$-neighbourhood of the obstacles, as per Assumption \ref{3d_obstacle_separation}, the robot is within the $\gamma$-neighbourhood of only one obstacle. In other words, according to Assumption \ref{3d_obstacle_separation}, the fact that $d(\mathbf{x}, \eth_{\mathcal{W}}^{\mathbf{x}}) \leq r_a + \gamma$, implies the existence of $i\in\mathbb{I}$ such that $\mathbf{x}\in\mathcal{N}_{\gamma}(\mathcal{D}_{r_a}(\mathcal{O}_i))$ and $\mathbf{x}\notin\mathcal{N}_{\gamma}(\mathcal{D}_{r_a}(\mathcal{O}_j))$ for all $j\in\mathbb{I}\setminus{i}$.

\begin{figure}
\centering
\includegraphics[width = 0.8\linewidth]{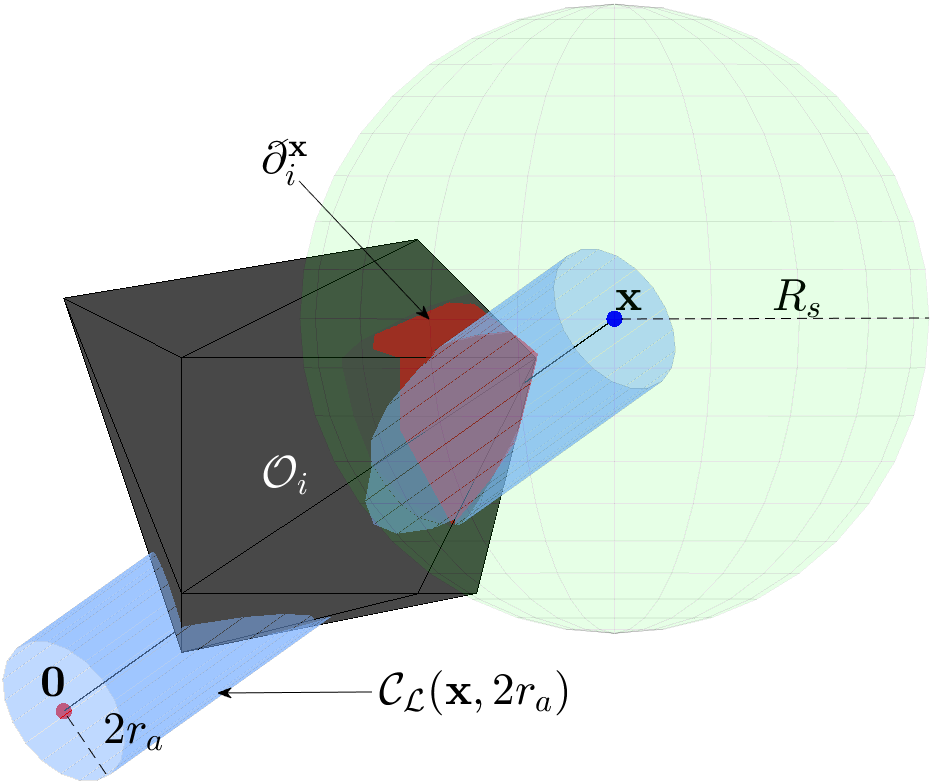}
\caption{A workspace scenario in which condition \eqref{sufficient_condition_to_check} is satisfied.}
\label{cylinder_intersection_condition_diagram}
\end{figure}

\begin{figure*}
\centering
\subfloat[][]{\includegraphics[width =0.25\linewidth]{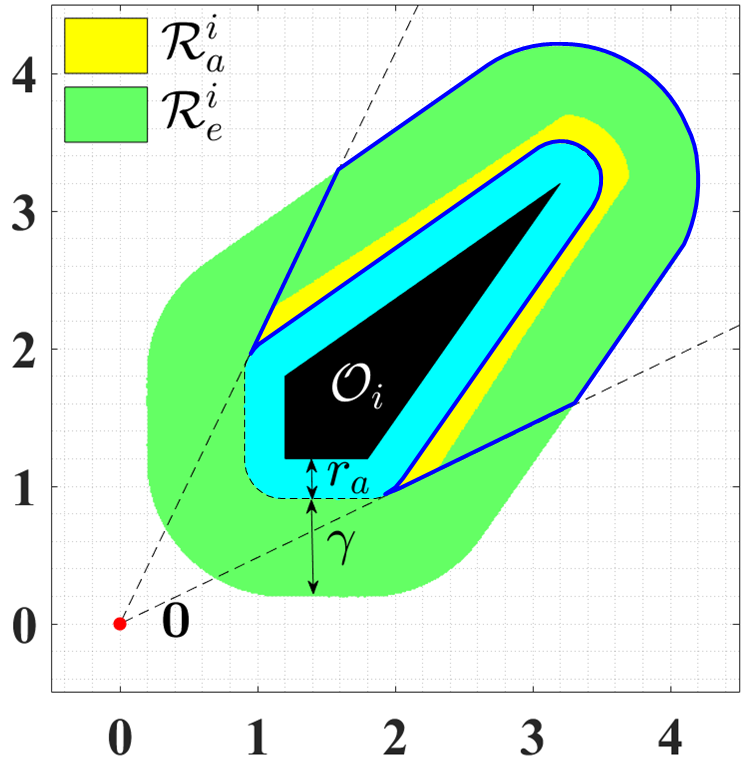}\label{Rs050}}
\subfloat[][]{\includegraphics[width =0.25\linewidth]{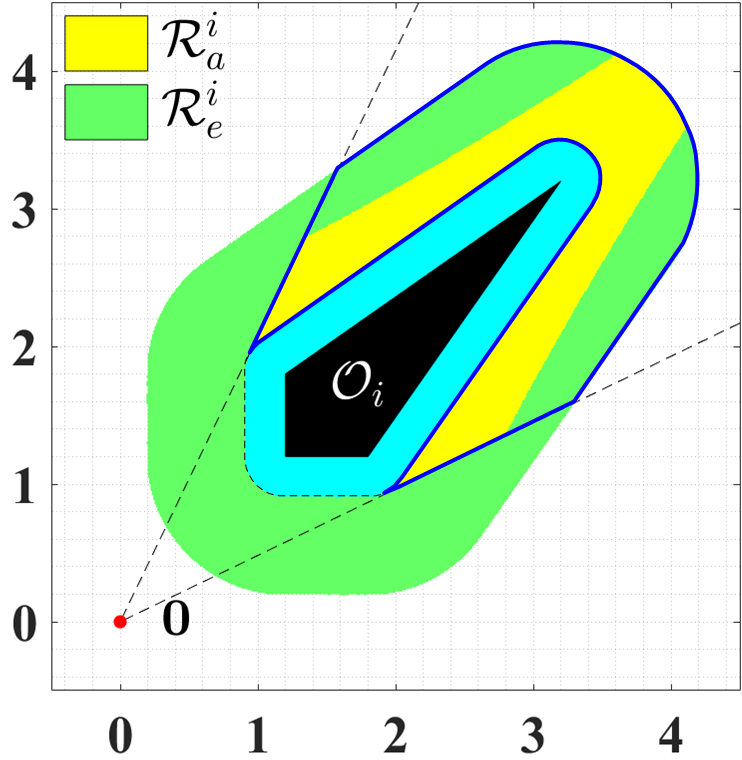}\label{Rs100}}
\subfloat[][]{\includegraphics[width =0.25\linewidth]{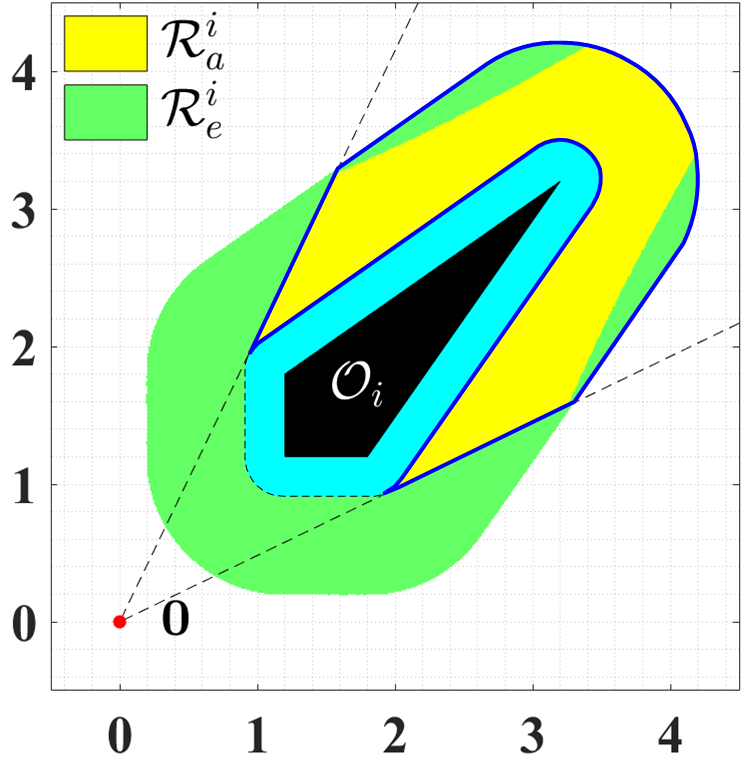}\label{Rs150}}
\subfloat[][]{\includegraphics[width =0.25\linewidth]{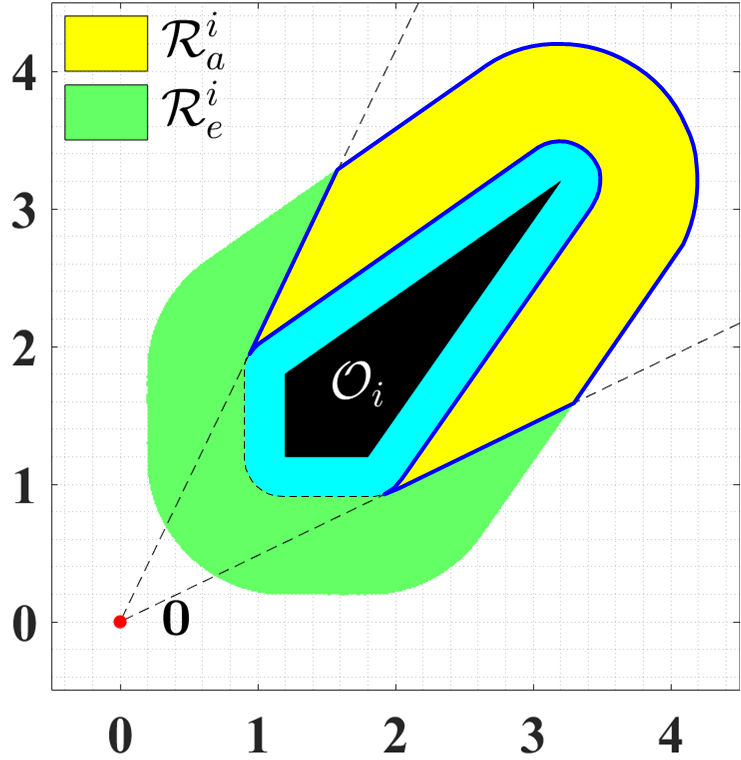}\label{Rs400}}
\caption{Partitioning of the region $\mathcal{N}_{\gamma}(\mathcal{D}_{r_a}(\mathcal{O}_i))$, with $r_a = 0.25\,m$ and $\gamma = 1.15\,m$, into the \textit{avoidance} region $\mathcal{R}_a^i$ and the \textit{exit} region $\mathcal{R}_e^i$, for different values of the sensing radius $R_s$: (a) $R_s = 0.5\,m$, (b) $R_s = 1\,m$, (c) $R_s = 1.5\,m$, (d) $R_s = 4\,m$.}
\label{effect_of_Rs}
\end{figure*}

When the robot operates in the \textit{move-to-target} mode, its velocity is directed towards the target location. Hence, if the robot enters in the $\gamma-$neighbourhood of obstacle $\mathcal{O}_i$, it should constantly verify whether the path joining the robot's location and the target is obstructed by $\mathcal{O}_i$. Therefore, we construct the set $\eth_i^{\mathbf{x}}$, which contains locations on the boundary of obstacle $\mathcal{O}_i$ that have a clear line of sight to the center of the robot. Given $\mathbf{x}\in\mathcal{N}_{\gamma}(\mathcal{D}_{r_a}(\mathcal{O}_i))$, the set $\eth_i^{\mathbf{x}}$ is defined as
\begin{equation}
    \eth_i^{\mathbf{x}} := \{\mathbf{y}\in\eth_{\mathcal{W}}^{\mathbf{x}} \,|\, \mathbf{y} \in\partial\mathcal{O}_i\}.\label{visibile_boundary_closest_obstacle}
\end{equation}

Observe that if $\mathbf{x}\in\mathcal{N}_{\gamma}(\mathcal{D}_{r_a}(\mathcal{O}_i))$ and 
\begin{equation}
    \eth_i^{\mathbf{x}}\cap\mathcal{C}_{\mathcal{L}}(\mathbf{x}, 2r_a) \ne \emptyset,\label{sufficient_condition_to_check}
\end{equation}
as shown in Fig. \ref{cylinder_intersection_condition_diagram}, then the robot, moving straight towards the target location, will eventually collide with obstacle $\mathcal{O}_i$, where $\mathcal{C}_{\mathcal{L}}(\mathbf{x}, 2r_a)$ represents the cylinder with width $2r_a$ and axis $\mathcal{L}_s(\mathbf{x}, \mathbf{0})$ as defined in Section \ref{definition:cylinder}. We define the \textit{avoidance} region, denoted by $\mathcal{R}_a^i$, as the set of all locations from the $\gamma-$neighbourhood of $r_a-$dilated obstacle $\mathcal{O}_i$ such that condition \eqref{sufficient_condition_to_check} is satisfied. The set $\mathcal{R}_a^i$ is defined as
\begin{equation}
    \mathcal{R}_a^i:= \left\{\mathbf{x}\in\mathcal{N}_{\gamma}(\mathcal{D}_{r_a}(\mathcal{O}_i))|\eth_i^{\mathbf{x}}\cap\mathcal{C}_{\mathcal{L}}(\mathbf{x}, 2r_a) \ne \emptyset\right\}.\label{Individual_avoidance_region}
\end{equation}

Since the set $\eth_{i}^{\mathbf{x}}$ is used to define the \textit{avoidance} region $\mathcal{R}_a^i$, the shape of $\mathcal{R}_a^i$ depends on the value of the sensing radius $R_s$, as shown in Fig. \ref{effect_of_Rs}. However, as stated in the next lemma, irrespective of the value of the sensing radius $R_s$, the \textit{avoidance} region $\mathcal{R}_a^i$ is always a subset of the \textit{landing} region $\mathcal{R}_l^i$, which is defined as
\begin{equation}
\begin{aligned}
    \mathcal{R}_l^i:= \{\mathbf{x}\in\mathcal{N}_{\gamma}(\mathcal{D}_{r_a}(\mathcal{O}_i))|\mathcal{L}_s(\mathbf{x}, \mathbf{0})\cap&\mathcal{D}_{r_a}(\mathcal{O}_i)\ne \emptyset,\\
    &\mathbf{x}^\top\mathbf{x}_{\pi} \geq 0\}.\label{landing_region}
    \end{aligned}
\end{equation}
Furthermore, if $R_s > l_i$, then one has $\mathcal{R}_a^i = \mathcal{R}_l^i$, where $l_i$ is the largest possible distance between any two points from the set $\mathcal{D}_{r_a + \gamma}(\mathcal{O}_i)$, that is
\begin{equation}
    l_i = \max\{\norm{\mathbf{p} - \mathbf{q}}|\mathbf{p}, \mathbf{q}\in\mathcal{D}_{r_a + \gamma}(\mathcal{O}_i)\}.\label{largest_distance_li}
    \end{equation}
\begin{lemma}
    For each \( i \in \mathbb{I} \setminus \{0\} \) and for any \( R_s > 0 \), it holds that \( \mathcal{R}_a^i \subset \mathcal{R}_l^i \). Additionally, if \( R_s > l_i \), where \( l_i \) is defined in \eqref{largest_distance_li}, then \( \mathcal{R}_a^i = \mathcal{R}_l^i \).
\label{lemma:superset_landing_region}
\end{lemma}
\begin{proof}
    See Appendix \ref{proof:lemma_superset_landing_region}. 
\end{proof}
In Fig. \ref{effect_of_Rs}, the blue curve represents the boundary of the \textit{landing} region $\mathcal{R}_{l}^i$ for obstacle $\mathcal{O}_i$. Notice that the \textit{avoidance} region $\mathcal{R}_a^i$ is always a subset of the \textit{landing} region $\mathcal{R}_l^i$, irrespective of the value of the sensing radius $R_s$.

Since $\mathcal{W}$ is a closed convex set, for all $\mathbf{x}\in\mathcal{N}_{\gamma}(\mathcal{D}_{r_a}(\mathcal{O}_0))$, condition \eqref{sufficient_condition_to_check} is not satisfied, where $\mathcal{O}_0 = \left(\mathcal{W}^{\circ}\right)^c$. Therefore, $\mathcal{R}_a^0$ is an empty set. The union of the \textit{avoidance} regions over all obstacles is given by
\begin{equation}
    \mathcal{R}_a := \bigcup_{i\in\mathbb{I}}\mathcal{R}_a^i.\label{union_avoidance_region}
\end{equation}

Next, we define the \textit{exit} region $\mathcal{R}_e$ as the part of the $\gamma-$neighbourhood of the $r_a-$dilated obstacles that do not belong the \textit{avoidance} region. The \textit{exit} region $\mathcal{R}_e$ is defined as
\begin{equation}
    \mathcal{R}_e = \mathcal{N}_{\gamma}(\mathcal{D}_{r_a}(\mathcal{O}_{\mathcal{W}}))\setminus\mathcal{R}_a.\label{exit_region}
\end{equation}
Note that when the robot's center $\mathbf{x}$ belongs to the \textit{exit} region, condition \eqref{sufficient_condition_to_check} is not satisfied. 
In other words, if the robot, with its center in the \textit{exit} region, moves directly towards the target location along the path connecting its center to the origin, it will not collide with the nearest obstacle within the sensing radius $R_s$. Hence, the robot should move straight towards the target location only if it is in the \textit{exit} region. 

Next, we provide the geometric construction of the flow set $\mathcal{F}$ and the jump set $\mathcal{J}$, used in \eqref{hybrid_control_law}.


\subsubsection{Flow and jump sets (\textit{move-to-target} mode)}\label{subsubsection_flowjump_move2target}

When the robot, in the \textit{move-to-target} mode, enters the $\gamma$-neighborhood of an obstacle obstructing its straight path to the target location, the control input switches to the \textit{obstacle-avoidance} mode.
Hence, the jump set $\mathcal{J}_0$ of the \textit{move-to-target} mode is defined as
\begin{equation}\label{mtt_jumpset}
    \mathcal{J}_0 := \{\xi\in\mathcal{K}\mid\mathbf{x}\in\mathcal{J}_0^{\mathcal{W}}, m = 0\},
\end{equation}
where the set $\mathcal{J}_0^{\mathcal{W}}$ is given as
\begin{equation}
\mathcal{J}_{0}^{\mathcal{W}}:= {\mathcal{N}_{\gamma_s}(\mathcal{D}_{r_a}(\mathcal{O}_{\mathcal{W}}))\cap\mathcal{R}_a},\label{jumpset_movetotarget}
\end{equation}
with $\gamma_s \in (0, \gamma).$ 
In \eqref{jumpset_movetotarget}, we allow the robot to enter the $\gamma_s$-neighborhood of obstacles before switching to the \textit{obstacle-avoidance} mode. This creates a hysteresis region, $\mathcal{N}_{\gamma - \gamma_s}(\mathcal{D}_{r_a + \gamma_s}(\mathcal{O}_{\mathcal{W}}))$, which acts as a buffer zone to prevent frequent switching between the modes due to small disturbances or noise.

The flow set of the \textit{move-to-target} mode is then defined as
\begin{equation}\label{mtt_flowset}
\mathcal{F}_0 := \{\xi\in\mathcal{K}\mid\mathbf{x}\in\mathcal{F}_0^{\mathcal{W}}, m = 0\},
\end{equation}
where the set $\mathcal{F}_0^{\mathcal{W}}$ is given by
\begin{equation}
\mathcal{F}_{0}^{\mathcal{W}} := \left(\mathcal{W}\setminus(\mathcal{D}_{r_a+\gamma_s}^{\circ}(\mathcal{O}_{\mathcal{W}}))\right)\cup\overline{\mathcal{R}_e}.\label{flowset_movetotarget}
\end{equation}
Notice that the union  $\mathcal{J}_0^{\mathcal{W}}$ and  $\mathcal{F}_0^{\mathcal{W}}$ covers the obstacle-free workspace $\mathcal{W}_{r_a}$, as shown in Fig. \ref{flow_jump_set_2D}. 

\begin{figure}
    \centering
    \includegraphics[width=0.7\linewidth]{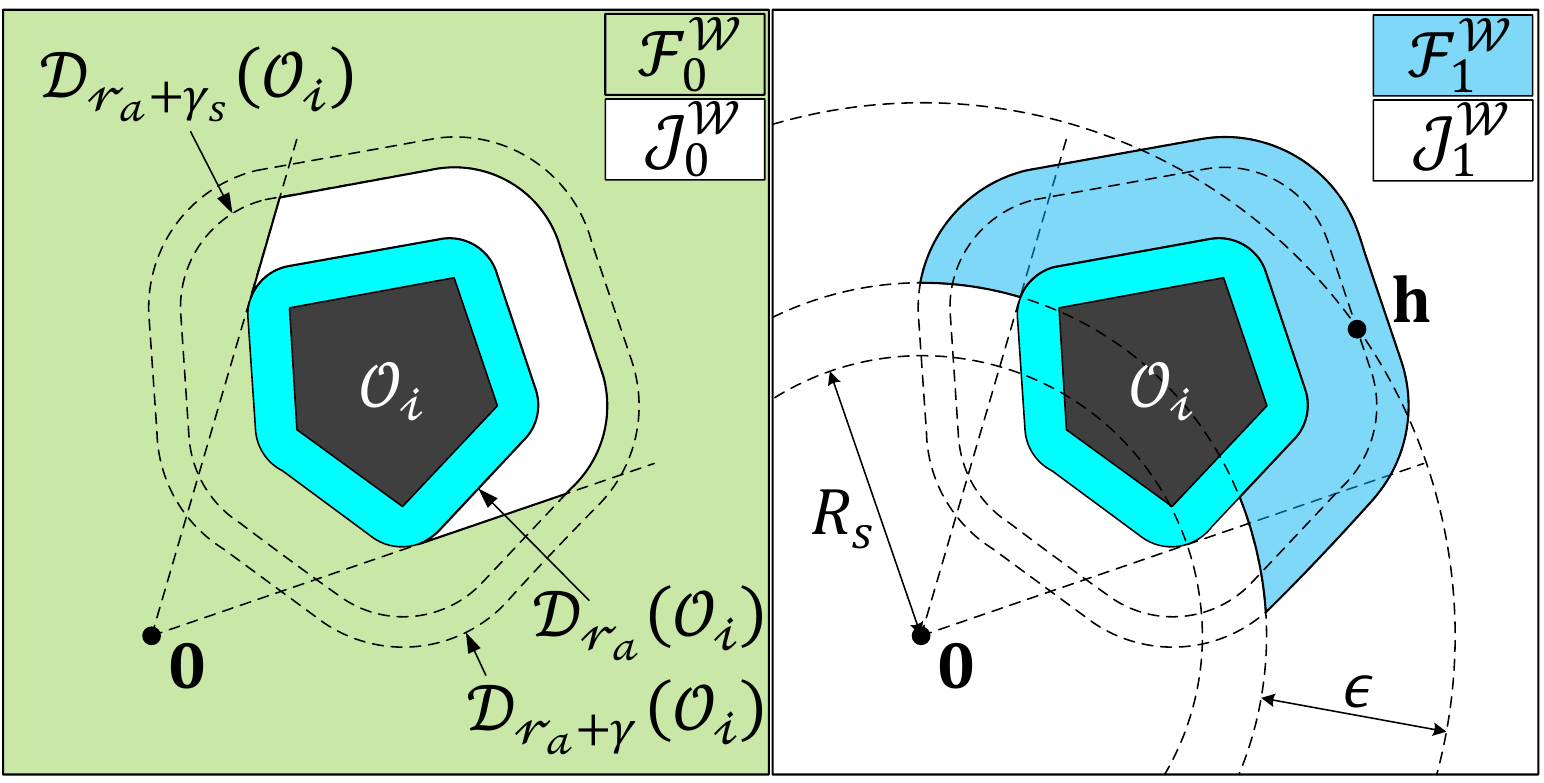}
    \caption{Two-dimensional illustration of the flow sets and the jump sets considered in Sections \ref{subsubsection_flowjump_move2target} and \ref{subsubsection_flowjump_obstacleavoidance}.}
    \label{flow_jump_set_2D}
\end{figure}

\subsubsection{Flow and jump sets (\textit{obstacle-avoidance} mode)}
\label{subsubsection_flowjump_obstacleavoidance}
 The robot operates in the \textit{obstacle-avoidance} mode inside the $\gamma$-neighbourhood of the obstacles. 
 Since the robot can safely be steered straight towards the target location and exit the $\gamma-$neighborhood of the obstacle being avoided, whenever $\mathbf{x}$ is in the \textit{exit} region \eqref{exit_region}, the control input should switch back to the \textit{move-to-target} mode only if $\mathbf{x}\in\mathcal{R}_e$.
 To that end, the jump set $\mathcal{J}_1$ of the \textit{obstacle-avoidance} mode is defined as follows:
\begin{equation}\label{oa_jumpset}
    \mathcal{J}_1 := \mathcal{J}_{\mathbf{x}}\cup\mathcal{J}_s,
\end{equation}
where the set $\mathcal{J}_{\mathbf{x}}$ is defined as
\begin{equation}
    \mathcal{J}_{\mathbf{x}} := \{\xi\in\mathcal{K}\mid \mathbf{x}\in\mathcal{J}_1^{\mathcal{W}}, m = 1\},
\end{equation}
and the set $\mathcal{J}_1^{\mathcal{W}}$ is given by
{\small\begin{equation}
    \mathcal{J}_1^{\mathcal{W}} := \left(\mathcal{W}\setminus\mathcal{D}_{r_a + \gamma}^{\circ}(\mathcal{O}_{\mathcal{W}})\right)\cup\mathcal{ER}^{\mathbf{h}}{\cup\mathcal{N}_{\gamma}(\mathcal{D}_{r_a}(\mathcal{O}_0))}\cup\mathcal{S}_{\mathbf{0}}.\label{jumpset_obstacleavoidance}
\end{equation}}

Next, we provide definitions of the sets $\mathcal{ER}^{\mathbf{h}}$ and $\mathcal{S}_{\mathbf{0}}$, which is then followed by the definition of the set $\mathcal{J}_s$. 
We make use of the \textit{hit point} $\mathbf{h}$ (\textit{i.e.}, the location $\mathbf{x}$ where the control input switched from the \textit{move-to-target} mode to the current \textit{obstacle-avoidance} mode) to define the set $\mathcal{ER}^{\mathbf{h}}$ as follows:
\begin{equation}
    \mathcal{ER}^{\mathbf{h}} := \{\mathbf{x}\in\mathcal{R}_e\big{|}\norm{\mathbf{h}} - \norm{\mathbf{x}} \geq \epsilon\},\label{partition_rm}
\end{equation}
where $\epsilon \in(0, \bar{\epsilon}]$ and $\bar{\epsilon}$ is a sufficiently small positive scalar. The set $\mathcal{ER}^{\mathbf{h}}$ contains the locations $\mathbf{x}$ from the \textit{exit} region $\mathcal{R}_e$ for which the target location is at least $\epsilon$ units closer to $\mathbf{x}$ than to the current \textit{hit point} $\mathbf{h}$. 
Since the obstacles are compact and convex, and the target location $\mathbf{0}$ is within the interior of the obstacle-free workspace $\mathcal{W}_{r_a}$, it is possible to guarantee the existence of the parameter $\bar{\epsilon}$ such that the intersection set $\mathcal{ER}^{\mathbf{h}}\cap\mathcal{N}_{\gamma}(\mathcal{D}_{r_a}(\mathcal{O}_i))$ is non-empty for every $\mathbf{h}\in\mathcal{J}_0^{\mathcal{W}}\cap\mathcal{N}_{\gamma}(\mathcal{D}_{r_a}(\mathcal{O}_i))$ for each $i\in\mathbb{I}\setminus\{0\}$, as stated in the following lemma:
\begin{lemma}
    Let Assumption \ref{3d_obstacle_separation} hold. Then, for each $i\in\mathbb{I}\setminus\{0\}$, for every $\mathbf{h}\in\mathcal{J}_{0}^{\mathcal{W}}\cap\mathcal{N}_{\gamma}(\mathcal{D}_{r_a}(\mathcal{O}_i))$, there exists $\bar{\epsilon} > 0$ such that for any $\epsilon\in(0, \bar{\epsilon}]$, one has
$\mathcal{ER}^{\mathbf{h}}\cap\mathcal{N}_{\gamma}(\mathcal{D}_{r_a}(\mathcal{O}_i))\ne \emptyset, $ where $\mathcal{ER}^{\mathbf{h}}$ is defined in \eqref{partition_rm}.\label{lemma:epsilon_exists}
\end{lemma}
    
\begin{proof}
    See Appendix \ref{proof_of_epsilon_exists}.
\end{proof}

According to \eqref{jumpset_obstacleavoidance} and \eqref{partition_rm}, the robot operating in the \textit{obstacle-avoidance} mode, can switch to the \textit{move-to-target} mode when its center belongs to the \textit{exit} region $\mathcal{R}_e$ and the target location $\mathbf{0}$ is at least $\epsilon$ units closer to $\mathbf{x}$ than to the current \textit{hit point} $\mathbf{h}$. This introduces a hysteresis region that prevents frequent switching between the modes due to small disturbances or noise. 
Note that if $\epsilon$ is set to a very high value, it may result in $\mathcal{ER}^{\mathbf{h}}\cap\mathcal{N}_{\gamma}(\mathcal{D}_{r_a}(\mathcal{O}_i)) = \emptyset$, where $\mathbf{h}\in\mathcal{J}_{0}^{\mathcal{W}}\cap\mathcal{D}_{r_a + \gamma}(\mathcal{O}_i)$ for some $i\in\mathbb{I}$. Therefore, one should choose a sufficiently small value for $\epsilon$ while compensating for the measurement noise such that $\mathcal{B}_{\|\mathbf{h}\| - \epsilon}(\mathbf{0})\cap\mathcal{D}_{r_a}(\mathcal{O}_i) \ne \emptyset$.
Next, we define the set $\mathcal{S}_0$, used in \eqref{jumpset_obstacleavoidance}.

\begin{figure}
\centering
\subfloat[][]{\includegraphics[width =0.425\linewidth]{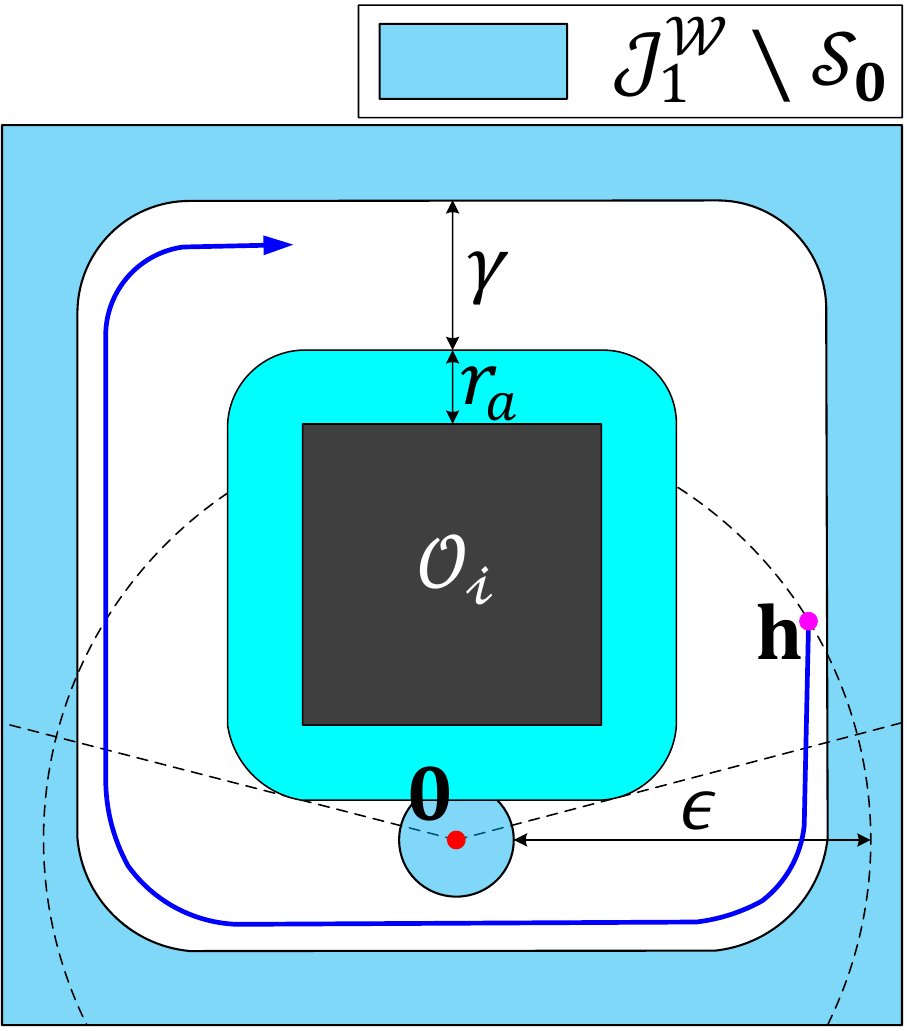}\label{without_S0}}\hspace{0.1cm}
\subfloat[][]{\includegraphics[width =0.425\linewidth]{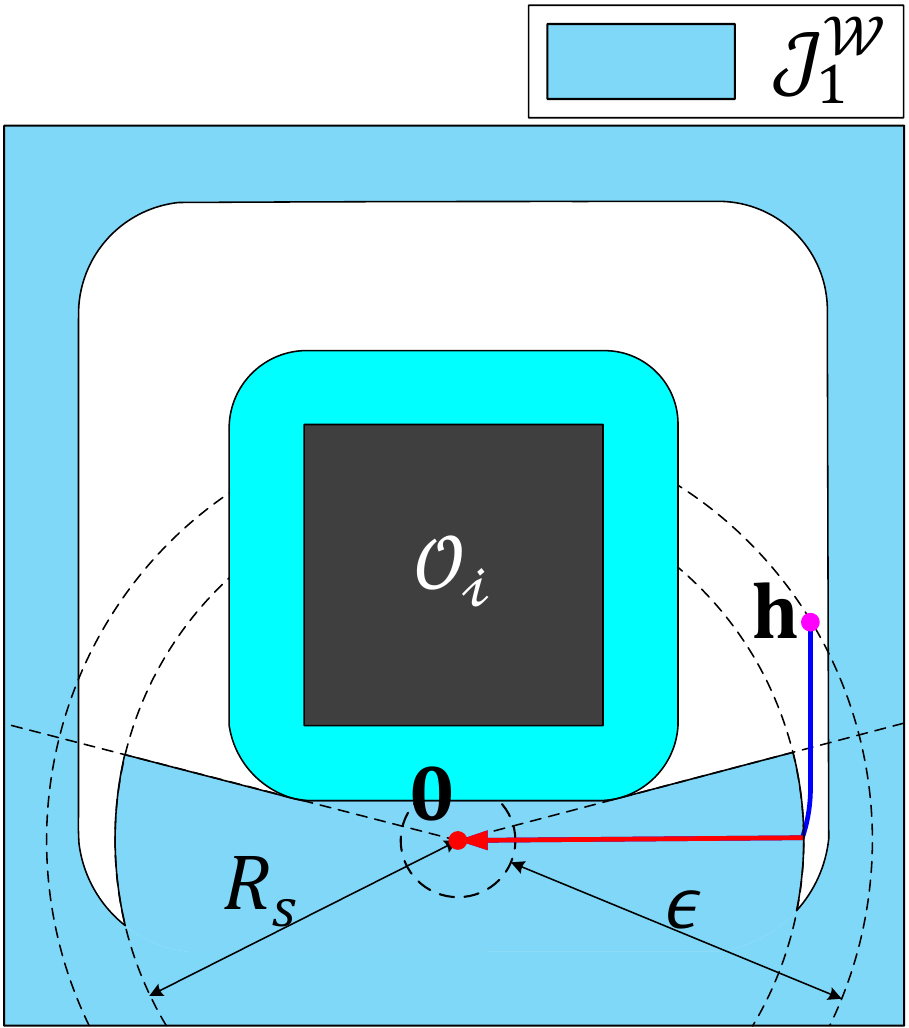}\label{with_S0}}
\caption{(a) Illustration of the region $\mathcal{J}_1^{\mathcal{W}} \setminus\mathcal{S}_{\mathbf{0}}$. (b) Illustration of the region $\mathcal{J}_1^{\mathcal{W}}$.}
\label{need_for_S0}
\end{figure}

The set $\mathcal{S}_{\mathbf{0}}$ contains locations from the set $\mathcal{N}_{\gamma}(\mathcal{D}_{r_a}(\mathcal{O}_{\mathcal{W}}))\cap\mathcal{B}_{R_s}(\mathbf{0})$ from which the robot can safely move straight towards the target location at the origin, and is given by
\begin{equation}
    \mathcal{S}_{\mathbf{0}}:=\{\mathbf{x}\in\mathcal{N}_{\gamma}(\mathcal{D}_{r_a}(\mathcal{O}_{\mathcal{W}}))\cap\mathcal{B}_{R_s}(\mathbf{0})|\eth_{\mathcal{W}}^{\mathbf{x}}\cap\mathcal{C}_{\mathcal{L}}(\mathbf{x}, 2r_a) = \emptyset\}.\label{line_of_sight_to_target}
\end{equation}
where $R_s > r_a + \gamma$.
Notice that if $\mathcal{S}_{\mathbf{0}}$ is excluded from the set $\mathcal{J}_{1}^{\mathcal{W}}$ in \eqref{jumpset_obstacleavoidance} and $\mathbf{0}$ is in the $\gamma-$neighbourhood of obstacle $\mathcal{O}_i$ for some $i\in\mathbb{I}\setminus\{0\}$, then the $\mathbf{x}$ trajectory, starting from some $\mathbf{h}\in\mathcal{N}_{\gamma}(\mathcal{D}_{r_a}(\mathcal{O}_{i}))$, may not enter in the set $\overline{\mathcal{ER}^{\mathbf{h}}}$ and as such in the set $\mathcal{J}_1^{\mathcal{W}}$ when $\epsilon$ is set relatively high. This may cause the robot to indefinitely operate in the \textit{obstacle-avoidance} mode in the $\gamma-$neighbourhood of obstacle $\mathcal{O}_i$. For example, consider Fig. \ref{without_S0}, in which the $\mathbf{x}$ trajectory starting from $\mathbf{h}$ does not enter the set $\mathcal{J}_{1}^{\mathcal{W}}\setminus\mathcal{S}_{\mathbf{0}}$. Therefore, we include the set $\mathcal{S}_{\mathbf{0}}$ in \eqref{jumpset_obstacleavoidance} to ensure that the $\mathbf{x}$ trajectories starting from any $\mathbf{h}\in\mathcal{N}_{\gamma}(\mathcal{D}_{r_a}(\mathcal{O}_i))$ always enter in the set $\mathcal{J}_{1}^{\mathcal{W}}$ when the control input corresponds to the \textit{obstacle-avoidance} mode, as shown in Fig. \ref{with_S0}.

Notice that in Lemma \ref{lemma:epsilon_exists}, the existence of $\bar{\epsilon}$ is guaranteed when the \textit{hit point} $\mathbf{h}$ belongs to the set $\mathcal{J}_0^{\mathcal{W}}$.
However, if $\mathbf{x}$ is initialized in the \textit{obstacle-avoidance} mode $(m = 1)$ in the $\gamma$-neighborhood of $\mathcal{D}_{r_a}(\mathcal{O}_i)$ for some $i\in\mathbb{I}$, and $\mathbf{h}$ initialized in $\mathcal{W}_{r_a}\setminus\mathcal{J}_0^{\mathcal{W}}$, then $\bar{\epsilon} > 0$ such that $\mathcal{ER}^{\mathbf{h}}\cap\mathcal{N}_{\gamma}(\mathcal{D}_{r_a}(\mathcal{O}_i)) \ne \emptyset$ may not exist.
In this case, the control input remains in obstacle-avoidance mode, and the trajectory $\mathbf{x}(t)$ evolves indefinitely within the $\gamma$-neighborhood of $\mathcal{D}_{r_a}(\mathcal{O}_i)$, without switching to the move-to-target mode.
To prevent this issue, we introduced the set $\mathcal{J}_s$ in \eqref{oa_jumpset}, where 
\begin{equation}\label{J_s_set_definition}\mathcal{J}_s := \{\xi\in\mathcal{K}|m = 1, s = [s_0, s_0 + \delta_s]\},\end{equation} with $s_0\in\mathbb{R}_{\geq0}$ denoting the initial value of the state $s$, \textit{i.e.,} $s(0,0)=s_0$, and $0 < \delta_s < \tau_s$ for some $\tau_s > 0$. 

\begin{remark}
The inclusion of the set $\mathcal{J}_{s}$ in the set $\mathcal{J}_1$ in  \eqref{oa_jumpset} enables the control to immediately switch to the \textit{move-to-target} mode if it is initialized in the \textit{obstacle-avoidance} mode (\textit{i.e.,} $\xi(0, 0) \in \mathcal{J}_1$). This ensures that the \textit{hit point} $\mathbf{h}$ always belongs to the set $\mathcal{J}_{0}^{\mathcal{W}}$ before the robot starts moving in the \textit{obstacle-avoidance} mode, thus guaranteeing the existence of the parameter $\bar{\epsilon}$, as stated in Lemma \ref{lemma:epsilon_exists}.\label{immediate_switching_the_mode}
\end{remark}

The flow set of the \textit{obstacle-avoidance} mode is defined as
\begin{equation}\label{oa_flowset}
\mathcal{F}_1 := \{\xi\in\mathcal{K}|\mathbf{x}\in\mathcal{F}_1^{\mathcal{W}}, m = 1, s\notin(s_0, s_0 + \delta_s)\},
\end{equation}
where the set $\mathcal{F}_{1}^{\mathcal{W}}$ is given as
{\small
\begin{equation} \mathcal{F}_{1}^{\mathcal{W}}:= \overline{\mathcal{N}_{\gamma}(\mathcal{D}_{r_a}(\mathcal{O}_{\mathcal{W}}))\setminus{(\mathcal{ER}^{\mathbf{h}}\cup\mathcal{N}_{\gamma}(\mathcal{D}_{r_a}(\mathcal{O}_0))\cup\mathcal{S}_0)}}.\label{flowset_obstacleavoidance}
\end{equation}}
Notice that the union of $\mathcal{J}_1^{\mathcal{W}}$ and $\mathcal{F}_{1}^{\mathcal{W}}$ exactly covers the obstacle-free workspace $\mathcal{W}_{r_a}$, as shown in Fig. \ref{flow_jump_set_2D}. 

\begin{remark}
Given that the workspace $\mathcal{W}$ is both convex and compact, there may exist some locations $\mathbf{x}\in\mathcal{N}_{\gamma}(\mathcal{D}_{r_a}(\mathcal{O}_0))$ from which the nearest point to the robot's center on the obstacle $\mathcal{O}_0 = (\mathcal{W}^{\circ})^c$ is not unique. This scenario prevents the implementation of the obstacle-avoidance term $\mathbf{v}(\mathbf{x}, \mathbf{h}, \mathbf{a})$ in the control law, at such locations. However, by excluding the set $\mathcal{N}_{\gamma}(\mathcal{D}_{r_a}(\mathcal{O}_0))$ from the set $\mathcal{F}_1^{\mathcal{W}}$, as defined in  \eqref{flowset_obstacleavoidance},  it is ensured that the \textit{obstacle-avoidance} mode is never activated within the set $\mathcal{N}_{\gamma}(\mathcal{D}_{r_a}(\mathcal{O}_0))$. 
\label{remark:not_executed_near_workspace_boundary}
\end{remark}

Finally, the flow set $\mathcal{F}$ and the jump set $\mathcal{J}$, used in \eqref{hybrid_control_law}, are defined as
\begin{equation}
    \mathcal{F}:= \bigcup_{m\in\mathbb{M}}\mathcal{F}_m, \quad\mathcal{J}:= \bigcup_{m\in\mathbb{M}}\mathcal{J}_m,\label{bothmodes_flowjumpset}
\end{equation}
where $\mathcal{J}_0$, $\mathcal{F}_0$, $\mathcal{J}_1$, and $\mathcal{F}_1$ are defined in \eqref{mtt_jumpset}, \eqref{mtt_flowset}, \eqref{oa_jumpset}, and \eqref{oa_flowset}, respectively.

\begin{remark}
    The set $\mathcal{J}_0^{\mathcal{W}}$ (22), which contributes to the jump set $\mathcal{J}_0$ of the move-to-target mode, is restricted to the $\gamma_s$-neighborhood of obstacles, where $\gamma_s \in (0, \gamma)$.
    Additionally, the inclusion of the set $\mathcal{ER}^{\mathbf{h}}$ (26) in the definition of $\mathcal{J}_1^{\mathcal{W}}$, which is used to construct the jump set $\mathcal{J}_1$ of the obstacle-avoidance mode, ensures that the control
    input $\mathbf{u}$ switches to the move-to-target mode only when the condition $d(\mathbf{0}, \mathbf{x}) \leq d(\mathbf{0}, \mathbf{h}) - \epsilon$ is satisfied, where $\epsilon \in (0, \bar{\epsilon}]$ and the existence of $\bar{\epsilon}$ is established in Lemma 2.
        
        Together, these design choices ensure that the Euclidean distance between any two consecutive switching locations is at least $\min\{\gamma - \gamma_s, \epsilon\}$.
        In practical implementations, where the state measurements may be affected by arbitrarily small noise bounded above by $\min\{\gamma - \gamma_s, \bar{\epsilon}\}$, this prevents chattering behavior where, due to measurement noise, the control law repeatedly switches between modes while the robot remains stationary.
\end{remark}

Next, we provide the update law $\mathbf{L}(\mathbf{x}, \mathbf{h}, \mathbf{a}, m, s)$ used in \eqref{control_L}.

\subsection{Update law \texorpdfstring{$\mathbf{L}(\mathbf{x}, \mathbf{h}, \mathbf{a}, m, s)$}{}}\label{section:update_law}
The update law $\mathbf{L}(\xi)$, used in \eqref{control_L}, updates the value of the \textit{hit point} $\mathbf{h}$, the unit vector $\mathbf{a}$, the mode indicator $m$ and the variable $s$ when the state $(\mathbf{x}, \mathbf{h}, \mathbf{a}, m, s)$ belongs to the jump set $\mathcal{J}$ defined in \eqref{bothmodes_flowjumpset} and is given by 
\begin{equation}
    \mathbf{L}(\xi) = \begin{cases}\mathbf{L}_0(\xi), &\xi\in\mathcal{J}_0,\\
    \mathbf{L}_1(\xi), &\xi\in\mathcal{J}_1.\end{cases}\label{update_law}
\end{equation}
When the state $\xi$ enters in the jump set $\mathcal{J}_0$ of the \textit{move-to-target} mode, defined in \eqref{mtt_jumpset}, the update law $\mathbf{L}_0(\mathbf{x}, \mathbf{h}, \mathbf{a}, 0, s)$ is given as
\begin{equation}
    \mathbf{L}_0(\mathbf{x}, \mathbf{h}, \mathbf{a}, 0, s) = \left\{\begin{bmatrix}\mathbf{x}\\\mathbf{a}^{\prime}\\ 1\\s+\tau_s\end{bmatrix}, \mathbf{a}^{\prime}\in\mathbf{A}(\mathbf{x})\right\} ,\label{updatelaw_part1}
\end{equation}
where $\tau_s > 0$.
Given $\mathbf{x}\in\mathcal{N}_{\gamma}(\mathcal{D}_{r_a}(\mathcal{O}_{\mathcal{W}}))$, the set-valued mapping $\mathbf{A}$ is defined as
\begin{equation}
    \mathbf{A}(\mathbf{x}) = \begin{cases}\mathbf{q}\in\mathcal{P}^{\perp}(\mathbf{x}), & \mathbf{x}^\times\mathbf{x}_{\pi} = \mathbf{0},\\
    \mathbf{q}\in\mathcal{P}^{\perp}(\mathbf{x})\cap\mathcal{P}(\mathbf{x}, \mathbf{x}_{\pi}), &\mathbf{x}^\times\mathbf{x}_{\pi} \ne \mathbf{0}, 
    \end{cases}\label{update_law_for_vector_a}
\end{equation}
where for any $\mathbf{p}\in\mathbb{R}^n$, the set $\mathcal{P}^{\perp}(\mathbf{p})$, which is defined as
\begin{equation}
    \mathcal{P}^{\perp}(\mathbf{p}) := \{\mathbf{q}\in\mathbb{S}^{n-1}|\mathbf{q}^\intercal\mathbf{p} = 0\},\label{orthogonal_projection_on_sphere}
\end{equation}
contains unit vectors that are perpendicular to the vector $\mathbf{p}.$

As per \eqref{updatelaw_part1}, when $\xi$ enters in the jump set $\mathcal{J}_0$ of the \textit{move-to-target} mode, the current value of $\mathbf{x}$ is assigned to the \textit{hit point} $\mathbf{h}$. Moreover, using \eqref{update_law_for_vector_a}, the unit vector $\mathbf{a}$ is updated to be perpendicular to $\mathbf{x}$ and lies in the plane spanned by $\mathbf{x}$ and $\mathbf{x}_{\pi}$, provided they are not collinear.  As a result, the plane $\mathcal{P}(\mathbf{h}, \mathbf{a})$ passes through the target location at the origin and intersects the interior of the obstacle being avoided. This ensures that $\mathcal{P}(\mathbf{h}, \mathbf{a})$ intersects both the \textit{avoidance} region $\mathcal{R}_a$ and the \textit{exit} region $\mathcal{R}_e$ associated with the obstacle to be avoided. This property guarantees that, while operating in the \textit{obstacle-avoidance} mode, the robot eventually enters in the \textit{exit} region where the obstacle no longer blocks its straight path to the origin, and switches back to the \textit{move-to-target} mode. This allows one to establish global convergence properties of the target location at the origin, as stated later in Theorem \ref{theorem:global_stability}.

When the state $\xi$ enters in the jump set $\mathcal{J}_1$ of the \textit{obstacle-avoidance} mode, defined in \eqref{oa_jumpset}, the update law $\mathbf{L}_1(\mathbf{x}, \mathbf{h}, \mathbf{a}, 1, s)$ is given by
\begin{equation}
    \mathbf{L}_1(\mathbf{x}, \mathbf{h}, \mathbf{a}, 1, s) = \left\{\begin{bmatrix}\mathbf{h}\\ \mathbf{a}\\0\\s+\tau_s\end{bmatrix}\right\}.\label{updatelaw_part2}
\end{equation}
According to \eqref{updatelaw_part2}, when the robot switches from the \textit{obstacle-avoidance} mode to the \textit{move-to-target} mode, the coordinates of the \textit{hit point} $\mathbf{h}$ and the unit vector $\mathbf{a}$ remain unchanged.

\begin{remark}
    Since the parameter $\gamma\in(0, \delta - r_s)$, according to Assumption \ref{3d_obstacle_separation}, the $(r_a + \gamma)-$dilated obstacles $\mathcal{D}_{r_a + \gamma}(\mathcal{O}_i)$, $\forall i\in\mathbb{I}$, are disjoint. 
    Furthermore, according to \eqref{flowset_obstacleavoidance}, the set $\mathcal{F}_1^{\mathcal{W}}$ is contained within the region $\mathcal{D}_{r_a+\gamma}(\mathcal{O}_{\mathcal{W}})$. 
    Hence, the proposed control law enables the robot to avoid one obstacle at a time. \label{remark:one_obstacle_at_a_time}
\end{remark}

\begin{remark}
    In our previous works \cite{sawant2023convex} and \cite{sawant_nonconvex}, the theoretical developments, such as the design of the flow and jump sets rely on the assumption of complete knowledge of obstacle geometries. A sensor-based implementation is then proposed to extend the applicability of the hybrid control frameworks proposed therein to unknown 2D environments. In contrast, the present work is built from the outset to operate in \textit{a priori} unknown n-dimensional environments. Specifically, the flow set $\mathcal{F}$ and jump set $\mathcal{J}$, defined in \eqref{bothmodes_flowjumpset}, are directly defined using measurements obtained via a range-bearing sensor mounted on the robot.  
\end{remark}

\textbf{Control design summary}: The proposed hybrid feedback control law can be summarized as follows:
\begin{itemize}
    \item \textbf{Parameters selection}: the target location is set at the origin with $\mathbf{0}\in\mathcal{W}_{r_a}^{\circ}$, and $\xi(0, 0)\in\mathcal{K}$. The gain parameters $\kappa_s$ and $\kappa_r$ are set to positive values, and a sufficiently small positive value is chosen for $\bar{\epsilon}$, used in \eqref{partition_rm}. The scalar parameter $\gamma$, used in the construction of the flow set $\mathcal{F}$ and the jump set $\mathcal{J}$, is selected such that $\gamma\in(0, \delta - r_s)$. The parameters $\gamma_a$ and $\gamma_s$ are set to satisfy $0 < \gamma_a < \gamma_s < \gamma.$ The parameter $\tau_s$ is selected such that $\tau_s > 0$ and $\delta_s$, used in \eqref{J_s_set_definition}, is set to ensure that $0 < \delta_s < \tau_s$.
    
    \item \textbf{\textit{Move-to-target} mode $m = 0$}: this mode is activated when $\xi \in \mathcal{F}_0$. As per \eqref{control_u}, the control input is given by $\mathbf{u}(\xi) = -\kappa_s\mathbf{x}$, causing $\mathbf{x}$ to evolve along the line segment $\mathcal{L}_s(\mathbf{0}, \mathbf{x})$ towards the origin. If, at some instance of time, $\xi$ enters in the jump set $\mathcal{J}_0$ of the \textit{move-to-target} mode, the state variables $(\mathbf{h}, \mathbf{a}, m, s)$ are updated using \eqref{updatelaw_part1}, and the control input switches to the \textit{obstacle-avoidance} mode.
    
    \item\textbf{\textit{Obstacle-avoidance} mode $m = 1$}: this mode is activated when $\xi \in \mathcal{F}_1$. As per \eqref{control_u}, the control input is given by $\mathbf{u}(\xi) = \kappa_r\mathbf{v}(\mathbf{x},\mathbf{h},\mathbf{a})$, causing $\mathbf{x}$ to evolve in the $\gamma-$neighborhood of the nearest obstacle along the plane $\mathcal{P}(\mathbf{h}, \mathbf{a})$ 
    until the state $\xi$ enters in the jump set $\mathcal{J}_1$ of the \textit{obstacle-avoidance} mode. When $\xi\in\mathcal{J}_1$, the state variables $(m, s)$ are updated using \eqref{updatelaw_part2}, and the control input switches to the \textit{move-to-target} mode.
\end{itemize}

This concludes the design of the proposed hybrid feedback controller \eqref{hybrid_control_law}. Next, we analyze the safety, stability and convergence properties of the proposed hybrid feedback controller.

\section{Stability Analysis}\label{sec:stability}
The hybrid closed-loop system resulting from the hybrid feedback control law \eqref{hybrid_control_law} is given by
\begin{equation}
    \underbrace{\begin{matrix}\mathbf{\dot{x}}\\\mathbf{\dot{h}}\\\mathbf{\dot{a}}\\\dot{m}\\\dot{s}
    \end{matrix}\begin{matrix*}[l]\;=\mathbf{u}(\xi)\\\;=\mathbf{0}\\\;=\mathbf{0}\\\;=0\\\;=1\end{matrix*}}_{\dot{\xi} = \mathbf{F}(\xi), {\xi\in\mathcal{F}}}\quad\quad\underbrace{\begin{matrix}\mathbf{x}^+\\\begin{bmatrix}\mathbf{h}^+\\\mathbf{a}^+\\m^+\\s^+\end{bmatrix}\end{matrix} \begin{matrix*}[l]=\mathbf{x}\vspace{0.56cm}\\\vspace{0.7cm}\in\mathbf{L}(\xi)\end{matrix*}}_{\xi^+ \in\mathbf{J}(\xi), {\xi\in\mathcal{J}}},\label{hybrid_closed_loop_system}
\end{equation}
where $\mathbf{u}(\xi)$ is defined in \eqref{control_u}, and the update law $\mathbf{L}(\xi)$ is provided in \eqref{update_law}. Definitions of the flow set $\mathcal{F}$ and the jump set $\mathcal{J}$ are provided in \eqref{bothmodes_flowjumpset}. Next, we analyze the hybrid closed-loop system \eqref{hybrid_closed_loop_system} in terms of forward invariance of the obstacle-free state space $\mathcal{K}$, along with the stability properties of the target set 
\begin{equation}
    \mathcal{A}:= \{\xi\in\mathcal{K}|\mathbf{x} = \mathbf{0}\}\label{target_set}.
\end{equation}

The next lemma shows that the hybrid closed-loop system \eqref{hybrid_closed_loop_system} satisfies the hybrid basic conditions \cite[Assumption 6.5]{goebel2012hybrid},  which guarantees the well-posedness
of the hybrid closed-loop system.
\begin{lemma}\label{basic-conditions}
The hybrid closed-loop system \eqref{hybrid_closed_loop_system} with data $(\mathcal{F}, \mathbf{F}, \mathcal{J}, \mathbf{J})$ satisfies the following hybrid basic conditions:
\begin{enumerate}
\item the flow set $\mathcal{F}$ and the jump set $\mathcal{J}$, defined in \eqref{bothmodes_flowjumpset}, are closed subsets of $\mathbb{R}^n\times\mathbb{R}^n\times\mathbb{R}^n\times\mathbb{R}\times\mathbb{R}$.
\item the flow map $\mathbf{F}$, defined in \eqref{hybrid_closed_loop_system}, is outer semicontinuous and locally bounded relative to $\mathcal{F}$, $\mathcal{F}\subset\text{dom }\mathbf{F},$ and $\mathbf{F}(\xi)$ is convex for every $\xi\in\mathcal{F},$
\item the jump map $\mathbf{J}$, defined in \eqref{hybrid_closed_loop_system}, is outer semicontinuous and locally bounded relative to $\mathcal{J}$, $\mathcal{J}\subset\text{dom }\mathbf{J}$. 
\end{enumerate}\label{lemma:hybrid_basic_conditions_for_3D}
\end{lemma}
\begin{proof}
    See Appendix \ref{proof_for_basic_conditions}.
\end{proof}

For safe autonomous navigation, the state $\mathbf{x}$ must always evolve within the obstacle-free workspace $\mathcal{W}_{r_a}$, defined in \eqref{r_a-eroded_obstacle-free_workspace}. This is equivalent to having the set $\mathcal{K}$ forward invariant for the hybrid closed-loop system \eqref{hybrid_closed_loop_system}. This is stated in the next Lemma.
\begin{lemma}
    Under Assumption \ref{3d_obstacle_separation}, for the hybrid closed-loop system \eqref{hybrid_closed_loop_system}, the obstacle-free set $\mathcal{K}:= \mathcal{W}_{r_a}\times\mathcal{W}_{r_a}\times\mathbb{S}^{n-1}\times\mathbb{M}\times\mathbb{R}_{\geq0}$ is forward invariant. \label{lemma:set_invariance}
\end{lemma}
\begin{proof}
    See Appendix \ref{proof:lemma_invariance}.
\end{proof}

When $\xi$ is steered in the jump set $\mathcal{J}_0$ of the \textit{move-to-target} mode with $\mathbf{x}$ lying in the $\gamma$-neighborhood of $r_a$-dilated obstacle $\mathcal{O}_i$ for some $i\in\mathbb{I}$, \textit{i.e.} $\mathbf{x}\in\mathcal{N}_{\gamma}(\mathcal{D}_{r_a}(\mathcal{O}_i))$, the state vector $\xi$ is updated as per \eqref{updatelaw_part1} and \eqref{hybrid_closed_loop_system}, and the control input switches to the \textit{obstacle-avoidance} mode. 
In this mode, $\mathbf{u}$ guides $\mathbf{x}$ within $\mathcal{N}_{\gamma}(\mathcal{D}_{r_a}(\mathcal{O}_i))$ along the plane $\mathcal{P}(\mathbf{h}, \mathbf{a})$ until it reaches the jump set $\mathcal{J}_1$ of the \textit{obstacle-avoidance} mode, as stated in the next lemma.
\begin{lemma}
    Under Assumption \ref{3d_obstacle_separation}, consider a solution $\xi$ to the hybrid closed-loop system \eqref{hybrid_closed_loop_system}. If $\xi(t_1, j_1)\in\mathcal{J}_0$ at some $(t_1, j_1)\in\text{ dom }\xi$ such that $\mathbf{x}(t_1, j_1)\in\mathcal{J}_0^{\mathcal{W}}\cap\mathcal{N}_{\gamma}(\mathcal{D}_{r_a}(\mathcal{O}_i))$ for some $i\in\mathbb{I}$, then for all $(t, j)\in(I_{j_1+1}\times j_1 +1)$, the following statements hold true:
    \begin{enumerate}
        \item $\mathbf{x}(t, j) \in \mathcal{N}_{\gamma}(\mathcal{D}_{r_a}(\mathcal{O}_i))\cap\mathcal{P}(\mathbf{h}, \mathbf{a})$;
        \item there exists $t_2 > t_1$ such that $t_2 < \infty$ and $\xi(t_2, j_1 + 1)\in\mathcal{J}_1,$
\end{enumerate}
where $\mathbf{h} = \mathbf{h}(t_1, j_1 +1) = \mathbf{h}(t, j)$ and $\mathbf{a} = \mathbf{a}(t_1, j_1 +1) = \mathbf{a}(t, j)$.
\label{lemma:always_enter_move_to_target_mode}
\end{lemma}
\begin{proof}
    See Appendix \ref{proof_of_lemma_always_enter}.
\end{proof}

Next, we provide one of our main results which states that, for all initial conditions in the obstacle-free set $\mathcal{K}$, the proposed hybrid controller not only ensures safe navigation but also guarantees global asymptotic stability of the target location at the origin.

\begin{theorem}
Under Assumption \ref{3d_obstacle_separation}, for the hybrid closed-loop system \eqref{hybrid_closed_loop_system}, the following holds true:
\begin{enumerate}
        \item[i)] the obstacle-free set $\mathcal{K}$ is forward invariant;
        \item[ii)] the target set $\mathcal{A}$ is globally asymptotically stable over the set $\mathcal{K}$;
        \item[iii)] the number of jumps is finite.
\end{enumerate}
\label{theorem:global_stability}
\end{theorem}

\begin{proof}
See Appendix \ref{proof_of_theorem}.
\end{proof}

\section{Application to sphere worlds}\label{sec:sphere_world}
Obviously, the hybrid feedback controller \eqref{hybrid_control_law}, which is designed for safe autonomous navigation in $n-$D environments with arbitrary convex obstacles, is applicable in sphere words \textit{i.e.,} environments with spherical obstacles. However, in this section, we will take advantage of the simplified geometry of the obstacles to redesign the control law \eqref{hybrid_control_law} in a way that ensures a monotonic decrease of the distance between the robot and the target location--a feature that is not guaranteed with the control law \eqref{hybrid_control_law} in environments with arbitrary convex obstacles.

Let us consider environments with spherical obstacles as defined in \cite{koditschek1990robot}. The workspace $\mathcal{W} := \mathcal{B}_{r_0}(\mathbf{c}_0)$ is a compact sphere with radius $r_0\in\mathbb{R}_{>0}$ and center $\mathbf{c}_0\in\mathbb{R}^n.$  In addition, the workspace $\mathcal{W}$ contains disjoint, compact spherical obstacles $\mathcal{O}_i:=\mathcal{B}_{r_i}(\mathbf{c}_i), i\in\mathbb{I}\setminus\{0\}$, where $r_i\in\mathbb{R}_{\geq0}$ and $\mathbf{c}_i\in\mathcal{W}$ represent the radius and the center of obstacle $\mathcal{O}_i.$ Similar to \cite{koditschek1990robot}, the workspace $\mathcal{W}$ satisfies Assumption \ref{3d_obstacle_separation}.

Taking advantage of the spherical geometry of the obstacles, we will appropriately design the unit vector $\mathbf{a}$ and modify the obstacle-avoidance control vector $\mathbf{v}(\mathbf{x},\mathbf{h}, \mathbf{a})$ in \eqref{n_dimensional_obstacle-avoidance_vector}, to ensure that in the \textit{obstacle-avoidance} mode, the distance between the target location and the robot is monotonically decreasing.  

\begin{figure}
    \centering
    \includegraphics[width = 0.49\linewidth]{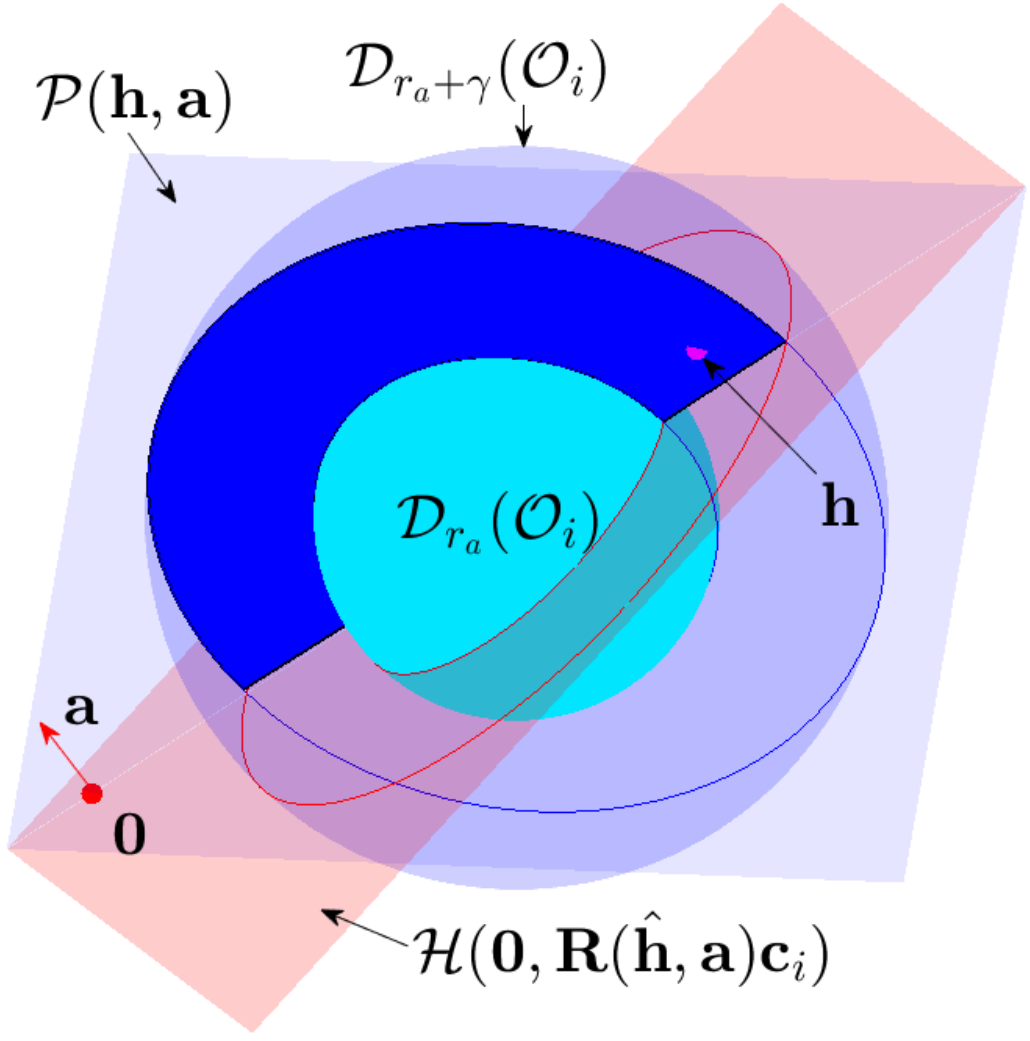}
    \centering
    \includegraphics[width = 0.44\linewidth]{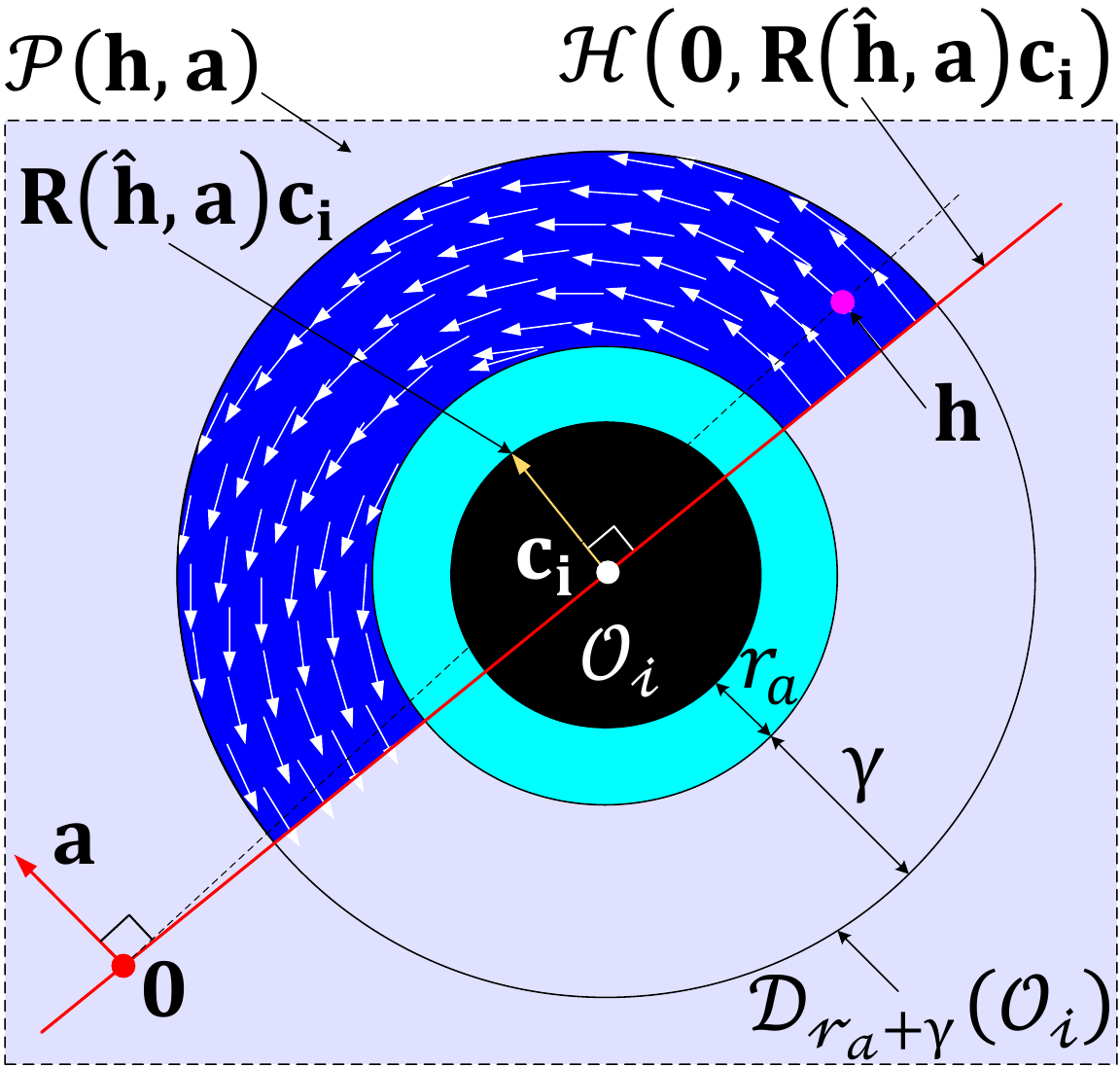}
\caption{The left figure depicts the set $\mathcal{N}_{\gamma}(\mathcal{D}_{r_a}(\mathcal{O}_i)) \cap \mathcal{P}(\mathbf{h}, \mathbf{a}) \cap \mathcal{P}_{\geq}(\mathbf{0}, \mathbf{R}(\hat{\mathbf{h}}, \mathbf{a})\mathbf{c}_i)$, shaded in a dark blue color. The right figure shows the direction of the modified obstacle-avoidance control vector, $\mathbf{v}_s(\mathbf{x}, \mathbf{h}, \mathbf{a})$, at $\mathbf{x} \in \mathcal{N}_{\gamma}(\mathcal{D}_{r_a}(\mathcal{O}_i)) \cap \mathcal{P}(\mathbf{h}, \mathbf{a}) \cap \mathcal{P}_{\geq}(\mathbf{0}, \mathbf{R}(\hat{\mathbf{h}},\mathbf{a})\mathbf{c}_i)$.}
\label{sphere_world_description_figure}
\end{figure}

For any $\mathbf{p}\in\mathcal{N}_{\gamma}(\mathcal{D}_{r_a}(\mathcal{O}_{\mathcal{W}}))$, we define a set-valued mapping $\mathbf{A}_s$ as follows:
\begin{equation}
    \mathbf{A}_s(\mathbf{p}) = \{\mathbf{q}\in\mathbf{A}(\mathbf{p}) |\mathbf{q}^\top\mathbf{p}_{\pi}\geq 0\},\label{sphere_update_law_vector_a}
\end{equation}
where the set-valued mapping $\mathbf{A}$ is defined in \eqref{update_law_for_vector_a}, and $\mathbf{p}_{\pi} = \mathbf{p} - \Pi(\mathbf{p}, \mathcal{O}_{\mathcal{W}})$.

Note that for any \textit{hit point} $\mathbf{h}\in\mathcal{N}_{\gamma}(\mathcal{D}_{r_a}(\mathcal{O}_i))$, for $i\in\mathbb{I}$, if one sets $\mathbf{a} \in \mathbf{A}_s(\mathbf{h})$,
then, since obstacle $\mathcal{O}_i$ is a sphere, it can be shown that for all $\mathbf{x}\in\mathcal{P}(\mathbf{h}, \mathbf{a})\cap\mathcal{N}_{\gamma}(\mathcal{D}_{r_a}(\mathcal{O}_i)),$  $\Pi(\mathbf{x}, \mathcal{O}_{i})$ belongs to the plane $\mathcal{P}(\mathbf{h}, \mathbf{a}).$ 
Therefore, for any given $\mathbf{h}\in\partial\mathcal{D}_{\beta}(\mathcal{O}_i)$, $\beta\in[r_a, r_a + \gamma]$ and for all $\mathbf{x}\in\partial\mathcal{D}_{\beta}(\mathcal{O}_i)\cap\mathcal{P}(\mathbf{h}, \mathbf{a})$, one can show that the vector $\mathbf{R}(\hat{\mathbf{h}}, \mathbf{a})\mathbf{P}(\hat{\mathbf{h}}, \mathbf{a})\mathbf{x}_{\pi}$ belongs to the intersection of the hyperplane $\mathcal{H}(\mathbf{0}, \mathbf{x}_{\pi})$ and the plane $\mathcal{P}(\mathbf{h}, \mathbf{a}),$ where $\mathbf{a}\in\mathbf{A}_s(\mathbf{h})$.
This allows one to ensure that if $\mathbf{h}\in\partial\mathcal{D}_{\beta}(\mathcal{O}_i)$ for $\beta\in[r_a,r_a+\gamma]$, and $\mathbf{x}(t_0, j_0)\in\partial\mathcal{D}_{\beta}(\mathcal{O}_i)\cap\mathcal{P}(\mathbf{h}, \mathbf{a})$ for some $(t_0, j_0)\in\text{ dom }\xi$, where $\mathbf{a} \in\mathbf{A}_s(\mathbf{h})$, then under the control input $\mathbf{u}(\xi) = \mathbf{v}_s(\mathbf{x}, \mathbf{h}, \mathbf{a})$, where 
\begin{equation}
    \mathbf{v}_s(\mathbf{x}, \mathbf{h}, \mathbf{a}) =\mathbf{R}(\hat{\mathbf{h}}, \mathbf{a})\mathbf{x}_{\pi}\label{modified_obstacle-avoidance_control},
\end{equation}
one has $\mathbf{x}(t, j)\in\partial\mathcal{D}_{\beta}(\mathcal{O}_i)\cap\mathcal{P}(\mathbf{h}, \mathbf{a})$ for all time $(t, j)\succeq(t_0, j_0),$ ensuring robot's safety, as stated later in Theorem \ref{theorem:sphere_global_stability}

Now, for a given $\mathbf{h}\in\mathcal{N}_{\gamma}(\mathcal{D}_{r_a}(\mathcal{O}_i))$, consider the hyperplane $\mathcal{H}(\mathbf{0}, \mathbf{R}(\hat{\mathbf{h}}, \mathbf{a})\mathbf{c}_i)$, where $\mathbf{a} \in \mathbf{A}_s(\mathbf{h})$ and $\mathbf{c}_i$ is the center of obstacle $\mathcal{O}_i$. Notice that, for all locations $\mathbf{x}$ in the set $\mathcal{N}_{\gamma}(\mathcal{D}_{r_a}(\mathcal{O}_i))\cap\mathcal{P}(\mathbf{h}, \mathbf{a})\cap\mathcal{H}_{\geq}(\mathbf{0}, \mathbf{R}(\hat{\mathbf{h}}, \mathbf{a})\mathbf{c}_i)$, which is depicted in Fig. \ref{sphere_world_description_figure}, the inner product between the modified obstacle-avoidance control term $\mathbf{v}_s(\mathbf{x}, \mathbf{h}, \mathbf{a})$ and the vector $\mathbf{x}$ is always non-positive. 
This allows one to show that if $\mathbf{h}\in\mathcal{N}_{\gamma}(\mathcal{D}_{r_a}(\mathcal{O}_i))$, $\mathbf{a}\in\mathbf{A}_s(\mathbf{h})$ and $\mathbf{x}\in\mathcal{N}_{\gamma}(\mathcal{D}_{r_a}(\mathcal{O}_i))\cap\mathcal{P}(\mathbf{h}, \mathbf{a})\cap\mathcal{H}_{\geq}(\mathbf{0}, \mathbf{R}(\hat{\mathbf{h}}, \mathbf{a})\mathbf{c}_i)$, then the control input vector $\mathbf{u}(\xi) = \mathbf{v}_s(\mathbf{x},\mathbf{h}, \mathbf{a})$ will ensure monotonic decrease of the distance $\|{\mathbf{x}}\|$ as long as the state $\mathbf{x}$ remains in the set $\mathcal{H}_{\geq}(\mathbf{0}, \mathbf{R}(\hat{\mathbf{h}}, \mathbf{a})\mathbf{c}_i)$, as stated later in Theorem \ref{theorem:sphere_global_stability}. 

Next, motivated by the preceding discussion, we modify the proposed hybrid feedback control law \eqref{hybrid_control_law} using the modified obstacle-avoidance control vector $\mathbf{v}_s(\mathbf{x}, \mathbf{h}, \mathbf{a})$, defined in \eqref{modified_obstacle-avoidance_control}, and the set-valued mapping $\mathbf{A}_s$, provided in \eqref{sphere_update_law_vector_a}, as follows:

\begin{enumerate}
\item the modified hybrid control input vector $\mathbf{u}_s(\xi)$, which is obtained by replacing the obstacle-avoidance control vector $\mathbf{v}(\mathbf{x}, \mathbf{h}, \mathbf{a})$ with the modified obstacle-avoidance control vector $\mathbf{v}_s(\mathbf{x}, \mathbf{h}, \mathbf{a})$ in \eqref{control_u}, is given as
\begin{equation}
    \mathbf{u}_s(\xi) = -\kappa_s(1 - m)\mathbf{x} + \kappa_rm\mathbf{v}_s(\mathbf{x}, \mathbf{h}, \mathbf{a})\label{modified_control_u},
\end{equation}
where $\kappa_s > 0$ and $\kappa_r > 0.$
\item the modified update law $\mathbf{L}_0^s(\xi)$ when the state $\xi$ enters in the jump set $\mathcal{J}_0$ of the \textit{move-to-target} mode is given as
\begin{equation}
    \mathbf{L}_0^s(\mathbf{x}, \mathbf{h}, \mathbf{a}, 0, s) = \left\{\begin{bmatrix}\mathbf{x}\\ \mathbf{a}^{\prime}\\1\\s+\tau_s\end{bmatrix}, \mathbf{a}^{\prime}\in\mathbf{A}_s(\mathbf{x})\right\},\label{modified_update_law}
\end{equation}
where the set-valued mapping $\mathbf{A}_s$ is defined according to \eqref{sphere_update_law_vector_a}.
\end{enumerate}


The hybrid closed-loop system resulting from the modified hybrid feedback control law is given by
\begin{equation}
    \underbrace{\begin{matrix}\mathbf{\dot{x}}\\\mathbf{\dot{h}}\\\mathbf{\dot{a}}\\\dot{m}\\\dot{s}
    \end{matrix}\begin{matrix*}[l]\;=\mathbf{u}_s(\xi)\\\;=\mathbf{0}\\\;=\mathbf{0}\\\;=0\\\;=1\end{matrix*}}_{\dot{\xi} = \mathbf{F}_s(\xi), {\xi\in\mathcal{F}}}\quad\quad\underbrace{\begin{matrix}\mathbf{x}^+\\\begin{bmatrix}\mathbf{h}^+\\\mathbf{a}^+\\m^+\\s^+\end{bmatrix}\end{matrix} \begin{matrix*}[l]=\mathbf{x}\vspace{0.56cm}\\\vspace{0.7cm}\in\mathbf{L}_s(\xi)\end{matrix*}}_{\xi^+ \in\mathbf{J}_s(\xi), {\xi\in\mathcal{J}}},\label{modified_hybrid_closed_loop_system}
\end{equation}
where the control input vector $\mathbf{u}_s$ is defined in \eqref{modified_control_u} and the update law $\mathbf{L}_s$ is obtained by replacing $\mathbf{L}_0$ in \eqref{update_law} with 
 $\mathbf{L}_0^s$ given in \eqref{modified_update_law}.

The next lemma shows that the hybrid closed-loop system \eqref{modified_hybrid_closed_loop_system} with the data $(\mathcal{F}, \mathbf{F}_s, \mathcal{J}, \mathbf{J}_s)$ satisfies the hybrid basic conditions, as stated in Lemma \ref{basic-conditions}.

\begin{lemma}
The hybrid closed-loop system \eqref{modified_hybrid_closed_loop_system} with data $(\mathcal{F}, \mathbf{F}_s, \mathcal{J}, \mathbf{J}_s)$ satisfies the hybrid basic conditions stated in Lemma \ref{basic-conditions}.

\label{lemma:hybrid_basic_conditions_for_sphere}
\end{lemma}
The proof of Lemma \ref{lemma:hybrid_basic_conditions_for_sphere} is similar to the proof of Lemma \ref{lemma:hybrid_basic_conditions_for_3D}, therefore, it is omitted.

Next, we demonstrate that for the robot operating in a sphere world, which satisfies Assumption \ref{3d_obstacle_separation}, the modified proposed hybrid feedback controller ensures safe navigation. It also guarantees global asymptotic stability of the target location at the origin, with a monotonic decrease in the distance between the robot's center and the target location.

\begin{theorem}
Under Assumption \ref{3d_obstacle_separation}, for the hybrid closed-loop system \eqref{modified_hybrid_closed_loop_system}, the following holds true:
\begin{enumerate}
        \item[i)] the obstacle-free set $\mathcal{K}$ is forward invariant,
        \item[ii)] the target set $\mathcal{A}$ is globally asymptotically stable over the set $\mathcal{K}$, 
        \item[iii)] the number of jumps is finite,
        \item[iv)] the distance from the robot's location to the target location is monotonically decreasing.
\end{enumerate}
\label{theorem:sphere_global_stability}
\end{theorem}
\begin{proof}
    See Appendix \ref{proof:sphere_world}.
\end{proof}

\section{Implementation Procedure}\label{implementation_procedure}

We consider a workspace with convex obstacles that satisfies Assumption \ref{3d_obstacle_separation} with some $\delta > 0$, as discussed in Section \ref{section:problem_formulation}. The target location is set at the origin within the interior of the obstacle-free workspace $\mathcal{W}_{r_a}^{\circ}$. The parameters $\gamma, \gamma_s$ and $\gamma_a$ are chosen to satisfy $0<\gamma_a<\gamma_s<\gamma<(\delta - r_s)$. A sufficiently small value for $\bar{\epsilon}$ is selected, and the parameter $\epsilon$, used in \eqref{partition_rm}, is chosen such that $\epsilon\in(0, \bar{\epsilon}]$. The parameter $\delta_s$, used in \eqref{J_s_set_definition}, is set such that $\delta_s\in(0, \tau_s)$, where $\tau_s > 0$. 
The robot is equipped with a range scanner with a sensing radius of $R_s > r_a + \gamma$. The composite state vector $\xi$ is initialized in the set $\mathcal{K}$.

\textbf{Switching from the \textit{move-to-target} mode to the \textit{obstacle-avoidance} mode:}
When the control input is initialized in the \textit{move-to-target} mode, according to \eqref{control_u}, it steers the robot straight towards the origin. 
The robot should constantly measure the distance between its center and the surrounding obstacles to identify whether $\xi$ has entered in the jump set $\mathcal{J}_0$ of the \textit{move-to-target} mode. 
To do this, the robot needs to identify the set $\eth_{\mathcal{W}}^{\mathbf{x}}$, as defined in \eqref{sensor_visible_boundary}, which contains the locations from the boundary of the surrounding obstacles that are less than $R_s$ units away from the center of the robot and have a clear line of sight to the center of the robot, where $R_s > r_a + \gamma$ represents the sensing radius.
Then, one can obtain the distance between the robot's center and the surrounding obstacles by evaluating $d(\mathbf{x}, \eth_{\mathcal{W}}^{\mathbf{x}})$ according to Section \ref{section:metric_projection}. 

If $d(\mathbf{x}, \eth_{\mathcal{W}}^{\mathbf{x}}) \leq r_a + \gamma_s$, one should identify whether the robot can move straight towards the target location without colliding with the nearest obstacle within the sensing region. In other words, one should check condition \eqref{sufficient_condition_to_check} to identify whether the center of the robot belongs to the \textit{avoidance} region $\mathcal{R}_a$, defined in \eqref{Individual_avoidance_region}, associated with the nearest obstacle, let us say $\mathcal{O}_i, i\in\mathbb{I}$. To that end, one needs to identify the set $\eth_{i}^{\mathbf{x}}$, as defined in \eqref{visibile_boundary_closest_obstacle}, which contains the locations from the set $\eth_{\mathcal{W}}^{\mathbf{x}}$ that belong to the boundary of the closest obstacle $\mathcal{O}_i$.
Once the set $\eth_{i}^{\mathbf{x}}$ has been identified, one needs to determine whether $\mathbf{x}$ belongs to the \textit{avoidance} region $\mathcal{R}_a$, defined in \eqref{Individual_avoidance_region}, by evaluating the intersection between the set $\eth_{i}^{\mathbf{x}}$ and the set $\mathcal{C}_{\mathcal{L}}(\mathbf{x}, 2r_a)$. If $\eth_{i}^{\mathbf{x}}\cap\mathcal{C}_{\mathcal{L}}(\mathbf{x}, 2r_a) \not= \emptyset$, then the center of the robot belongs to the \textit{avoidance} region $\mathcal{R}_a$ and the state $\xi$ has entered in the jump set $\mathcal{J}_0$ of the \textit{move-to-target}. Otherwise, the robot continues to operate in the \textit{move-to-target} mode.

\textbf{Switching from the \textit{obstacle-avoidance} mode to the \textit{move-to-target} mode:}
When the state $\xi$ enters in the jump set $\mathcal{J}_0$, the state $\xi$ is updated as per \eqref{updatelaw_part1} and \eqref{hybrid_closed_loop_system}, and the control input switches to the \textit{obstacle-avoidance} mode.
According to Lemma \ref{lemma:always_enter_move_to_target_mode}, when the robot operates in the \textit{obstacle-avoidance} mode, it stays inside the $\gamma-$neighborhood of the closest obstacle. As the robot operates in the \textit{obstacle-avoidance} mode, we continuously evaluate the intersection $\eth_{i}^{\mathbf{x}}\cap\mathcal{C}_{\mathcal{L}}(\mathbf{x}, 2r_a)$ to check whether the center of the robot has entered in the \textit{exit} region $\mathcal{R}_e$. If $\eth_{i}^{\mathbf{x}}\cap\mathcal{C}_{\mathcal{L}}(\mathbf{x}, 2r_a) = \emptyset$, then the center of the robot belongs to the exit region $\mathcal{R}_e$. 
Additionally, if the target location is at least $\epsilon$ units closer to $\mathbf{x}$ than to the current \textit{hit point} $\mathbf{h}$, it implies that $\xi\in\mathcal{J}_1$. Then, the value of the variables $m$ and $s$ are updated as per \eqref{updatelaw_part2} and the control input switches to the \textit{move-to-target} mode. 

Finally, if the control input is initialized in the \textit{obstacle-avoidance} mode, according to \eqref{oa_jumpset} and \eqref{J_s_set_definition}, the state $\xi(0, 0)$ belongs to the jump set of the \textit{obstacle-avoidance} mode $\mathcal{J}_1$. As a result, according to \eqref{updatelaw_part2}, the variables $\mathbf{h}, \mathbf{a}, m$ and $s$, are updated and the control input switches to the \textit{move-to-target} mode.

The above-mentioned implementation procedure is summarized in Algorithm \ref{alg:general_implemenration}.

\subsection{Control-input smoothing mechanism for practical implementations}
Since the proposed hybrid closed-loop system \eqref{hybrid_closed_loop_system} satisfies the hybrid basic conditions, as mentioned in Lemma \ref{basic-conditions}, the resultant robot trajectories during the flow are always smooth. However, the control input vector $\mathbf{u}$, defined in \eqref{control_u}, in general, changes direction instantaneously when switching between modes, causing the overall robot trajectories to become non-smooth.
This instantaneous change in the direction of $\mathbf{u}$ is caused by a discrete change in the value of the mode indicator variable $m$ when the state $\xi$ enters in the jump set $\mathcal{J}$. Therefore, in order to generate smooth robot trajectories, one solution consists in replacing $m$ in \eqref{control_u} with the following continuous scalar function $\lambda$ which takes values between $0$ and $1$:
\begin{equation}
    \lambda(m, s, \tau) = \begin{cases}\varphi_{t_s}(s - \tau), &m = 1,\\
    \varphi_{t_s}(\tau - s), &m = 0,
    \end{cases}\label{definition_lambda_s}
\end{equation}
 where the scalar variable $\tau$ is kept constant during the flow and is updated only when $\xi$ enters in the jump set $\mathcal{J}$. The update law for $\tau$, which is denoted by $\mathbf{U}(\xi)$ is given by 
\begin{equation}
\mathbf{U}(\xi) = \begin{cases}
        s + \tau_s, & \xi\in\mathcal{J}_{sm}^0,\\
        s + \tau_s + t_s, & \xi\in\mathcal{J}_{sm}^1,
    \end{cases}\label{defintion_tau}
\end{equation}
with $t_s>0$ and $\tau_s > 0$ as defined in \eqref{updatelaw_part1}. 
For $z\in\{0, 1\}$, the jump set $\mathcal{J}_{sm}^{z}$ is defined as
\begin{equation}
    \mathcal{J}_{sm}^{z} := \{\xi\in\mathcal{J}_z\mid s - \tau \geq t_s\},\label{smooth_jump_set_definition}
\end{equation}
where $\mathcal{J}_z$ for $z = 0$ and $z = 1$ is given by \eqref{mtt_jumpset} and \eqref{oa_jumpset}, respectively. 
According to \eqref{defintion_tau} and \eqref{smooth_jump_set_definition}, $\tau$ is updated only when $\xi$ enters in the jump set $\mathcal{J}_{sm}$ and the difference $s- \tau$ is greater than or equal to $t_s$ seconds, where $\mathcal{J}_{sm} = \mathcal{J}_{sm}^0\cup\mathcal{J}_{sm}^1$. This ensures that the consecutive switching instances are separated by at least $t_s$ seconds.
According to \eqref{defintion_tau} and \eqref{smooth_jump_set_definition}, the variable $\tau$ is updated only when $\xi$ enters the jump set $\mathcal{J}$ and the condition $s - \tau \geq t_s$ is satisfied. This ensures that consecutive switching events are separated by at least $t_s$ seconds.

Given a scalar variable $p$, the continuous scalar function $\varphi_{t_s}(p)$ is defined by
\begin{equation}
    \varphi_{t_s}(p) = \begin{cases}
        1, & p \geq t_s,\\
        p/t_s, & 0\leq p \leq t_s,\\
        0, & p \leq 0.
    \end{cases}\label{definition_varphi}
\end{equation}
Depending on the value of $p$, the function $\varphi_{t_s}(p)$ takes values between $0$ and $1$.

Notice that according to \eqref{updatelaw_part1} and \eqref{defintion_tau}, when $\xi\in\mathcal{J}_0$, the variables $s$ and $\tau$ are updated to $s + \tau_s$. 
Since $\tau$ is kept constant during the flow, one has $\dot{\tau} = 0$. Additionally, as per \eqref{hybrid_closed_loop_system}, we have $\dot{s} = 1$. Therefore, as the robot enters in the \textit{obstacle-avoidance} mode, the difference $s - \tau$ equals zero and keeps increasing until the robot switches to the \textit{move-to-target} mode. This, according to \eqref{definition_lambda_s} and \eqref{definition_varphi}, causes the value of $\lambda$ to go from $0$ to $1$ in $t_s$ seconds\footnote{The parameter $t_s$ should be set to a relatively small positive value to ensure that, depending on the current mode of operation, the function $\lambda$ reaches a constant value of either $0$ or $1$ before the state $(\xi, \tau)$ enters the jump set of the respective mode.}.
Similarly, according to \eqref{updatelaw_part1} and \eqref{defintion_tau}, when $\xi\in\mathcal{J}_1$, the variables $s$ and $\tau$ are updated to $s + \tau_s$ and $s + \tau_s + t_s$, respectively. Addtionally, during the flow we have $\dot{s} = 1$ and $\dot{\tau} = 0$. Therefore, as the robot enters in the \textit{move-to-target} mode, the difference $\tau - s$ equals $t_s$ and keeps decreasing until the robot switches to the \textit{obstacle-avoidance} mode. This, according to \eqref{definition_lambda_s} and \eqref{definition_varphi}, causes the value of $\lambda$ to go from $1$ to $0$ in $t_s$ seconds. 
Therefore, by replacing $m$ in \eqref{control_u} with $\lambda(m, s,\tau)$ one can ensure that the control input vector changes its value between $-\kappa_s\mathbf{x}$ and $\kappa_r\mathbf{v}(\mathbf{x}, \mathbf{h}, \mathbf{a})$ in a continuous manner, depending on the current mode of operation. The smoothed version of the hybrid control input is then given by 
\begin{equation}
    \mathbf{u}_{sm}(\xi, \tau) = -\kappa_s(1 - \lambda(m, s, \tau))\mathbf{x} + \kappa_r \lambda(m, s, \tau)\mathbf{v}\label{smoothed_control_u}.
\end{equation}

The resulting hybrid closed-loop system is represented by
\begin{equation}
\underbrace{\begin{matrix}\mathbf{\dot{x}}\\\mathbf{\dot{h}}\\\mathbf{\dot{a}}\\\dot{m}\\\dot{s}\\\dot{\tau}
    \end{matrix}\begin{matrix*}[l]\;=\mathbf{u}_{sm}(\xi, \tau)\\\;=\mathbf{0}\\\;=\mathbf{0}\\\;=0\\\;=1\\\;=0\end{matrix*}}_{(\dot{\xi}, \dot{\tau}) = \mathbf{F}_{sm}(\xi, \tau), {(\xi, \tau)\in\mathcal{F}\times\mathbb{R}}}\quad\quad\underbrace{\begin{matrix}\mathbf{x}^+\\\begin{bmatrix}\mathbf{h}^+\\\mathbf{a}^+\\m^+\\s^+\end{bmatrix}\\\tau^+\end{matrix} \begin{matrix*}[l]=\mathbf{x}\vspace{0.56cm}\\\vspace{0.7cm}\in\mathbf{L}(\xi)\\=\mathbf{U}(\xi)\end{matrix*}}_{(\xi^+, \tau^+) \in\mathbf{J}_{sm}(\xi), {(\xi, \tau)\in\mathcal{J}_{sm}\times\mathbb{R}}}\label{smoothed_hybrid_closed_loop_system}
\end{equation}
where $\mathcal{J}_{sm} = \mathcal{J}_{sm}^0\cup\mathcal{J}_{sm}^1$.
Notice that the variable $\tau$ does not change value during the flow, and is updated according to $\mathbf{U}(\xi)$ \eqref{defintion_tau} only when $\xi\in\mathcal{J}_{sm}$.


\begin{remark}
    The control input vector $\mathbf{u}_{sm}$ allows one to generate smooth robot trajectories that converge to the target location at the origin as long as the robot's location  $\mathbf{x}$ and the mode indicator $m$ are not initialized as follows:
    \begin{itemize}
    \item $\mathbf{x}(0, 0)\in\partial\mathcal{W}_{r_a}\cap\mathcal{F}_1^{\mathcal{W}}$ and $m(0, 0) = 0$. In this case, one cannot ensure the forward invariance of the set $\mathcal{K}\times\mathbb{R}$ as $\mathbf{u}_{sm}(\xi(0, 0), \tau(0, 0))$ at $\mathbf{x}(0, 0)$ is directed towards the interior of set $\mathcal{D}_{r_a}(\mathcal{O}_{\mathcal{W}})$. 
    \item $\mathbf{x}(0, 0)\in\mathcal{F}_{0}^{\mathcal{W}}\setminus\mathcal{D}_{r_a + \gamma}(\mathcal{O}_{\mathcal{W}})$ and $m(0, 0) =1$. In this case, the obstacle-avoidance control input vector $\mathbf{v}(\mathbf{x}, \mathbf{h}, \mathbf{a})$ is not viable as the uniqueness of $\Pi(\mathbf{x}, \mathcal{O}_{\mathcal{W}})$ is not guaranteed.
    \end{itemize}
To avoid such situations, one should initialize the robot's location  $\mathbf{x}$ in the interior of the obstacle-free workspace $\mathcal{W}_{r_a}$ and ensure that the robot starts operating in the \textit{move-to-target} mode with the stabilization control input vector $-\kappa_s\mathbf{x}$. This can be achieved by setting $m(0, 0) = 0$ and choosing $\tau(0, 0) \in(-\infty, s(0, 0))$, as per \eqref{definition_lambda_s} and \eqref{smoothed_control_u}. 
\end{remark}

\begin{algorithm}
\caption{Implementation of the proposed hybrid control law \eqref{hybrid_control_law} in \textit{a priori} unknown environment.}

\begin{algorithmic}[1] 
\STATE \textbf{Set} target location at the origin $\mathbf{0}.$
\STATE \textbf{Initialize} $\mathbf{x}(0, 0)\in\mathcal{W}_{r_a}$, $\mathbf{h}(0, 0) \in \mathcal{W}_{r_a}$, $\mathbf{a}(0, 0) \in \mathbb{S}^{n-1}$, $m(0, 0)\in\mathbb{M}$ and $s(0, 0)\in\mathbb{R}_{\geq 0}$. Choose a sufficiently small value of $\bar{\epsilon}$ according to Lemma \ref{lemma:epsilon_exists}, and initialize $\epsilon\in(0, \bar{\epsilon}].$ Select a minimum safety distance $r_s$ such that $r_s\in(0, \delta)$ and set the parameters $\gamma, \gamma_s$ and $\gamma_a$ such that $0 <\gamma_a < \gamma_s < \gamma < \delta - r_s$. Choose $R_s > r_a + \gamma$, used in \eqref{sensor_visible_boundary}. The parameter $\delta_s$, used in \eqref{J_s_set_definition}, is set such that $\delta_s\in(0, \tau_s)$, where $\tau_s > 0$.
\STATE\textbf{Measure} $\mathbf{x}$ and the set $\eth_{\mathcal{W}}^{\mathbf{x}}$ as defined in \eqref{sensor_visible_boundary}.
\IF{$m = 0$,}
\IF{$d(\mathbf{x}, \eth_{\mathcal{W}}^{\mathbf{x}})\leq r_a+\gamma_s$,}
\STATE\textbf{Identify} the set $\eth_{i}^{\mathbf{x}}\subset\eth_{\mathcal{W}}^{\mathbf{x}}$ as defined in \eqref{visibile_boundary_closest_obstacle}.
\IF{$\eth_i^{\mathbf{x}}\cap\mathcal{C}_{\mathcal{L}}(\mathbf{x}, 2r_a) \ne \emptyset$,}
\STATE\textbf{Update }$\xi\leftarrow\mathbf{J}(\xi)$ using \eqref{updatelaw_part1} and \eqref{hybrid_closed_loop_system}.
\ENDIF
\ENDIF
\ENDIF
\IF{$m=1$,}
\IF{$s = s(0, 0)$,}
\STATE\textbf{Update }$\xi\leftarrow\mathbf{J}(\xi)$ using \eqref{updatelaw_part2} and \eqref{hybrid_closed_loop_system}.
\ELSE
\IF{$d(\mathbf{x}, \eth_{\mathcal{W}}^{\mathbf{x}})\leq r_a + {\gamma},$}
\STATE\textbf{Identify} the set $\eth_{i}^{\mathbf{x}}\subset\eth_{\mathcal{W}}^{\mathbf{x}}$ as defined in \eqref{visibile_boundary_closest_obstacle}.
\IF{$\eth_i^{\mathbf{x}}\cap\mathcal{C}_{\mathcal{L}}(\mathbf{x}, 2r_a) = \emptyset$,}
\IF{$\norm{\mathbf{x}} \leq \norm{\mathbf{h}} - \epsilon,$}
\STATE\textbf{Update }$\xi\leftarrow\mathbf{J}(\xi)$ using \eqref{updatelaw_part2} and \eqref{hybrid_closed_loop_system}.
\ENDIF
\ENDIF
\ELSE
\STATE\textbf{Update }$\xi\leftarrow\mathbf{J}(\xi)$ using \eqref{updatelaw_part2} and \eqref{hybrid_closed_loop_system}.
\ENDIF
\ENDIF
\ENDIF
\STATE\textbf{Execute }$\mathbf{F}(\xi)$ given \eqref{hybrid_control_law}, used in \eqref{hybrid_closed_loop_system}.
\STATE\textbf{Go to} step 3.
\end{algorithmic}

\label{alg:general_implemenration}
\end{algorithm}

\section{Simulation Results}
\label{section:simulation}

We consider an unbounded workspace \textit{i.e.}, obstacle $\mathcal{O}_0 = \emptyset$, with $2$ three-dimensional, convex obstacles, as shown in Fig. \ref{figure:result_1}. We apply the proposed hybrid feedback controller \eqref{hybrid_control_law} for the robot initialized at $8$ different locations in the obstacle-free workspace. The target is located at the origin. The radius of the robot is set to $0.2 m$. The minimum safety distance $r_s = 0.1m$ and the parameter $\gamma = 0.5m$. The gains $\kappa_s$ and $\kappa_r$, used in \eqref{control_u}, are set to $1$ and $2$, respectively. The sensing radius $R_s$, used in \eqref{sensor_visible_boundary}, is set to $2m$. The parameter $\epsilon$, used in \eqref{partition_rm}, is set to $0.5m$. From Fig. \ref{figure:result_1}, it can be observed that the robot converges to the target location while simultaneously avoiding collisions with the obstacles. Fig. \ref{figure:result_2} shows that the center of the robot stays at least $r_a$ meters away from the boundary of the obstacles. 

\begin{figure}[ht]
    \centering
    \includegraphics[width = 0.7\linewidth]{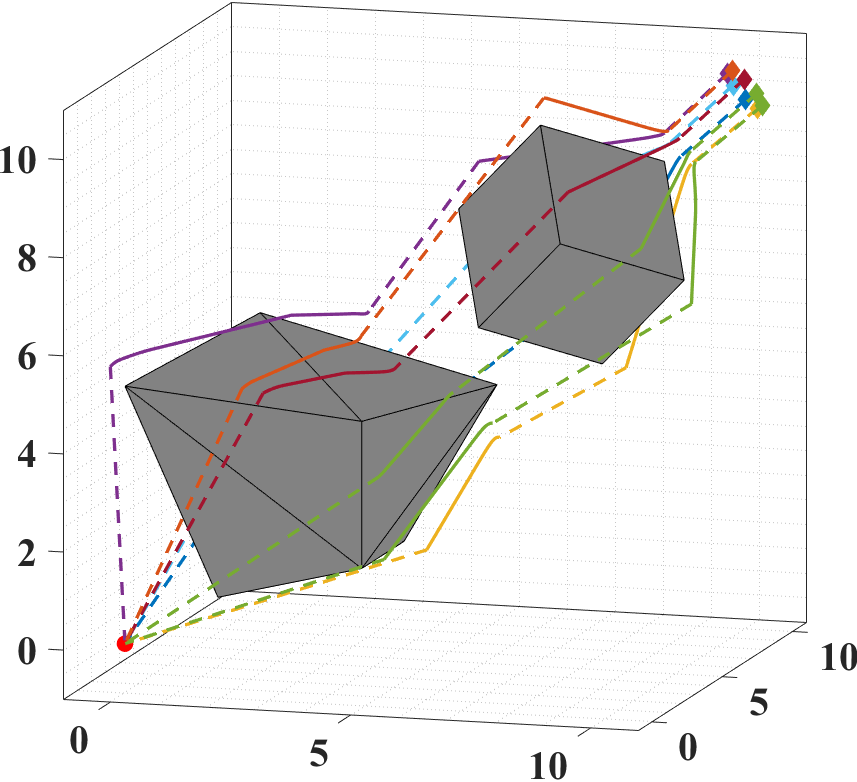}
    \caption{Robot trajectories starting from different locations.}
    \label{figure:result_1}
\end{figure}

\begin{figure}[ht]
    \centering
    \includegraphics[width = 1\linewidth]{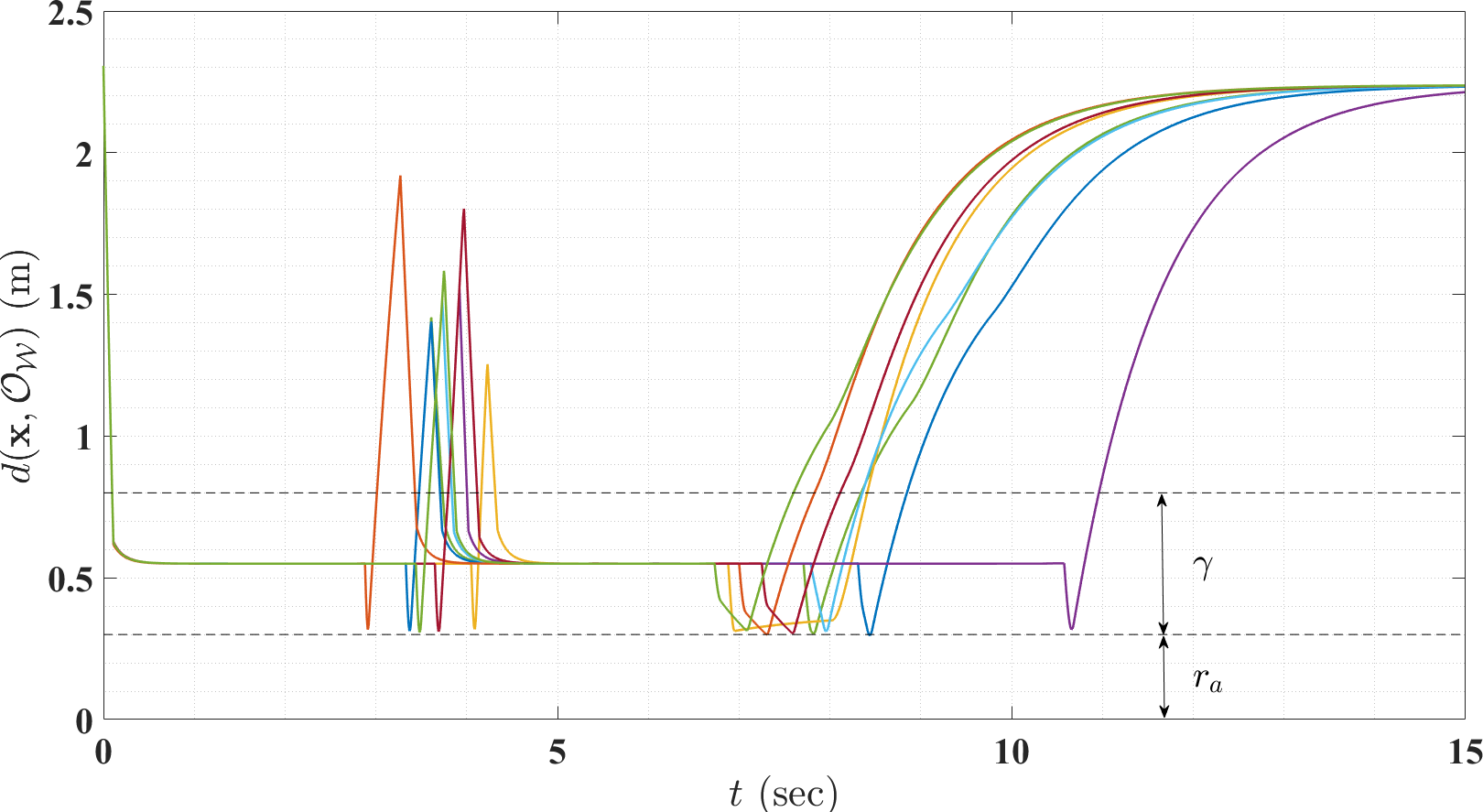}
    \caption{Distance of the center of the robot $\mathbf{x}$ from the boundary of the obstacle occupied workspace $\mathcal{O}_{\mathcal{W}}$.}
    \label{figure:result_2}
\end{figure}

We now consider a two-dimensional workspace containing 12 convex obstacles. A point robot is initialized at 24 different locations along the workspace boundary, as depicted in Fig. \ref{multiple_initial_2D}, with the target located at the origin. The safety distance $r_s$ is set to $0.1\,m$. The parameters $\gamma$ and $\epsilon$ are set to $0.2\,m$ and $0.5\,m$, respectively. The gain parameters are $\kappa_s = 1$ and $\kappa_r = 2$, and the sensing radius $R_s$ is set to $1\,m$. In Fig. \ref{multiple_initial_2D}, dashed trajectories represent the robot's motion in the \textit{move-to-target} mode, while solid trajectories indicate motion in the \textit{obstacle-avoidance} mode. Under the hybrid feedback controller \eqref{hybrid_control_law}, the robot safely converges to the target from all initial locations.

\begin{figure}[ht]
    \centering
    \includegraphics[width=1\linewidth]{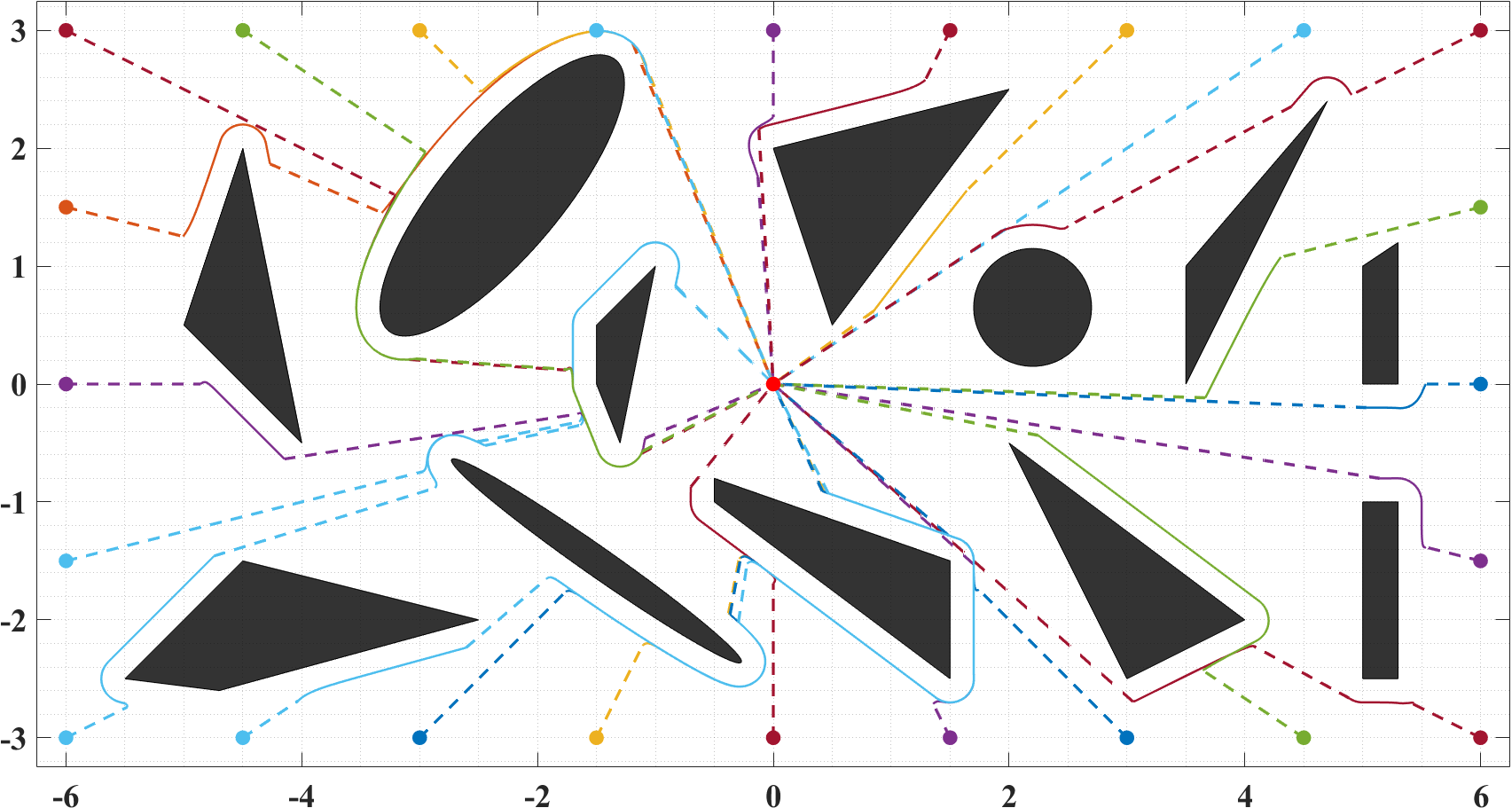}
    \caption{Safe convergence of robot trajectories to the target location.}
    \label{multiple_initial_2D}
\end{figure}

Next, we examine the effect of varying the sensing radius $R_s$ on the robot's trajectories. The robot with radius $r = 0.15m$ is initialized at $[0, 9.5]^\top$ in a two-dimensional workspace with a single convex obstacle, as shown in Fig. \ref{effect_of_Rs_on_trajectories}. The target location is set at the origin. The safety distance $r_s$ is set to $0.05m$. The parameters $\gamma$ and $\epsilon$ are set to $1.5m$ and $0.5m$, respectively. The gains are $\kappa_s = 1$ and $\kappa_r = 2$. In Fig. \ref{effect_of_Rs_on_trajectories}, the blue portions of the trajectories correspond to the motion of the robot in \textit{move-to-target} mode, while red portions correspond to the motion in the \textit{obstacle-avoidance} mode. 
When the sensing radius is relatively small, the robot cannot detect all points on the boundary of obstacle \( \mathcal{O}_i \) that are within its line of sight but outside the sensing region $\mathcal{B}_{R_s}(\mathbf{x})$. As a result, the robot has only partial information about the obstacle's boundary. Consequently, the robot may switch to the \textit{move-to-target} mode even when it does not have a clear line of sight to the target, as shown in Figs. \ref{R170}, \ref{R250}, and \ref{R340}. 

\begin{figure}[ht]
\centering
\subfloat[][]{\includegraphics[width =0.25\linewidth]{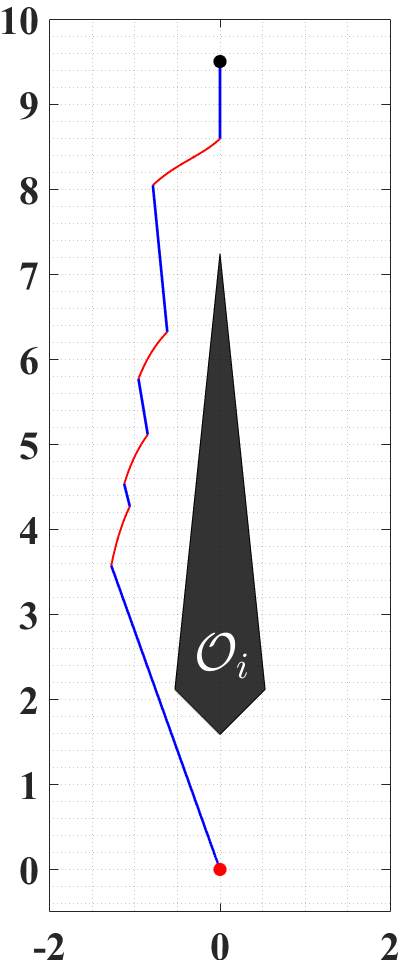}\label{R170}}
\subfloat[][]{\includegraphics[width =0.25\linewidth]{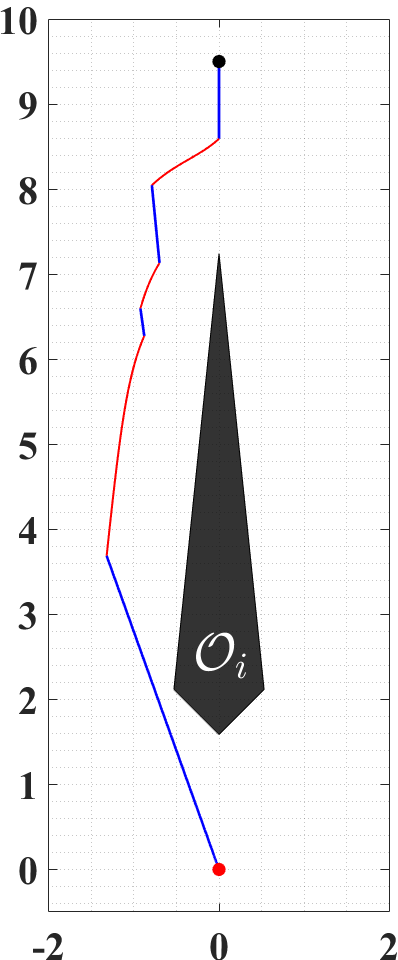}\label{R250}}
\subfloat[][]{\includegraphics[width =0.25\linewidth]{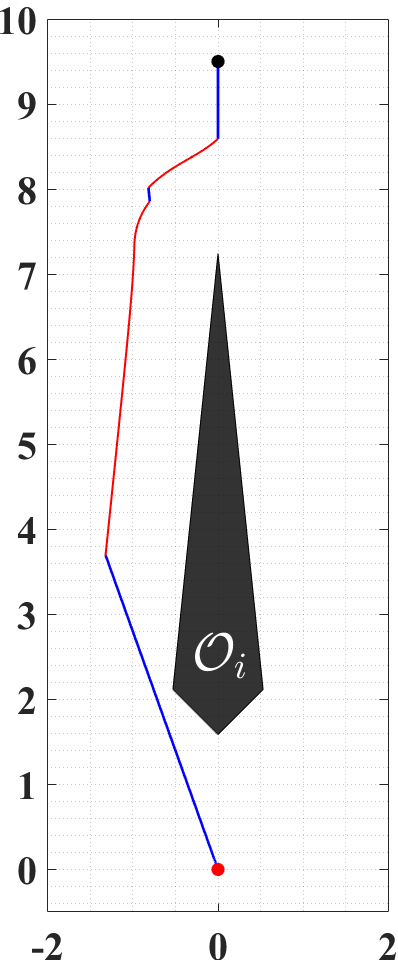}\label{R340}}
\subfloat[][]{\includegraphics[width =0.25\linewidth]{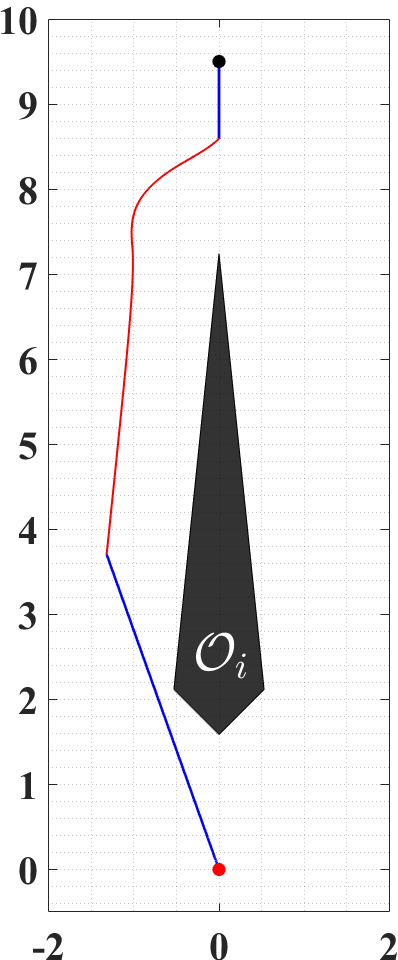}\label{R400}}
\caption{Robot trajectories for different values of the sensing radius $R_s$: (a) $R_s = 1.7m$, (b) $R_s = 2.5m$, (c) $R_s = 3.4m$, (d) $R_s = 4m$.}
\label{effect_of_Rs_on_trajectories}
\end{figure}

We next provide a comparison with the separating hyperplane approach developed in \cite{arslan2019sensor}. Similar to our approach, this approach can be implemented in \textit{a priori} unknown environments using the information obtained via a range-bearing sensor mounted on the robot. Contrary to our approach, this approach only works for convex obstacles that satisfy the obstacle curvature condition \cite[Assumption 2]{arslan2019sensor}. 
When the workspace consists of obstacles that do not satisfy the obstacle curvature condition, the separating hyperplane approach generates trajectories that converge to an undesired equilibrium (local minimum), as shown in Fig. \ref{comparison:hyperplane_not_working}. On the other hand, the proposed hybrid feedback controller guarantees safe, global asymptotic convergence to the target location, as seen in Fig. \ref{comparison:hybrid_working}. 
The complete simulation video can be found at\footnote{[Online]. Available: \url{https://youtu.be/R6OowaF6dTc}}.

\begin{figure}
\centering
\subfloat[][]{\includegraphics[width =0.492\linewidth]{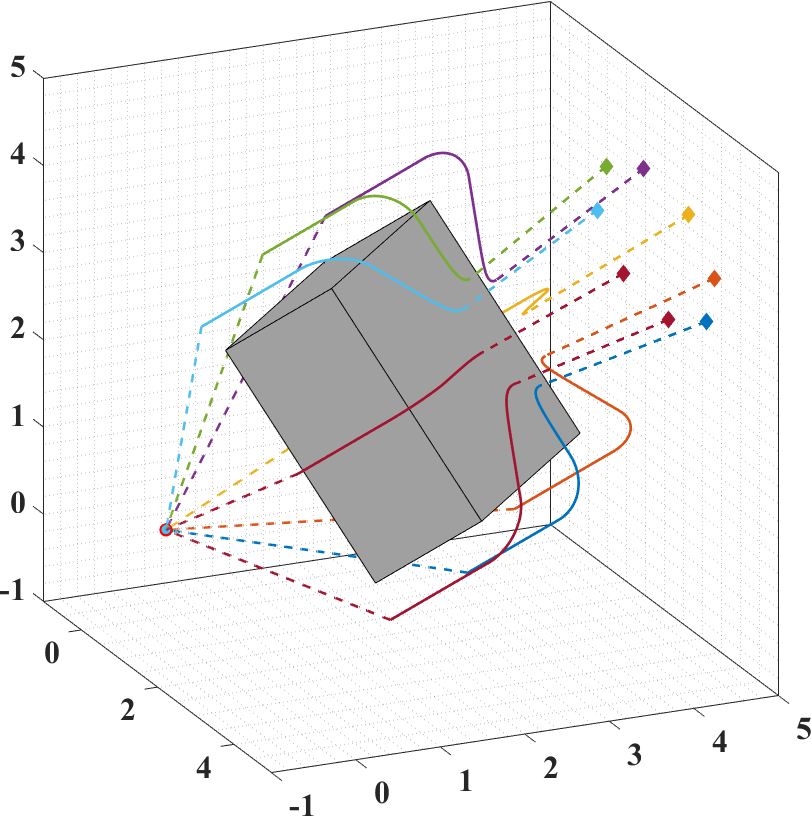}\label{comparison:hybrid_working}}
\subfloat[][]{\includegraphics[width =0.492\linewidth]{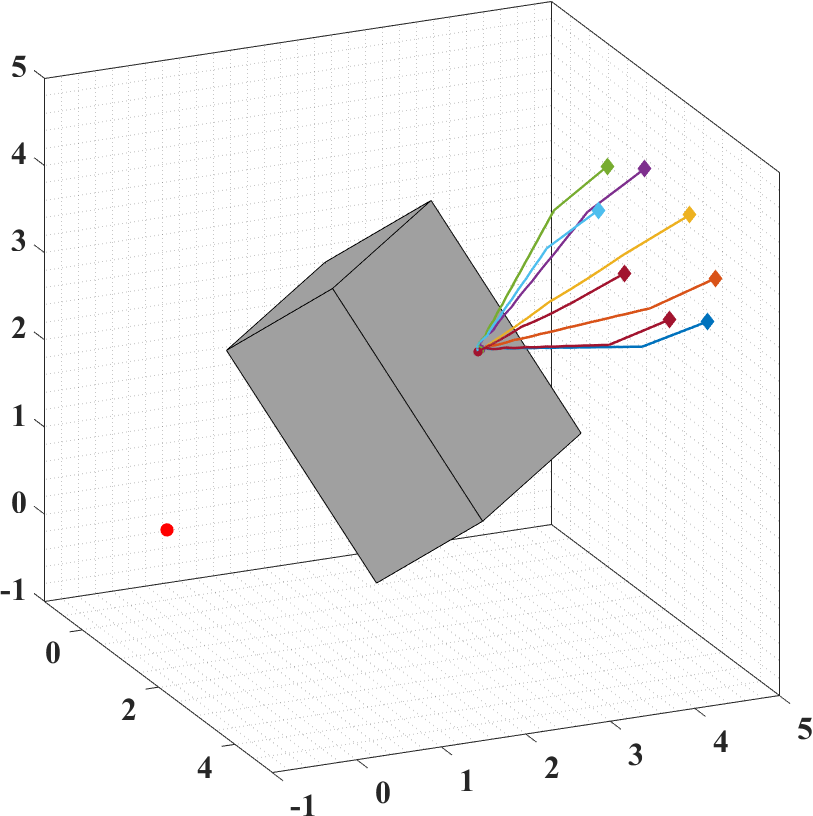}\label{comparison:hyperplane_not_working}}
\caption{(a) Robot trajectories obtained using our proposed hybrid feedback approach. (b) Robot trajectories obtained using the separating hyperplane approach \cite{arslan2019sensor}. }
\label{figure:comparison}
\end{figure}

\begin{figure}[ht]
    \centering
    \includegraphics[width = 0.8\linewidth]{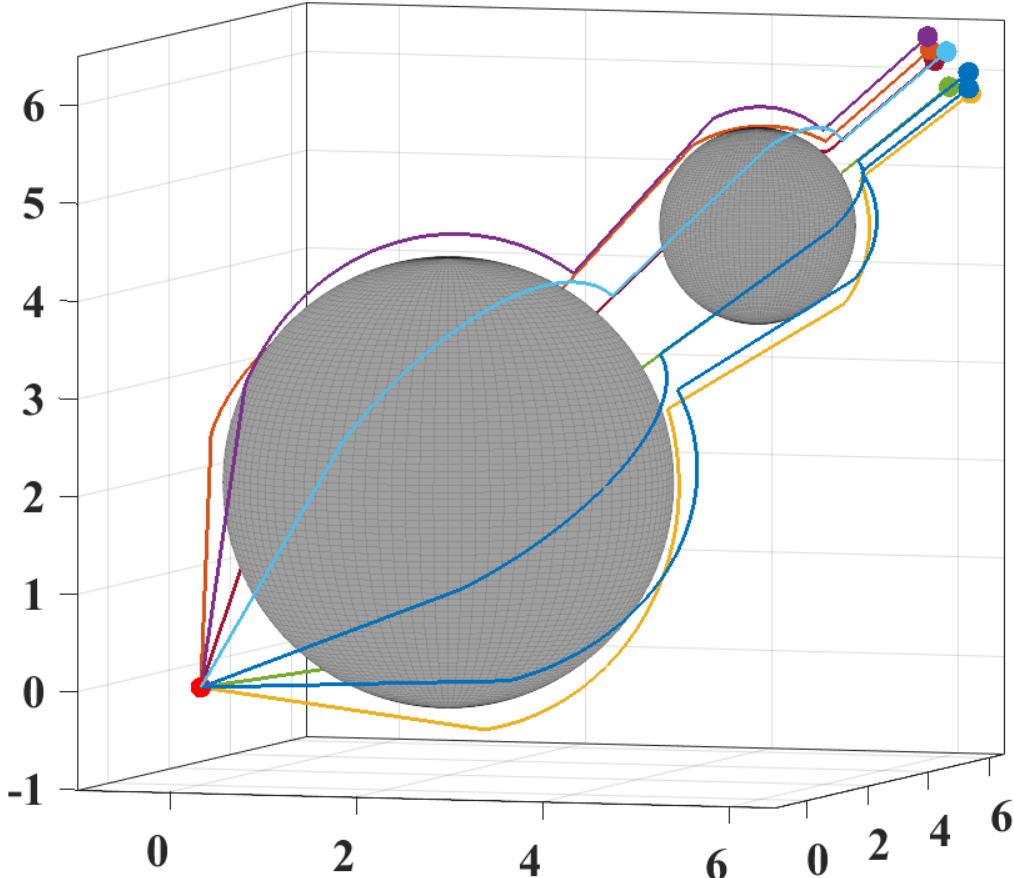}
    \caption{Robot trajectories safely navigating around spherical obstacles and converging to the target location at the origin.}
    \label{result_sphere_world}
\end{figure}

In the next simulation scenario, as shown in Fig. \ref{result_sphere_world}, we consider 2 three-dimensional spherical obstacles and apply the proposed hybrid controller \eqref{hybrid_control_law} with modifications mentioned in Section \ref{sec:sphere_world}, for a point robot initialized at $8$ different locations in the obstacle-free workspace. The safety distance $r_s$ is set to $0.15m$ and the parameter $\gamma = 0.15 m.$ The gains $\kappa_s$ and $\kappa_r$, used in \eqref{control_u}, are set to $0.5$. The sensing radius $R_s$, used in \eqref{sensor_visible_boundary}, is set to $2m$. The parameter $\epsilon$, used in \eqref{partition_rm}, is set to $0.05m$. Similar to the previous simulations, it can be observed from Fig. \ref{result_sphere_distance} that the robot converges asymptotically to the target location without colliding with the obstacles. Since the obstacles are spheres, the distance between the robot and the target location monotonically decreases as the robot converges to the target location, as stated in Theorem \ref{theorem:sphere_global_stability} and as shown in Fig. \ref{sphere_target_distance}.

\begin{figure}[ht]
    \centering
    \subfloat[][]{\includegraphics[width = 0.48\linewidth]{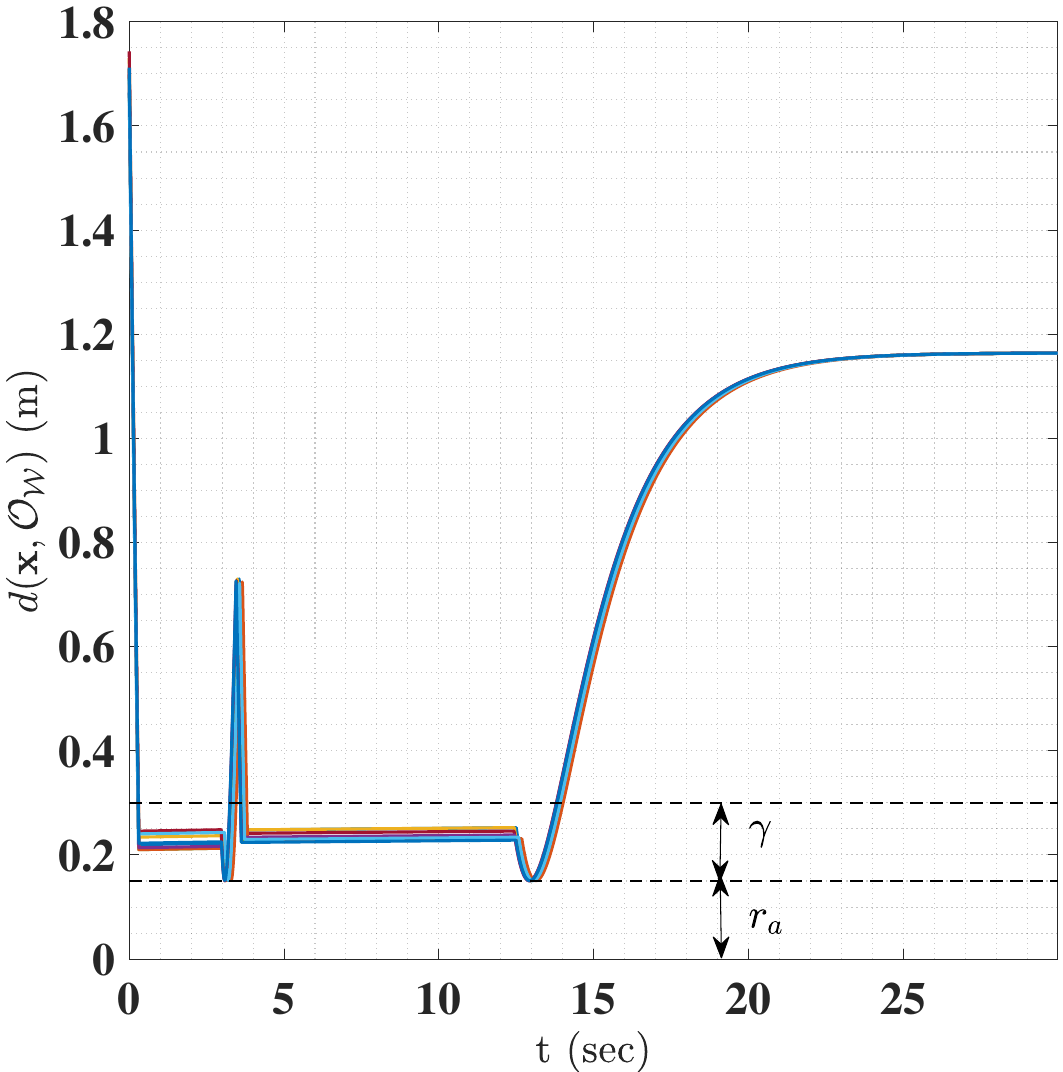}\label{sphere_obs_distance}}
    \subfloat[][]{\includegraphics[width = 0.48\linewidth]{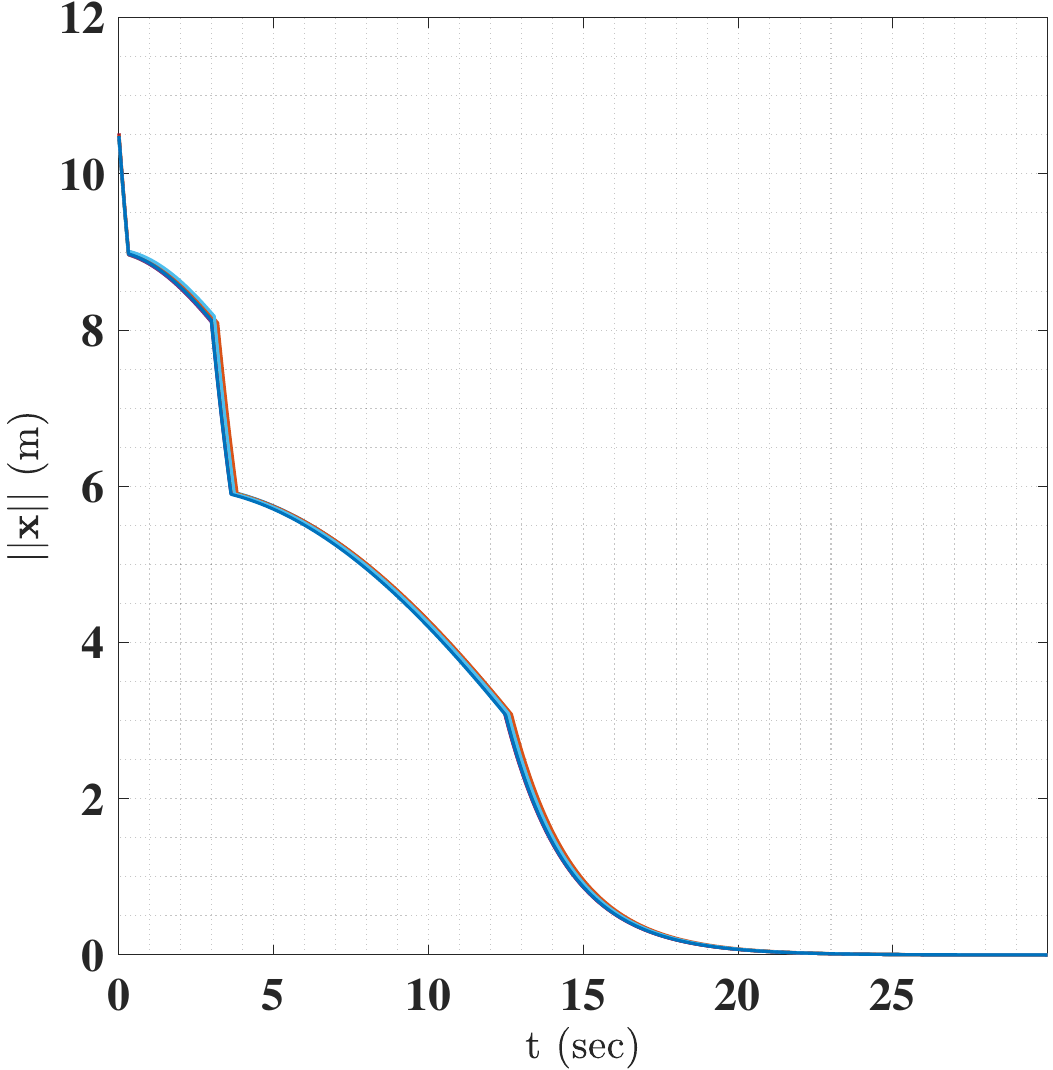}\label{sphere_target_distance}}
    \caption{(a) Distance between the robot $\mathbf{x}$ and the obstacles $\mathcal{O}_{\mathcal{W}}$ versus time. (b) Distance between the robot $\mathbf{x}$ and the target location versus time.}
    \label{result_sphere_distance}
\end{figure}

\begin{figure}[ht]
    \centering
    \subfloat[][]{\includegraphics[width = 0.48\linewidth]{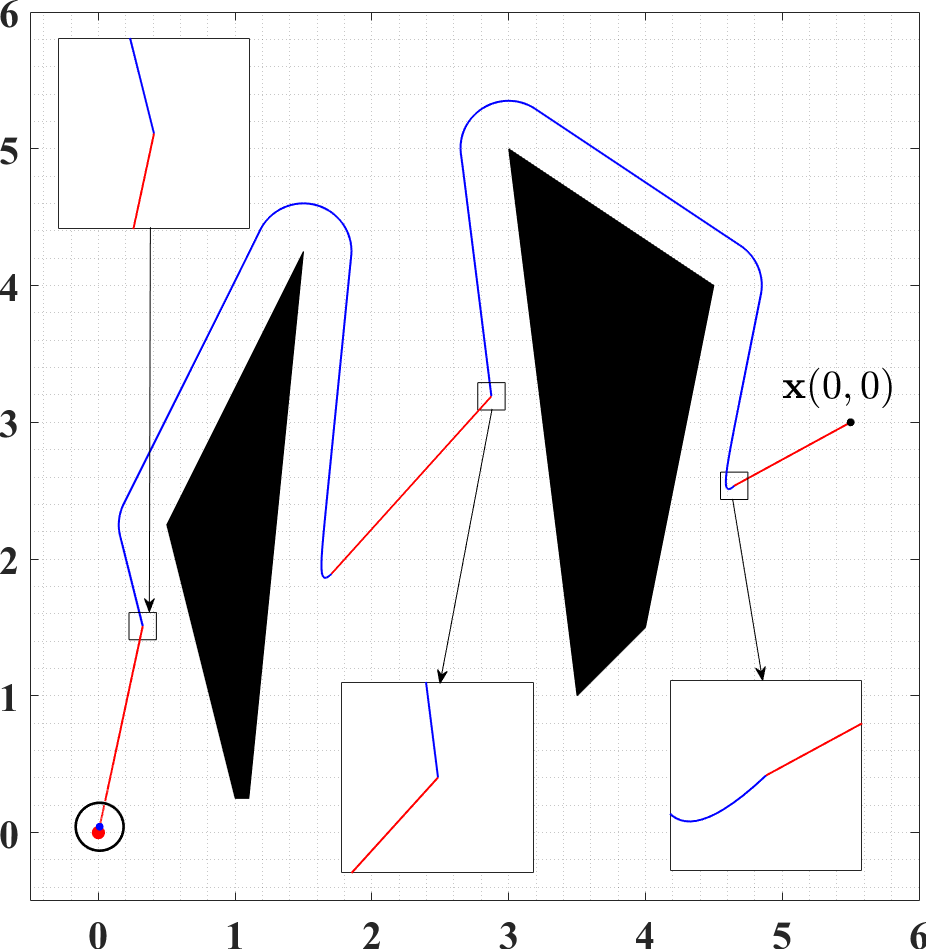}\label{non_smooth_trajectory}}
    \subfloat[][]{\includegraphics[width = 0.48\linewidth]{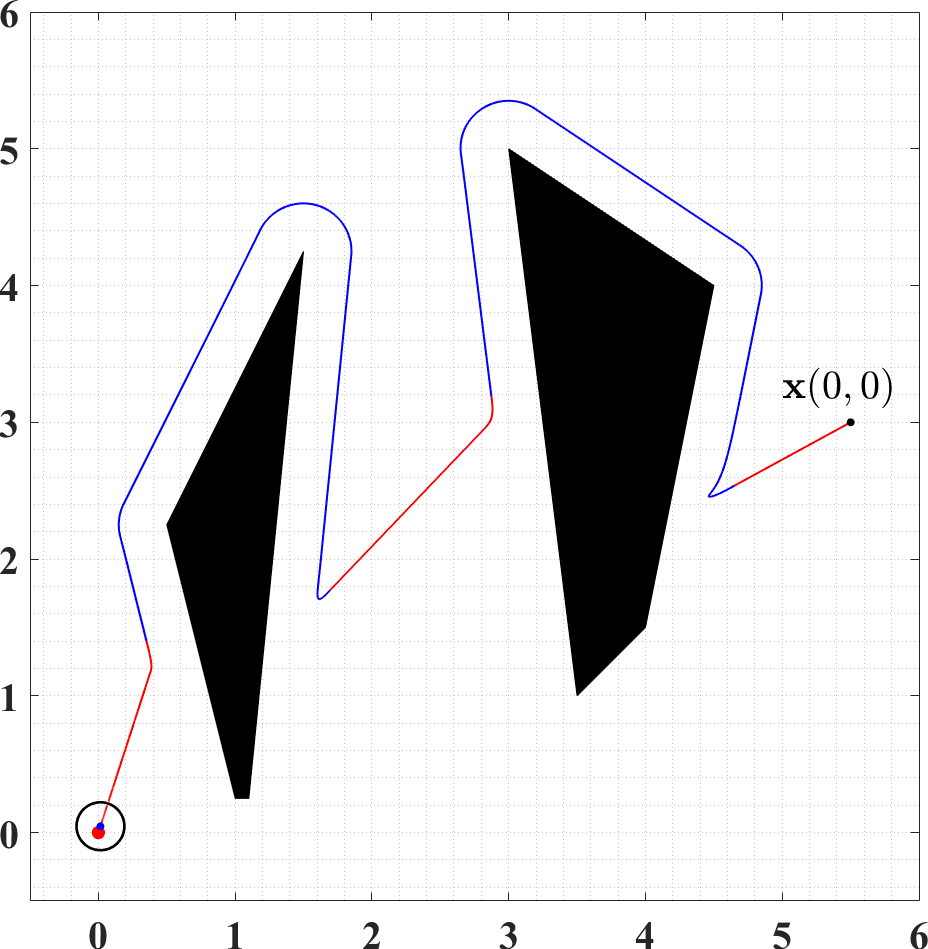}\label{smooth_trajectory}}
    \caption{(a) Robot trajectory for the hybrid closed-loop system \eqref{hybrid_closed_loop_system} using the control input vector $\mathbf{u}(\xi)$. (b) Robot trajectory for the smoothed hybrid closed-loop system \eqref{smoothed_hybrid_closed_loop_system} using the control input vector $\mathbf{u}_{sm}(\xi, \tau)$.}
\label{smooth_non_smooth_tarjectories}
\end{figure}

We compare the robot trajectories generated by the hybrid closed-loop system \eqref{hybrid_closed_loop_system} with the smoothed robot trajectories obtained from the modified hybrid closed-loop system \eqref{smoothed_hybrid_closed_loop_system}.
The environment contains two convex obstacles and the target location is positioned at the origin, as shown in Fig. \ref{smooth_non_smooth_tarjectories}. The robot, with a radius of $r = 0.175 m$, is initialized at the point $[5.5, 3]^\top$. The safety distance $r_s$ is set to $0.025m$, and the parameter $\gamma$ is set to $0.3m$. The control gains $\kappa_s$ and $\kappa_r$ are set to 0.2 and 2, respectively. The sensing radius $R_s$ is set to $3m$, and the parameter $\epsilon$ is set to $1m$. The time parameters $t_s$ and $\tau_s$, used in \eqref{defintion_tau}, are set to $0.5$ seconds and $1$ second, respectively. 

In Fig. \ref{smooth_non_smooth_tarjectories}, the red portions of the trajectories represent the motion of the robot in the \textit{move-to-target} mode, whereas the blue portions correspond to the motion in the \textit{obstacle-avoidance} mode. Figure \ref{norm_control} shows the time-evolution of the magnitude of the control input, while Fig. \ref{angle_control} depicts the direction of the control input represented by the angle between the control input vector and the $x$-axis. Given $\mathbf{p} = [p_1, p_2]^\top \in \mathbb{R}^2$, the angle between $\mathbf{p}$ and the $x$-axis is computed as:
\begin{equation} \mathbf{atan2v}(\mathbf{p}) = \mathbf{mod }(\mathbf{atan2}(p_2/p_1) + 2\pi, 2\pi). \end{equation}
In Fig. \ref{norm_control} and Fig. \ref{angle_control}, the red trajectories represent the control input vector $\mathbf{u}$, while the blue trajectories correspond to $\mathbf{u}_{sm}$.

Generally, the magnitude and direction of the control input vector $\mathbf{u}$ change instantaneously when switching between modes, as illustrated in Fig. \ref{norm_control} and Fig. \ref{angle_control}, respectively. This can sometimes cause an abrupt change in the robot's direction of motion when the control input $\mathbf{u}$ switches between modes, as shown in Fig. \ref{non_smooth_trajectory}. 
On the other hand, the magnitude and direction of the control input vector $\mathbf{u}_{sm}$, used in the modified hybrid closed-loop system \eqref{smoothed_hybrid_closed_loop_system}, change continuously, as depicted in Fig. \ref{norm_control} and Fig. \ref{angle_control}, respectively. Therefore, the control input vector $\mathbf{u}_{sm}$ generates smooth robot trajectories, as shown in Fig. \ref{smooth_trajectory}.

\begin{figure}
    \centering
    \includegraphics[width = 1\linewidth]{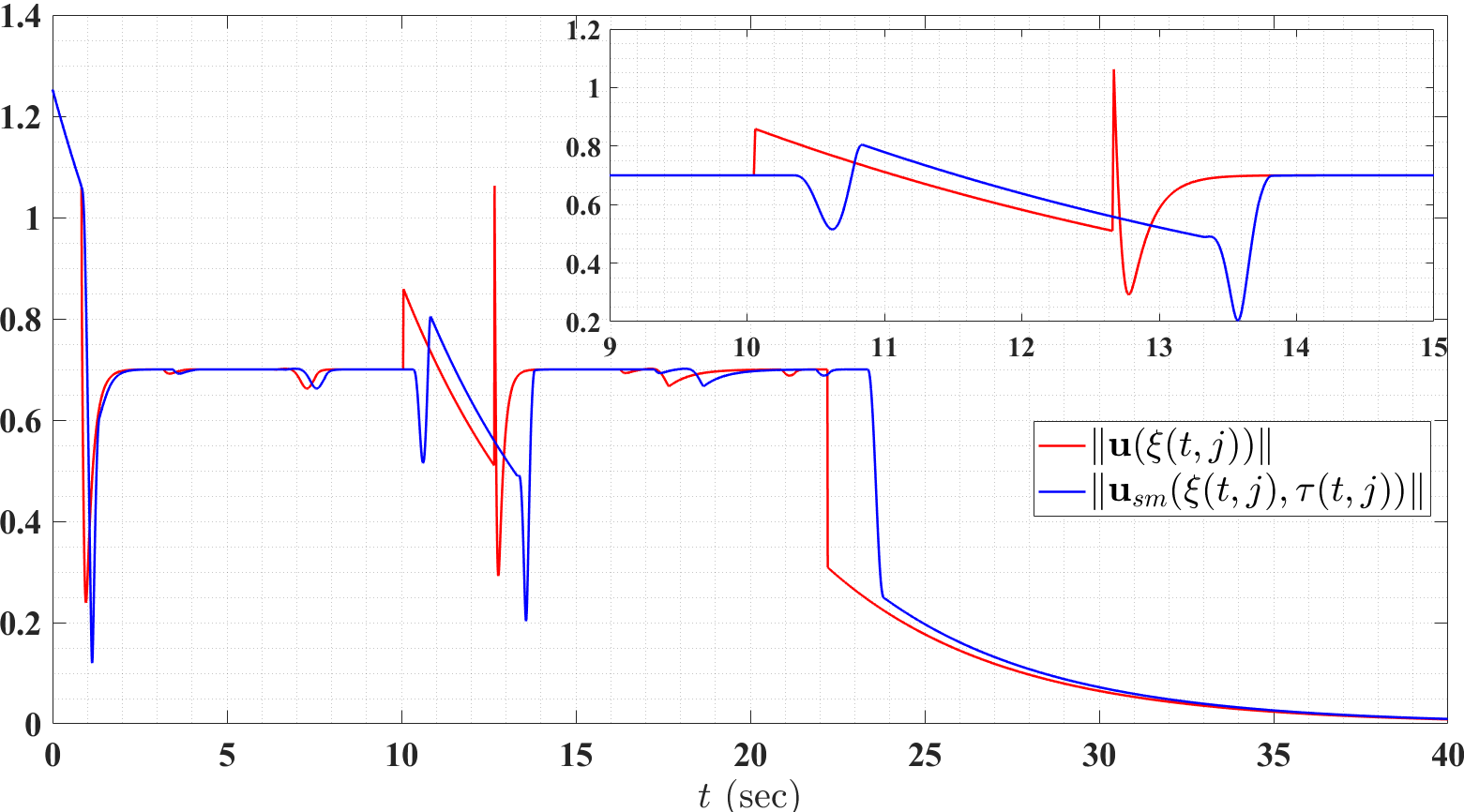}
    \caption{Magnitude of the control input vector versus time. The red trajectory represents $\|\mathbf{u}\|$, while the blue trajectory represents $\|\mathbf{u}_{sm}\|$.}
    \label{norm_control}
\end{figure}

\begin{figure}
    \centering
    \includegraphics[width = 1\linewidth]{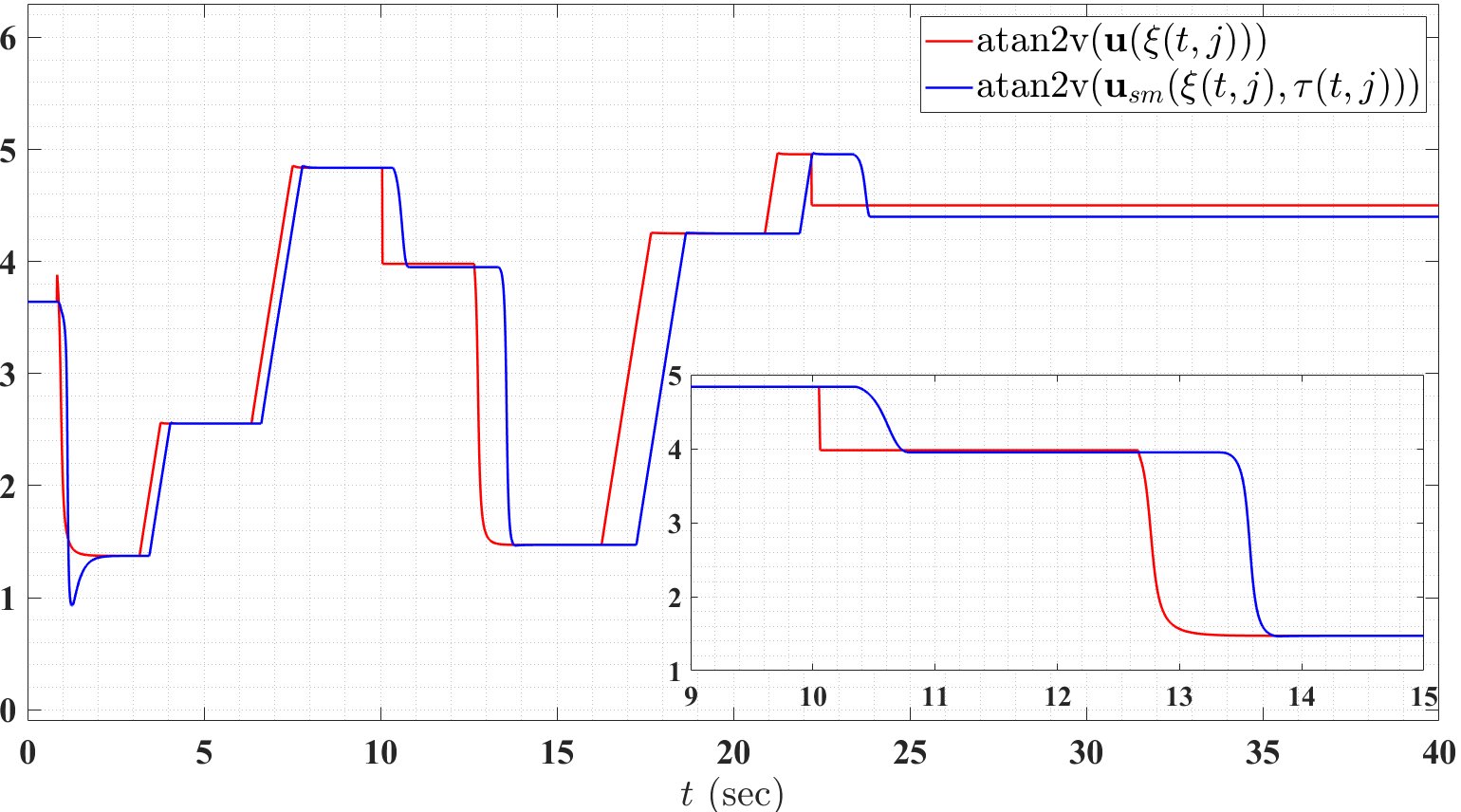}
    \caption{Plot of the angle between the control input vectors and the positive $x$-axis versus time. The red trajectory represents $\mathbf{atan2v}(\mathbf{u})$, while the blue trajectory represents $\mathbf{atan2v}(\mathbf{u}_{sm})$.}
    \label{angle_control}
\end{figure}

\section{Experimental validation}\label{section:experimental_validation}
{We use the TurtleBot 4 Standard mobile robot shown in Fig. \ref{TBT4} to implement the proposed control input $\mathbf{u}_{sm}$ defined in \eqref{smoothed_control_u}.
Since TurtleBot 4 accepts velocity commands in the form of linear and angular velocities, its motion is represented using the unicycle model:
\begin{equation}\label{unicycle_model}
    \begin{aligned}
    \dot{\mathbf{x}} &= v\begin{bmatrix}\cos(\theta) & \sin(\theta)\end{bmatrix}^\top,\\
    \dot{\theta} &=\omega,
    \end{aligned}
\end{equation}
where $\theta\in(-\pi, \pi]$ is the heading angle of the robot relative to the positive $x$-axis of the world frame, and $v$ and $\omega$ are the linear and angular velocity inputs.

\begin{figure}
    \centering
    \includegraphics[width=0.55\linewidth]{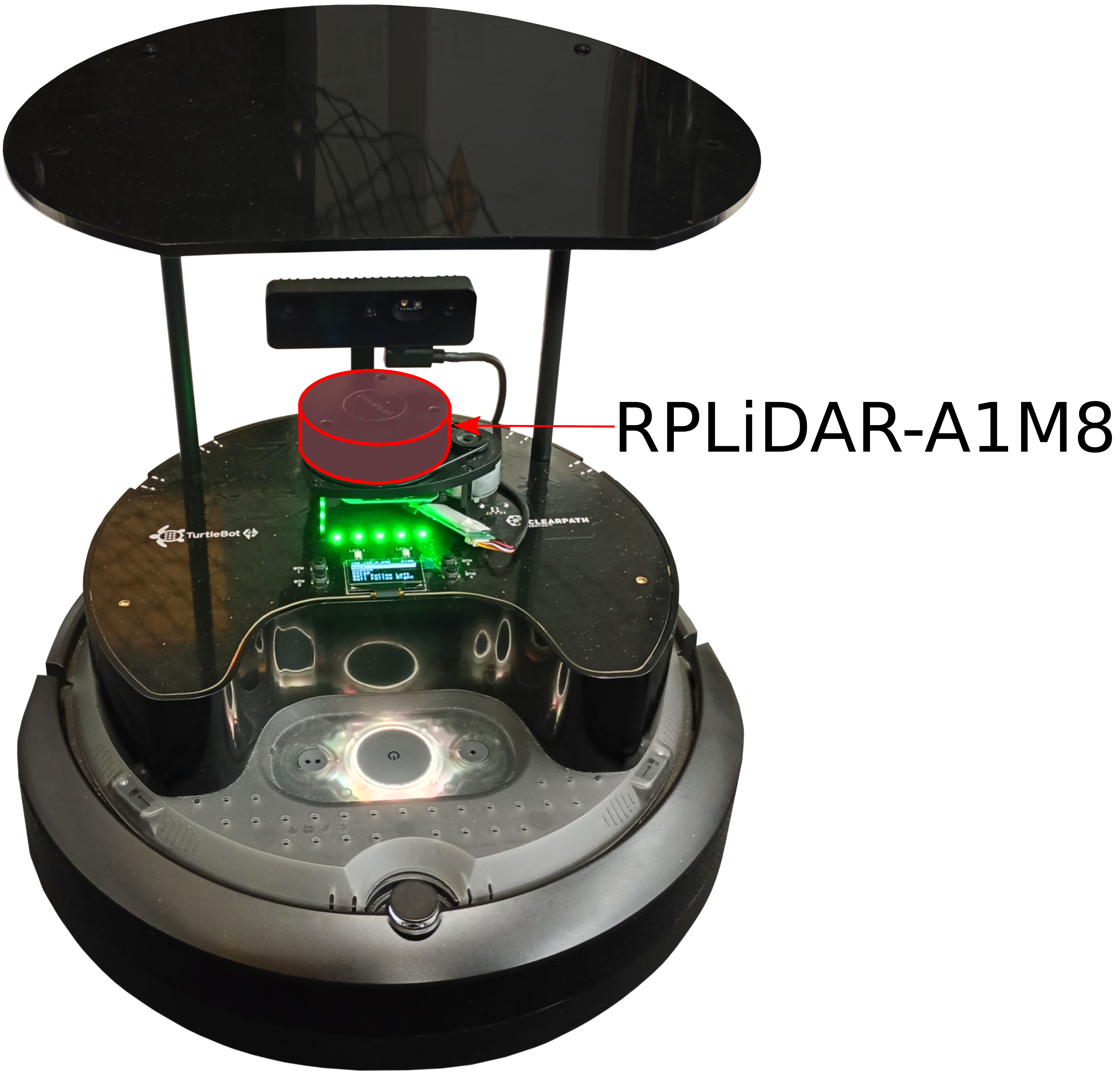}
    \caption{TurtleBot 4 Standard mobile robot.}
    \label{TBT4}
\end{figure}

The proposed control input $\mathbf{u}_{sm}$ defined in \eqref{smoothed_control_u} is used to obtain $v$ and $\omega$ as follows:
\begin{equation}\label{linear_and_angular_velocity_computation}
    \begin{aligned}
        v &= \min\left\{v_{\max}, k_v\|\mathbf{u}_{sm}\|\left(\cos\bigg(\frac{\Delta\theta}{2}\bigg)\right)^{2p}\right\},\\
        \omega &= \min\left\{\omega_{\max}, \max\left\{-\omega_{\max}, k_{\omega}\sin\left(\frac{\Delta\theta}{2}\right)\right\}\right\} ,
    \end{aligned}
\end{equation}
where $k_v>0$ is the linear velocity gain, $k_{\omega} > 0$ is the angular velocity gain, $v_{\max}>0$ is the upper bound on the linear velocity, and $\omega_{\max}>0$ is the upper bound on the absolute angular velocity.
The heading error $\Delta\theta\in(-\pi, \pi]$ is calculated as
\begin{equation}
    \Delta\theta = \operatorname{wrap}_{(-\pi, \pi]}(\theta_d - \theta),
\end{equation}
where $\theta_d = \operatorname{atan2}(u_{sm2}, u_{sm1})$ is the desired heading angle and $\mathbf{u}_{sm} = \left[u_{sm1}, u_{sm2}\right]^\top$.
If $\theta_d - \theta \notin (-\pi,\pi]$, then $\operatorname{wrap}_{(-\pi,\pi]}(\theta_d-\theta)$ maps the angle to its equivalent representative in $(-\pi,\pi]$.
The term $\left(\cos\left(\frac{\Delta\theta}{2}\right)\right)^{2p}$ in \eqref{linear_and_angular_velocity_computation} reduces the linear velocity of the Turtlebot 4 as the heading error $\Delta\theta$ increases.

\subsection{Experimental settings}
The experiments were conducted on TurtleBot 4 Standard platform using ROS 2 Humble on Ubuntu 22.04.
The control algorithm was executed on an external computer equipped with an Intel(R) Core(TM) i7-8550U CPU @ 1.80 GHz and 16 GB RAM, running Ubuntu 22.04 and ROS 2 Humble.
The computer communicated with the TurtelBot 4 over WiFi.
The robot provided odometry through the $\operatorname{/odom}$ topic and LiDAR measurements through the $\operatorname{/scan}$ topic, while the computer sent velocity commands through the $\operatorname{/cmd\_vel}$ topic.

Since the LiDAR measurements are expressed in the LiDAR frame, they were transformed into the robot frame before being used to implement the control law.
The workspace contains five convex obstacles as shown in Fig. \ref{experiment:real_workspace}.
The initial position of the TurtbleBot 4 is set as the origin of the inertial frame, with its heading aligned with the positive $x$-axis.
The target is set to $[4.2, 2.4]^\top$ m.
The parameters involved in \eqref{smoothed_control_u} and \eqref{linear_and_angular_velocity_computation} that are used in the experiment are summarized in Table \ref{experiment:parameter_table}.
\begin{table}[]
    \centering
    \begin{tabular}{|c|c|}
      \hline Parameter &  Value \\\hline
       Robot radius + safety distance $r_a$  & $0.18$ m\\\hline
       Sensing radius $R_s$ in \eqref{sensor_visible_boundary} & $1$ m\\\hline
       $\gamma_a$, $\gamma_s$ and $\gamma$ in \eqref{scalar_function_eta_definition} & $0.03$, $0.15$ and $0.17$ m\\\hline
       $\epsilon$ in \eqref{partition_rm} & $0.1$ m\\\hline
       $\delta_s$ and $\tau_s$ used in \eqref{J_s_set_definition} & $0.15$ and $0.2$ seconds\\\hline
       Gains $\kappa_s$ and $\kappa_r$ in \eqref{smoothed_control_u} & $1$ and $1$\\\hline
       Gains $k_v$ and $k_{\omega}$ in \eqref{linear_and_angular_velocity_computation}& $1$ and $2$\\\hline
       Maximum linear velocity $v_{\max}$ & $0.1$ m/s\\\hline
       Maximum angular velocity $\omega_{\max}$ & $1.9$ rad/s\\\hline
       $p$ in \eqref{linear_and_angular_velocity_computation} & 5\\\hline
    \end{tabular}
    \caption{Control parameters used in the experimental setup}
    \label{experiment:parameter_table}
\end{table}

\subsection{Experimental results}
The results in Fig. \ref{experiment:real_workspace} and Fig. \ref{experiment:time_profiles} illustrate the safe navigation of the TurtleBot 4 from the initial position to the target position.
The video of the experiment is available online\footnote{[Online]. Available: \url{https://youtu.be/kOQRXWgXRUM}}.
Fig. \ref{experiment:real_workspace} shows a long-exposure-style trail of the TurtleBot 4 trajectory, obtained by retaining the highest color intensity at each pixel over the duration of the experiment.
In Fig. \ref{simulated_workspace}, the red squares indicate the positions of hit points where the control input switches from the \textit{move-to-target} mode to the \textit{obstacle-avoidance} mode.
The purple squares indicate the positions where the control input switches from the \textit{obstacle-avoidance} mode back to the \textit{move-to-target} mode.

\begin{figure}
    \centering
    \subfloat[][]{\includegraphics[width = 1\linewidth]{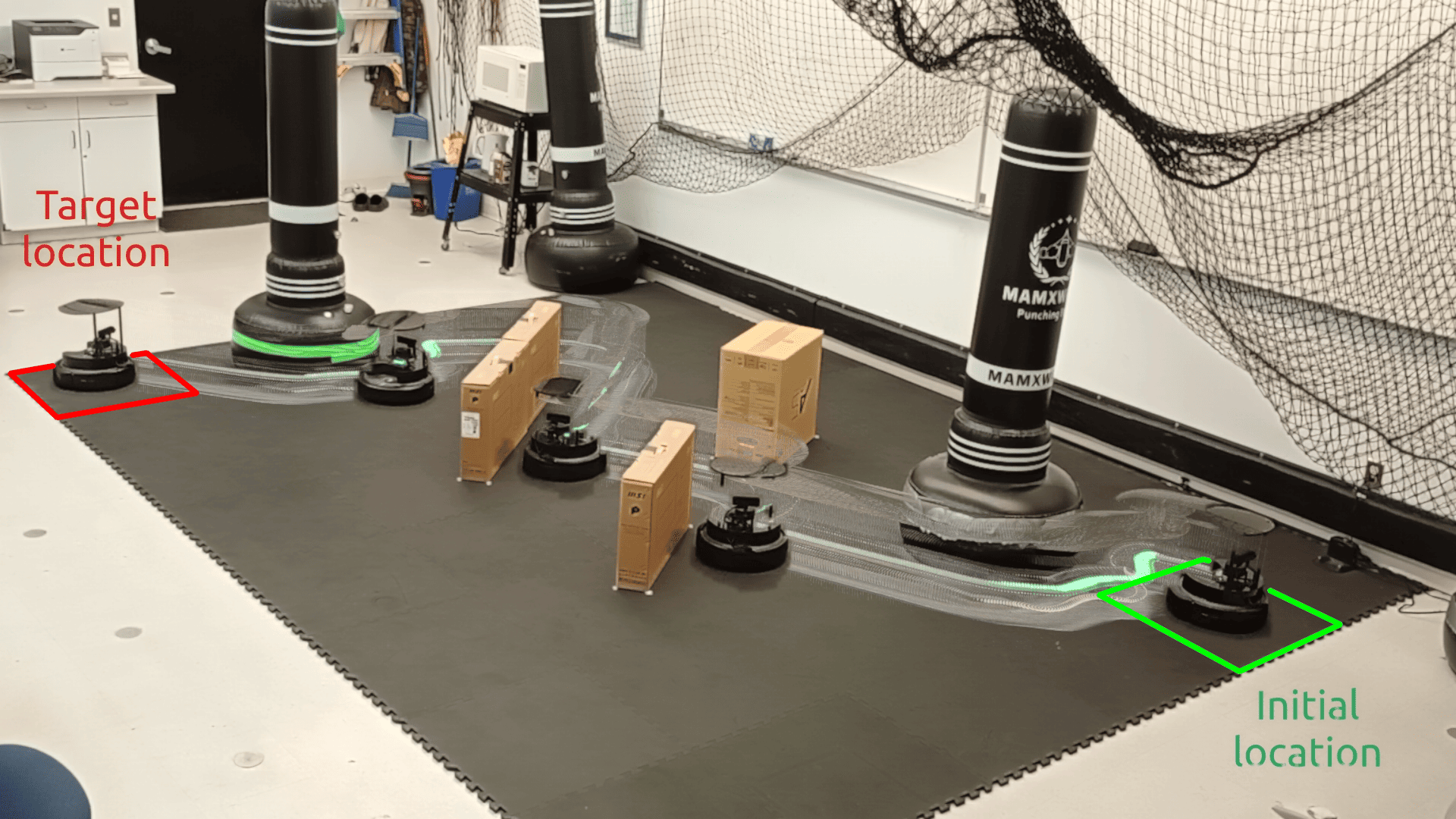}\label{real_workspace}}\\
    \subfloat[][]{\includegraphics[width = 1\linewidth]{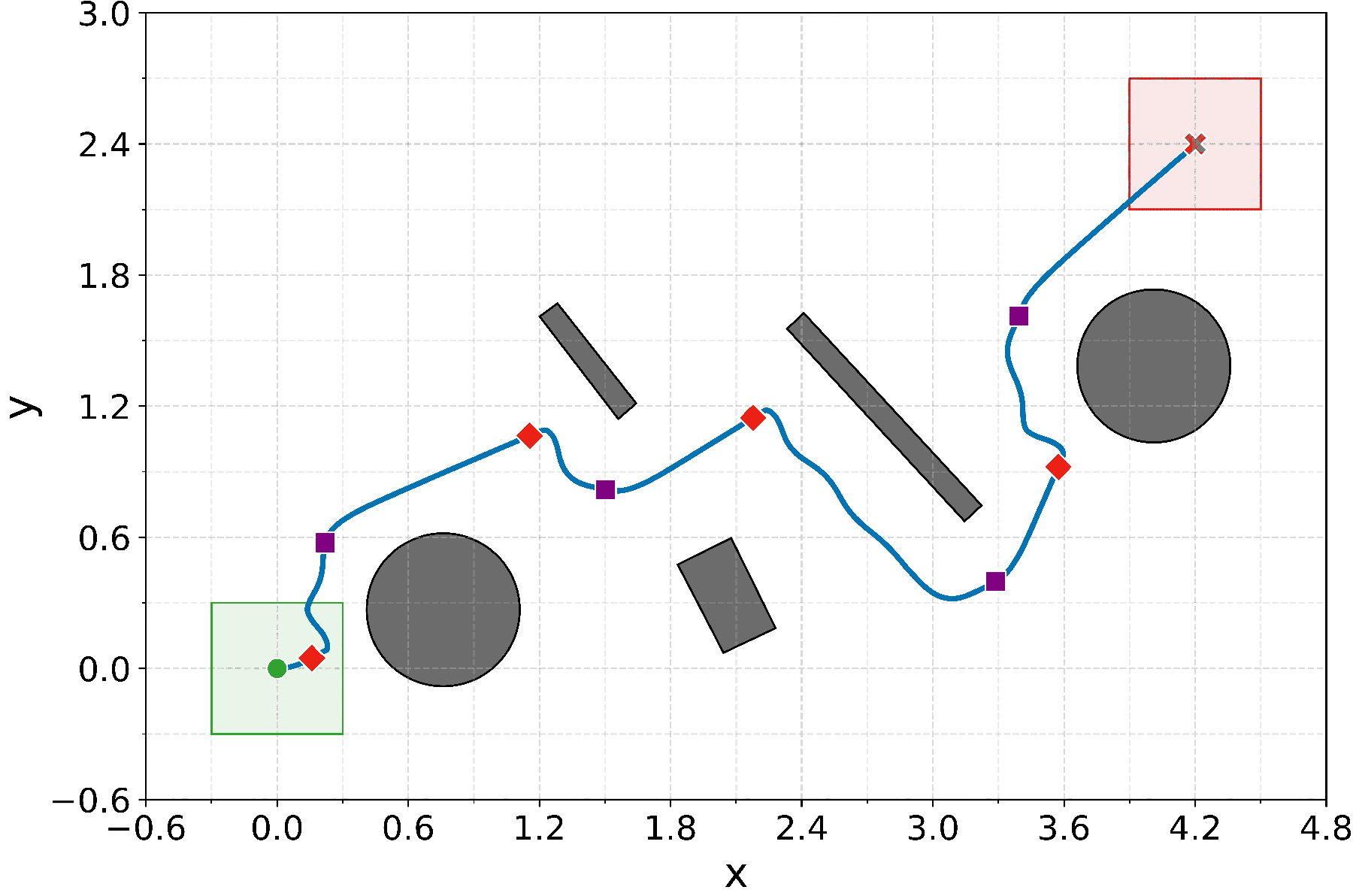}\label{simulated_workspace}}
    \caption{(a) Workspace configuration with highlighted initial (green) and target (red) regions and a long-exposure-style trail of the TurtleBot 4 trajectory. (b) Top-view trajectory of the TurtleBot 4.}
    \label{experiment:real_workspace}
\end{figure}

\begin{figure}
    \centering
    \subfloat[][]{\includegraphics[width = 1\linewidth]{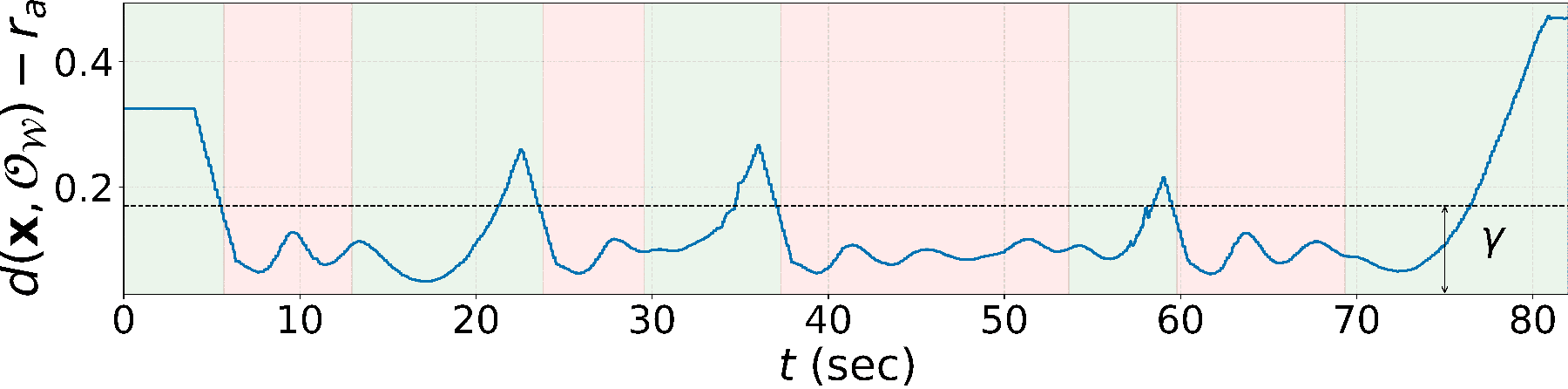}}\\
    \subfloat[][]{\includegraphics[width = 1\linewidth]{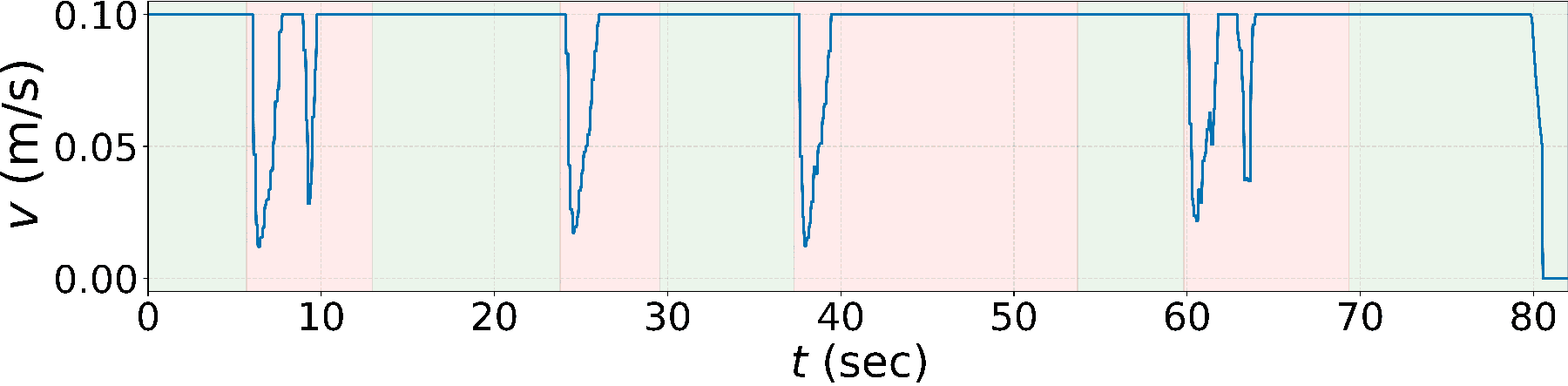}\label{linear_velocity}}\\
    \subfloat[][]{\includegraphics[width = 1\linewidth]{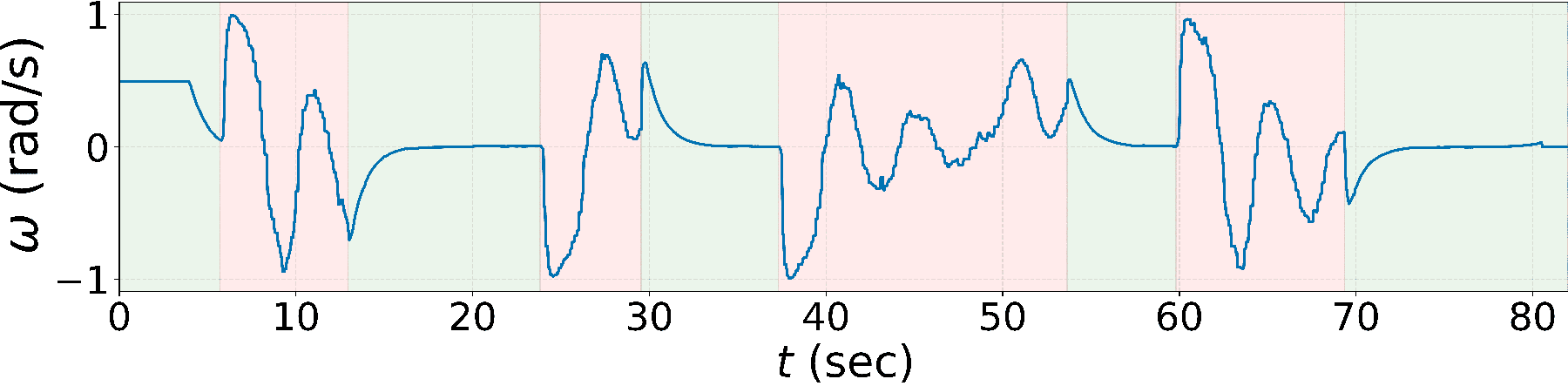}\label{angular_velocity}}
    \caption{Plots of $d(\mathbf{x}, \mathcal{O}_{\mathcal{W}}) - r_a$, linear velocity $v$, and angular velocity $\omega$ versus time for the TurtleBot 4 trajectory shown in Fig. \ref{simulated_workspace}. The green and red shaded regions correspond to the robot operating in the \textit{move-to-target} and \textit{obstacle-avoidance} modes, respectively.}
    \label{experiment:time_profiles}
\end{figure}

}

\section{Conclusion}\label{sec:conclusion}
We propose a hybrid feedback controller for safe autonomous robot navigation in $n-$dimensional environments with arbitrarily-shaped convex obstacles. These obstacles may have nonsmooth boundaries and large sizes, and can be arbitrarily located, provided they meet certain mild disjointedness requirements, as per Assumption \ref{3d_obstacle_separation}. The proposed hybrid controller guarantees global asymptotic stability of the target location in the obstacle-free workspace. 
The obstacle-avoidance component of the control law relies on the projection of the robot's center onto the obstacle being avoided, enabling applications in \textit{a priori} unknown environments, as discussed in Section \ref{implementation_procedure}. The proposed hybrid feedback control law generates discontinuous control inputs when switching between modes. A smoothing mechanism has been suggested to overcome this problem in practical applications. Extending our approach to robots, with second-order dynamics, navigating in $n-$dimensional environments with non-convex obstacles would be an interesting future work.

\begin{appendix}
\subsection{Proof of Lemma \ref{lemma:superset_landing_region}}\label{proof:lemma_superset_landing_region} First, we show that for any $R_s > 0$, it holds that $\mathcal{R}_a^i\subset\mathcal{R}_l^i$. As shown in Fig. \ref{landing_region_proof_image}, let us partition $\mathcal{R}_a^i$ into two mutually exclusive sets as follows:
\begin{equation}
    \mathcal{R}_a^i = \mathcal{R}_1\cup\mathcal{R}_2,
\end{equation}
where the sets $\mathcal{R}_1$ and $\mathcal{R}_2$ are given by
\begin{equation}
    \begin{aligned}
        \mathcal{R}_1 &= \{\mathbf{x}\in\mathcal{R}_a^i|\eth_i^{\mathbf{x}}\cap\left(\mathcal{C}_{\mathcal{L}}(\mathbf{x}, 2r_a)\cap(\mathcal{D}_{r_a}(\mathcal{L}_s(\mathbf{x}, \mathbf{0})))^{\circ}\right)\ne\emptyset\},\\
        \mathcal{R}_2 &= \mathcal{R}_a^i\setminus\mathcal{R}_1.
    \end{aligned}\nonumber
\end{equation}

When $\mathbf{x}\in\mathcal{R}_1$, it is straightforward to notice that $\mathbf{x}\in\mathcal{R}_l^i$, where $\mathcal{R}_l^i$ is defined in \eqref{landing_region}. Therefore, we proceed to prove that $\mathbf{x}\in\mathcal{R}_2\implies\mathbf{x}\in\mathcal{R}_l^i$.
Since $\mathcal{O}_i$ is a convex obstacle, when $\mathbf{x}\in\mathcal{R}_2$, one has $d(\mathcal{L}_s(\mathbf{x}, \mathbf{0}), \eth_i^{\mathbf{x}}) = d(\mathcal{L}_s(\mathbf{x}, \mathbf{0}), \mathcal{O}_i) = r_a$. Now, if we show that $\mathbf{x}^\top\mathbf{x}_{\pi} \geq 0$ for all $\mathbf{x}\in\mathcal{R}_2$, then, as per \eqref{landing_region},  one can conclude that $\mathbf{x}\in\mathcal{R}_2\implies\mathbf{x}\in\mathcal{R}_l^i$.

For $\mathbf{x}\in\mathcal{R}_2$, there exists $\mathbf{p}\in\eth_i^{\mathbf{x}}$ such that $\mathbf{p}\in\mathcal{C}_{\mathcal{L}}(\mathbf{x}, 2r_a)$. Therefore, it is true that $\mathbf{p}\in\mathcal{H}_{\leq}(\mathbf{x}, \mathbf{x})\cap\partial\mathcal{O}_i$. Now, let us assume that $\mathbf{x}^\top\mathbf{x}_{\pi} <0$, which implies that $\Pi(\mathbf{x}, \mathcal{O}_i)\in\mathcal{H}_{>}(\mathbf{x}, \mathbf{x})$. Since $\mathbf{p}\in\partial\mathcal{O}_i$ and $\Pi(\mathbf{x}, \mathcal{O}_i)\in\partial\mathcal{O}_i$, and obstacle $\mathcal{O}_i$ is convex, one has $\mathcal{L}_s(\mathbf{p}, \Pi(\mathbf{x}, \mathcal{O}_i))\subset\mathcal{O}_i$. Furthermore, since $\mathbf{p}\in\mathcal{C}_{\mathcal{L}}(\mathbf{x}, 2r_a)\cap\mathcal{H}_{\leq}(\mathbf{x}, \mathbf{x})$ and $\Pi(\mathbf{x}, \mathcal{O}_i)\in\mathcal{H}_{>}(\mathbf{x}, \mathbf{x})$, it is true that $\mathcal{L}_s(\mathbf{p},\Pi(\mathbf{x}, \mathcal{O}_i))\cap\left(\mathcal{B}_{\|\mathbf{x}_{\pi}\|}(\mathbf{x})\right)^{\circ} \ne \emptyset$. This implies that there exists $\mathbf{q}\in\mathcal{L}_s(\mathbf{p}, \Pi(\mathbf{x}, \mathcal{O}_i))\cap\mathcal{O}_i$ such that $\|\mathbf{x} - \mathbf{q}\| < \|\mathbf{x}_{\pi}\|$, which is a contradiction. Therefore, for $\mathbf{x}\in\mathcal{R}_2$, one can conclude that $\mathbf{x}^\top\mathbf{x}_{\pi} \geq 0$, and hence it is proved that for any $R_s > 0$, $\mathcal{R}_a^i\subset\mathcal{R}_l^i$.

Next, we prove that if $R_s > l_i$, then $\mathcal{R}_a^i = \mathcal{R}_l^i$. Notice that the implication $\mathbf{x}\in\mathcal{R}_a^i\implies \mathbf{x}\in\mathcal{R}_l^i$ has been proved earlier for any $R_s > 0$. Therefore, we focus on the backward implication and show that when $R_s > l_i$, $\mathbf{x}\in\mathcal{R}_l^i\implies\mathbf{x}\in\mathcal{R}_a^i$. For $\mathbf{x}\in\mathcal{R}_l^i$, there exists $\mathbf{p}\in\partial\mathcal{O}_i$ such that $d(\mathbf{p}, \mathcal{L}_s(\mathbf{x}, \mathbf{0})) \leq r_a$, $\mathcal{L}_s(\mathbf{x}, \mathbf{p})\cap\mathcal{O}_i=\mathbf{p}$ and $\mathbf{x}^\top(\mathbf{x} - \mathbf{p})\geq 0$. Since $R_s > l_i$, it is straightforward to notice that $\mathbf{p}\in\eth_i^{\mathbf{x}}$. Therefore, according to \eqref{Individual_avoidance_region}, one can conclude that when $R_s > l_i$, $\mathbf{x}\in\mathcal{R}_l^i\implies\mathbf{x}\in\mathcal{R}_a^i$.

\begin{figure}[ht]
\centering
\includegraphics[width = 0.7\linewidth]{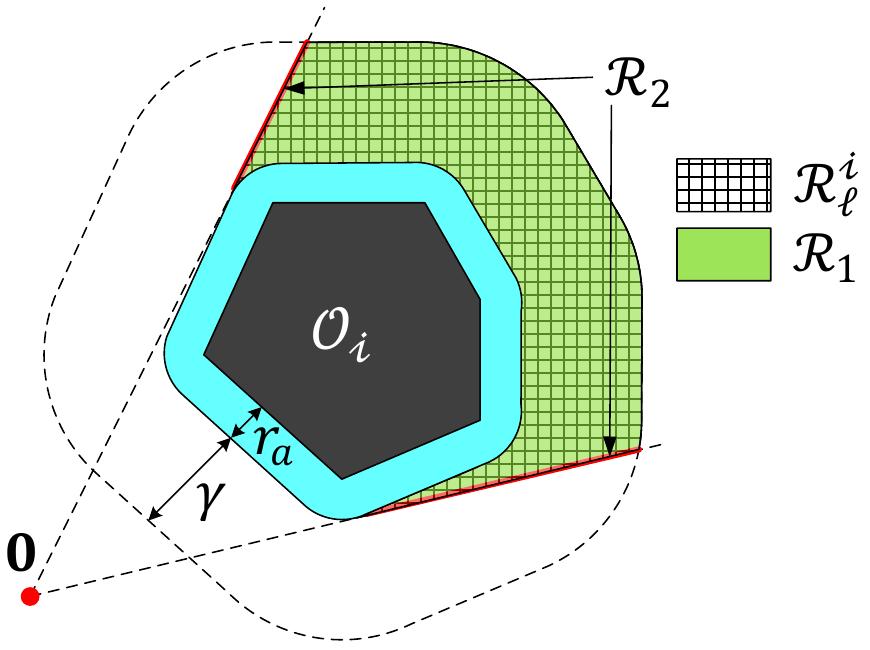}
\caption{The representation of the sets $\mathcal{R}_1$ and $\mathcal{R}_2$.}
\label{landing_region_proof_image}
\end{figure}

\subsection{Proof of Lemma \ref{lemma:epsilon_exists}}
\label{proof_of_epsilon_exists}
Since the target location at the origin $\mathbf{0}$ belongs to the interior of the obstacle-free workspace $\mathcal{W}_{r_a}^{\circ}$, there exists some distance between the target location and the $r_a-$dilated obstacle $\mathcal{D}_{r_a}(\mathcal{O}_i)$.
In other words, since $\mathbf{0}\in\mathcal{W}_{r_a}^{\circ}, $ there exists $\bar{\delta} > 0$ such that $d(\mathbf{0},\mathcal{D}_{r_a}(\mathcal{O}_i)) = \bar{\delta}.$
Notice that, according to \eqref{jumpset_movetotarget}, the set $\mathcal{J}_0^{\mathcal{W}}\cap\mathcal{N}_{\gamma}(\mathcal{D}_{r_a}(\mathcal{O}_i))$ belongs to the \textit{avoidance} region associated with obstacle $\mathcal{O}_i$ \textit{i.e.,}
$\left(\mathcal{J}_{0}^{\mathcal{W}}\cap\mathcal{N}_{\gamma}(\mathcal{D}_{r_a}(\mathcal{O}_i))\right)\subset\mathcal{R}_a^i$. In addition, as per Lemma \ref{lemma:superset_landing_region}, $\mathcal{R}_a^i\subset\mathcal{R}_l^i$, where $\mathcal{R}_l^i$ is defined in \eqref{landing_region}. Therefore, it is clear that $\left(\mathcal{J}_0^{\mathcal{W}}\cap\mathcal{N}_{\gamma}(\mathcal{D}_{r_a}(\mathcal{O}_i))\right)\subset\mathcal{R}_l^i$. According to \eqref{landing_region}, 
$\Pi(\mathbf{0}, \mathcal{D}_{r_a}(\mathcal{O}_i))$ does not belong to the set $\mathcal{R}_l^i$. As a result, one has $d(\mathbf{0}, \mathcal{R}_l^i)>\bar{\delta}.$ Hence, it is clear that $\mathcal{B}_{\bar{\delta}}(\mathbf{0})\cap\mathcal{N}_{\gamma}(\mathcal{D}_{r_a}(\mathcal{O}_i))\cap\mathcal{J}_0^{\mathcal{W}} = \emptyset,$ where $\mathcal{B}_{\bar{\delta}}(\mathbf{0})\cap\mathcal{N}_{\gamma}(\mathcal{D}_{r_a}(\mathcal{O}_i)) \ne \emptyset.$ Hence, one can set $\bar{\epsilon}\in(0, d(\mathbf{0}, \mathcal{R}_l^i)-\bar{\delta}]$ to ensure that $\mathcal{ER}^{\mathbf{h}}\cap\mathcal{N}_{\gamma}(\mathcal{D}_{r_a}(\mathcal{O}_i))\ne \emptyset$ for all $\mathbf{h}\in\mathcal{J}_0^{\mathcal{W}}\cap\mathcal{N}_{\gamma}(\mathcal{D}_{r_a}(\mathcal{O}_i))$, for any $\epsilon\in(0, \bar{\epsilon}]$, where the set $\mathcal{ER}^{\mathbf{h}}$ is defined in \eqref{partition_rm}.

\subsection{Proof of Lemma \ref{basic-conditions}}
\label{proof_for_basic_conditions}

The flow set $\mathcal{F}$ and the jump set $\mathcal{J}$, defined in \eqref{bothmodes_flowjumpset} are by construction closed subsets of $\mathbb{R}^n\times\mathbb{R}^n\times\mathbb{R}^n\times\mathbb{R}\times\mathbb{R}.$ Hence, condition 1 in Lemma \ref{lemma:hybrid_basic_conditions_for_3D} is satisfied.

Since the flow map $\mathbf{F}(\xi)$ is defined for all $\xi\in\mathcal{F}$, one has $\mathcal{F}\subset\text{ dom }\mathbf{F}.$ The flow map $\mathbf{F},$ given in \eqref{hybrid_closed_loop_system}, is continuous on $\mathcal{F}_0.$ Next, we verify the continuity of $\mathbf{F}$ on $\mathcal{F}_1.$ 
Since $\gamma\in(0, \delta- r_s)$, the sets $\mathcal{N}_{\gamma}(\mathcal{D}_{r_a}(\mathcal{O}_i))$, for all $i\in\mathbb{I}$, are disjoint. Since obstacles $\mathcal{O}_i, i\in\mathbb{I}\setminus\{0\}$ are convex, for all locations $\mathbf{x}\in\mathcal{N}_{\gamma}(\mathcal{D}_{r_a}(\mathcal{O}_i)),i\in\mathbb{I}\setminus\{0\}$, the closest point, from $\mathbf{x}$, on the boundary of the nearest obstacle $\Pi(\mathbf{x}, \mathcal{O}_i)$ is unique. Furthermore, according to \eqref{flowset_obstacleavoidance}, the set $\mathcal{F}_{1}^{\mathcal{W}}\subset\bigcup_{i\in\mathbb{I}\setminus\{0\}}\mathcal{N}_{\gamma}(\mathcal{D}_{r_a}(\mathcal{O}_{i}))$. Hence, according to \cite[Lemma 4.1]{rataj2019curvature} and \eqref{flowset_obstacleavoidance}, $\Pi(\mathbf{x}, \mathcal{O}_{\mathcal{W}})$ is continuous for all $\mathbf{x}\in\mathcal{F}_{1}^{\mathcal{W}}$. Hence, the obstacle-avoidance control vector $\kappa_r\mathbf{v}(\mathbf{x}, \mathbf{h}, \mathbf{a})$, used in \eqref{control_u}, is continuous for all locations $\mathbf{x}\in\mathcal{F}_{1}^{\mathcal{W}}$ with the unit vector $\mathbf{a}$ chosen as per \eqref{updatelaw_part1}. As a result, $\mathbf{F}$ is continuous on $\mathcal{F}_1$ and as such it is continuous on $\mathcal{F}.$ This shows fulfillment of condition 2 in Lemma \ref{lemma:hybrid_basic_conditions_for_3D}. 

Since the jump map $\mathbf{J}(\xi)$ is defined for all $\xi\in\mathcal{J}$, one has $\mathcal{J}\subset\text{ dom }\mathbf{J}.$
The jump map $\mathbf{J}$, defined in \eqref{hybrid_closed_loop_system}, is single-valued on $\mathcal{J}_1$. Hence, according to \cite[Definitions 5.9 and 5.14]{goebel2012hybrid}, the jump map $\mathbf{J}$ is outer semicontinuous and locally bounded relative to $\mathcal{J}_1$.

Finally, we prove that the jump map $\mathbf{J}$ is outer semicontinuous and locally bounded relative to $\mathcal{J}_0$. According to \eqref{updatelaw_part1} and \eqref{hybrid_closed_loop_system}, the jump map $\mathbf{J}$ is single-valued for the state vector $(\mathbf{x}, \mathbf{h}, m, s)$ on $\mathcal{J}_0$. Consider the jump map $\mathbf{J}$ for the state $\mathbf{a}$ on $\mathcal{J}_0.$
We show that the set-valued mapping $\mathbf{A}:\mathbb{R}^n\rightrightarrows\mathbb{S}^{n-1}$, used in \eqref{updatelaw_part1}, is outer semicontinuous and locally bounded. To that end, consider any sequence $\{\mathbf{q}_i\}_{i\in\mathbb{N}}\subset\mathbb{R}^n$ that converges to some $\mathbf{q}\in\mathbb{R}^n.$ 
Let us assume that the sequence $\{\mathbf{p}_i\}_{i\in\mathbb{N}}$ converges to some $\mathbf{p}\in\mathbb{S}^{n-1}$, where $\mathbf{p}_i \in \mathbf{A}(\mathbf{q}_i)$ \eqref{update_law_for_vector_a}.
Note that $\mathbf{p}_i^\top\mathbf{q}_i = 0$ and if $\mathbf{q}_i^\times\mathbf{q}_{i\pi}\ne\mathbf{0}$, then $\mathbf{p}_i\in\mathcal{P}(\mathbf{q}_i, \mathbf{q}_{i\pi})$, where $\mathbf{q}_{i\pi} = \mathbf{q}_i - \Pi(\mathbf{q}_i, \mathcal{O}_{\mathcal{W}})$.
Therefore, one can conclude that $\mathbf{p}^\top\mathbf{q} = 0$ and if $\mathbf{q}^\times\mathbf{q}_{\pi}\ne\mathbf{0}$, then $\mathbf{p}\in\mathcal{P}(\mathbf{q}, \mathbf{q}_{\pi})$, and as such $\mathbf{p} \in \mathbf{A}(\mathbf{q})$. Hence,  according \cite[Definition 5.9]{goebel2012hybrid}, the mapping $\mathbf{A}$ is outer semicontinuous relative to $\mathcal{J}_0$. Since $\text{ rge }\mathbf{A} = \mathbb{S}^{n-1}\subset\mathbb{R}^n$ is bounded, according to \cite[Definition 5.14]{goebel2012hybrid}, the set-valued mapping $\mathbf{A}$ is locally bounded, where the range of $\mathbf{A}$ is defined as per \cite[Definition 5.8]{goebel2012hybrid}.
Hence, $\mathbf{J}$ is outer semi-continuous and locally bounded relative to $\mathcal{J}_0$. This shows the fulfillment of condition 3 in Lemma \ref{lemma:hybrid_basic_conditions_for_3D}.

\subsection{Proof of Lemma \ref{lemma:set_invariance}}\label{proof:lemma_invariance}
First, we prove that the union of the flow and jump sets covers exactly the obstacle-free state space $\mathcal{K}$. For $m = 0$, according to \eqref{jumpset_movetotarget} and \eqref{flowset_movetotarget}, by construction we have $\mathcal{F}_{0}^{\mathcal{W}}\cup\mathcal{J}_{0}^{\mathcal{W}} = \mathcal{W}_{r_a}.$ Similarly, for $m=1$, according to \eqref{jumpset_obstacleavoidance} and \eqref{flowset_obstacleavoidance}, by construction one has $\mathcal{F}_1^{\mathcal{W}}\cup\mathcal{J}_1^{\mathcal{W}} = \mathcal{W}_{r_a}$. Inspired by \cite[Appendix 11]{berkane2021obstacle}, the satisfaction of the following equation:
\begin{equation}
\mathcal{F}_m^{\mathcal{W}}\cup\mathcal{J}_m^{\mathcal{W}}= \mathcal{W}_{r_a}, m\in\mathbb{M},
\end{equation}
along with \eqref{mtt_jumpset}, \eqref{mtt_flowset}, \eqref{oa_jumpset}, \eqref{oa_flowset}, and \eqref{bothmodes_flowjumpset} implies $\mathcal{F}\cup\mathcal{J} = \mathcal{K}.$

Now, inspired by \cite[Appendix 1]{berkane2021obstacle}, for the hybrid closed-loop system \eqref{hybrid_closed_loop_system}, with data $\mathcal{H}_{\mathcal{S}} = (\mathcal{F}, \mathbf{F}, \mathcal{J}, \mathbf{J}),$ define $\mathbf{S}_{\mathcal{H}_{\mathcal{S}}}(\mathcal{K})$ as the set of all maximal solutions $\xi$ to $\mathcal{H}_{\mathcal{S}}$ with $\xi(0, 0) \in\mathcal{K}$.
Since $\mathcal{F}\cup\mathcal{J} = \mathcal{K}$, each $\xi\in\mathbf{S}_{\mathcal{H}_{\mathcal{S}}}(\mathcal{K})$ satisfies $\xi(t, j)\in\mathcal{K}$ for all $(t, j)\in\dom {\xi}$, where the domain $\dom \xi$ is defined in \cite[Definition 2.3]{goebel2012hybrid}.
Additionally, if every maximal solution $\xi\in\mathbf{S}_{\mathcal{H}_{\mathcal{S}}}(\mathcal{K})$ is complete, then the set $\mathcal{K}$ will be forward invariant \cite[Definition 3.13]{sanfelice2021hybrid}. Since the hybrid closed-loop system \eqref{hybrid_closed_loop_system} satisfies the hybrid basic conditions, as stated in Lemma \ref{basic-conditions}, one can use \cite[Proposition 6.10]{goebel2012hybrid}, to verify the following viability condition:
\begin{equation}
    \mathbf{F}(\xi) \cap\mathbf{T}_{\mathcal{F}}(\xi) \ne \emptyset, \forall\xi\in\mathcal{F}\setminus\mathcal{J},\label{viability_condition}
\end{equation}
which will allow us to establish the completeness of the solution $\xi$ to the hybrid closed-loop system \eqref{hybrid_closed_loop_system}. In \eqref{viability_condition}, $\mathbf{T}_{\mathcal{F}}(\xi)$ represents the tangent cone\footnote{The tangent cone to a set $\mathcal{K}\subset\mathbb{R}^n$ at $\mathbf{x}\in\mathbb{R}^n$, denoted by $\mathbf{T}_{\mathcal{K}}(\mathbf{x})$, is defined as in \cite[Def. 5.12 and Fig. 5.4]{goebel2012hybrid}.} to the set $\mathcal{F}$ at $\xi$.

Let $(\mathbf{x}, \mathbf{h}, \mathbf{a}, m, s)\in\mathcal{F}\setminus\mathcal{J}$, which implies by virtue of \eqref{mtt_jumpset}, \eqref{mtt_flowset}, \eqref{oa_jumpset}, \eqref{oa_flowset}, and \eqref{bothmodes_flowjumpset} that $(\mathbf{x}, \mathbf{h}, \mathbf{a}, s)\in(\mathcal{F}_m^{\mathcal{W}}\setminus\mathcal{J}_m^{\mathcal{W}})\times\mathcal{W}_{r_a}\times\mathbb{S}^{n-1}\times\mathbb{R}_{\geq 0}$ for some $m\in\mathbb{M}$. 
For $\mathbf{x}\in(\mathcal{F}_{m}^{\mathcal{W}})^{\circ}\setminus\mathcal{J}_{m}^{\mathcal{W}}$ with $(\mathbf{h}, \mathbf{a}, m, s)\in\mathcal{W}_{r_a}\times\mathbb{S}^{n-1}\times\mathbb{M}\times\mathbb{R}_{\geq0}$ the tangent cone $\mathbf{T}_{\mathcal{F}}(\xi) = \mathbb{R}^n\times\mathbf{T}_{\mathcal{W}_{r_a}}(\mathbf{h})\times\mathcal{H}(\mathbf{0}, \mathbf{a})\times\{0\}\times\mathbf{T}_{\mathbb{R}_{\geq 0}}(s), $ where the set $\mathbf{T}_{\mathcal{W}_{r_a}}(\mathbf{h})$ is given by
\begin{equation}
    \mathbf{T}_{\mathcal{W}_{r_a}}(\mathbf{h}) = \begin{cases}
    \begin{matrix*}[l] \mathbb{R}^n, & \text{if }\mathbf{h}\in(\mathcal{W}_{r_a})^{\circ},\\
    \mathcal{H}_{\geq}(\mathbf{0}, (\mathbf{h} - \Pi(\mathbf{h}, \mathcal{O}_{\mathcal{W}}))),& \text{if }\mathbf{h}\in\partial\mathcal{W}_{r_a},\end{matrix*}
    \end{cases}\label{definition_hp}
\end{equation}
where for $\mathbf{h}\in\partial\mathcal{W}_{r_a}$, the projection $\Pi(\mathbf{h}, \mathcal{O}_{\mathcal{W}})$ is unique. For $s\in\mathbb{R}_{\geq 0}$, the set $\mathbf{T}_{\mathbb{R}_{\geq 0}}(s)$ is defined as
\begin{equation}
    \mathbf{T}_{\mathbb{R}_{\geq 0}}(s) = \begin{cases}
    \begin{matrix*}[l] \mathbb{R}, & \text{if }s\in\mathbb{R}_{> 0},\\
   \mathbb{R}_{\geq 0},& \text{if }s = 0.\end{matrix*}
    \end{cases}\label{definition_sp}
\end{equation}

Since, according to \eqref{hybrid_closed_loop_system}, $\dot{\mathbf{h}} = \mathbf{0} $ and $\dot{s} = 1$, we have $\dot{\mathbf{h}}\in\mathbf{T}_{\mathcal{W}_{r_a}}(\mathbf{h})$ and $\dot{s}\in\mathbf{T}_{\mathbb{R}_{\geq 0}}(s)$, respectively, and \eqref{viability_condition} holds.

Next, we consider the case where $\xi\in\mathcal{F}\setminus\mathcal{J}$ with $\mathbf{x}\in\partial\mathcal{F}_m^{\mathcal{W}}\setminus\mathcal{J}_m^{\mathcal{W}}$, $m\in\mathbb{M}$.
For $m = 0$, according to \eqref{jumpset_movetotarget} and \eqref{flowset_movetotarget}, one has
\begin{equation}
\partial\mathcal{F}_0^{\mathcal{W}}\setminus\mathcal{J}_0^{\mathcal{W}} \subset \partial\mathcal{D}_{r_a}(\mathcal{O}_{\mathcal{W}})\cap\mathcal{R}_e,\label{forward_invariance_equation_1}
\end{equation}
and for $\mathbf{x}\in\partial\mathcal{F}_0^{\mathcal{W}}\setminus\mathcal{J}_0^{\mathcal{W}},$ the projection $\Pi(\mathbf{x}, \mathcal{O}_{\mathcal{W}})$ is unique. 
Hence, for all $\xi\in\mathcal{F}_0\setminus\mathcal{J}_0$ with $\mathbf{x}\in\partial\mathcal{F}_{0}^{\mathcal{W}}\setminus\mathcal{J}_0^{\mathcal{W}}$, 
\begin{equation}
    \mathbf{T}_{\mathcal{F}}(\xi) = \mathcal{H}_{\geq}(\mathbf{0}, \mathbf{x}_{\pi})\times\mathbf{T}_{\mathcal{W}_{r_a}}(\mathbf{h})\times\mathcal{H}(\mathbf{0},\mathbf{a})\times\{0\}\times\mathbf{T}_{\mathbb{R}_{\geq 0}}(s),
\end{equation}
where $\mathbf{T}_{\mathcal{W}_{r_a}}(\mathbf{h})$ and $\mathbf{T}_{\mathbb{R}_{\geq 0}}(s)$ are defined in \eqref{definition_hp} and \eqref{definition_sp}, respectively, and $\mathbf{x}_{\pi} = \mathbf{x}- \Pi(\mathbf{x}, \mathcal{O}_{\mathcal{W}})$. Also, according to \eqref{control_u}, for $m =0$, one has $\mathbf{u}(\xi) = -\kappa_s\mathbf{x}, \kappa_s>0$. According to \eqref{exit_region} and \eqref{forward_invariance_equation_1}, for $\xi\in\mathcal{F}_0\setminus\mathcal{J}_0$ with $\mathbf{x}\in\partial\mathcal{F}_0^{\mathcal{W}}\setminus\mathcal{J}_{0}^{\mathcal{W}}$, one can conclude that $\mathbf{u}(\xi)\in\mathcal{H}_{\geq}(\mathbf{0}, \mathbf{x}_{\pi})$. Moreover, according to \eqref{hybrid_closed_loop_system}, it is clear that $\dot{\mathbf{h}} = \mathbf{0}\in\mathbf{T}_{\mathcal{W}_{r_a}}(\mathbf{h})$, $\dot{s} = 1\in\mathbf{T}_{\mathbb{R}_{\geq 0}}(s)$ and $\dot{\mathbf{a}} = \mathbf{0}\in\mathcal{H}(\mathbf{0}, \mathbf{a})$. Therefore, the viability condition \eqref{viability_condition} holds for $m =0$.

For $m = 1$, according to \eqref{jumpset_obstacleavoidance} and \eqref{flowset_obstacleavoidance} one has
\begin{equation}
    \partial\mathcal{F}_{1}^{\mathcal{W}}\setminus\mathcal{J}_1^{\mathcal{W}}\subset\partial\mathcal{D}_{r_a}(\mathcal{O}_{\mathcal{W}}),
\end{equation}
and for $\mathbf{x}\in\partial\mathcal{F}_1^{\mathcal{W}}\setminus\mathcal{J}_1^{\mathcal{W}}$, the projection $\Pi(\mathbf{x}, \mathcal{O}_{\mathcal{W}})$ is unique and the set $\mathcal{B}_{\norm{\mathbf{x}_{\pi}}}(\mathbf{x})$ intersects with $\partial\mathcal{O}_{\mathcal{W}}$ only at $\Pi(\mathbf{x}, \mathcal{O}_{\mathcal{W}}).$ Hence, for all $\xi\in\mathcal{F}_1\setminus\mathcal{J}_1$ with $\mathbf{x}\in\partial\mathcal{F}_1^{\mathcal{W}}\setminus\mathcal{J}_1^{\mathcal{W}}$, 
\begin{equation}
    \mathbf{T}_{\mathcal{F}}(\xi) = \mathcal{H}_{\geq}(\mathbf{0}, \mathbf{x}_{\pi})\times\mathbf{T}_{\mathcal{W}_{r_a}}(\mathbf{h})\times\mathcal{H}(\mathbf{0}, \mathbf{a})\times\{0\}\times\mathbf{T}_{\mathbb{R}_{\geq 0}}(s),
\end{equation}
where $\mathbf{T}_{\mathcal{W}_{r_a}}(\mathbf{h})$ and $\mathbf{T}_{\mathbb{R}_{\geq 0}}(s)$ are defined in \eqref{definition_hp} and \eqref{definition_sp}, respectively. Also, according to \eqref{control_u}, for $m = 1$, one has $\mathbf{u}(\xi) = \kappa_r\mathbf{v}(\mathbf{x},\mathbf{h}, \mathbf{a}), \kappa_r > 0.$ From \eqref{scalar_function_eta_definition}, it follows that $\eta(\mathbf{x})=1$ for $\mathbf{x}\in\partial\mathcal{F}_1^{\mathcal{W}}\setminus\mathcal{J}_1^{\mathcal{W}}$. As a result, according to \eqref{control_u} and \eqref{n_dimensional_obstacle-avoidance_vector}, the control vector is simplified to {$\mathbf{u}(\xi) = \kappa_r{\mathbf{P}(\hat{\mathbf{h}}, \mathbf{a})\mathbf{x}_{\pi}}.$} 
Since for any two orthonormal vectors $\hat{\mathbf{h}} ,\mathbf{a}\in\mathbb{S}^{n-1}$, the matrix $\mathbf{P}(\hat{\mathbf{h}}, \mathbf{a})$ is positive semidefinite, one has $\mathbf{x}_{\pi}^\intercal\mathbf{P}(\hat{\mathbf{h}},\mathbf{a})\mathbf{x}_{\pi} \geq 0$.
Therefore, for $\xi\in\mathcal{F}_{1}\setminus\mathcal{J}_1$ with $\mathbf{x}\in\mathcal{F}_1^{\mathcal{W}}\setminus\mathcal{J}_1^{\mathcal{W}}$, one has $\mathbf{u}(\mathbf{x}, \mathbf{h}, \mathbf{a}, 1,s)^\intercal(\mathbf{x}- \Pi(\mathbf{x}, \mathcal{O}_{\mathcal{W}})) \geq 0$. Moreover, according to \eqref{hybrid_closed_loop_system}, it is clear that $\dot{\mathbf{h}}= \mathbf{0}\in\mathbf{T}_{\mathcal{W}_{r_a}}(\mathbf{h})$, $\dot{s} = 1$ and $\dot{\mathbf{a}} = \mathbf{0}\in\mathcal{H}(\mathbf{0}, \mathbf{a})$. Hence, the viability condition in \eqref{viability_condition} holds for $m=1.$

Hence, according to \cite[Proposition 6.10]{goebel2012hybrid}, since \eqref{viability_condition} holds for all $\xi\in\mathcal{F}\setminus\mathcal{J}$, there exists a nontrivial solution to $\mathcal{H}$ for each initial condition in $\mathcal{K}$. Finite escape time can only occur through flow. They can neither occur for $\mathbf{x}$ in the set $\mathcal{F}_1^{\mathcal{W}}$, as this set is bounded as per definition \eqref{flowset_obstacleavoidance}, nor for $\mathbf{x}$ in the set $\mathcal{F}_0^{\mathcal{W}}$ as this would make $\mathbf{x}^{\intercal}\mathbf{x}$ grow unbounded, and would contradict the fact that $\frac{d}{dt}(\mathbf{x}^\intercal\mathbf{x})\leq0$ in view of the definition of $\mathbf{u}(\mathbf{x}, \mathbf{h}, \mathbf{a}, 0, s)$. Therefore, all maximal solutions do not have finite escape times. Furthermore, according to \eqref{hybrid_closed_loop_system}, $\mathbf{x}^+ = \mathbf{x},$ and from the definition of the update law in \eqref{updatelaw_part1} and \eqref{updatelaw_part2}, it follows immediately that $\mathbf{J}(\mathcal{J})\subset\mathcal{K}$. Hence, the solutions to the hybrid closed-loop system \eqref{hybrid_closed_loop_system} cannot leave $\mathcal{K}$ through jump and, as per \cite[Proposition 6.10]{goebel2012hybrid}, all maximal solutions are complete.

\subsection{Proof of Lemma \ref{lemma:always_enter_move_to_target_mode}}
\label{proof_of_lemma_always_enter}

First, we prove that when the control input corresponds to the \textit{obstacle-avoidance} mode, one has $\mathbf{x}(t, j)\in\mathcal{N}_{\gamma}(\mathcal{D}_{r_a}(\mathcal{O}_i))$ for all $(t, j)\in(I_{j_1+1}\times j_1 +1)$. To that end,  we make use of Nagumo's theorem \cite[Theorem 4.7]{blanchini2008set} and show that when the control input corresponds to the \textit{obstacle-avoidance} mode, one has
\begin{equation}
    \mathbf{u}(\xi) \in\mathbf{T}_{\mathcal{N}_{\gamma}(\mathcal{D}_{r_a}(\mathcal{O}_i))}(\mathbf{x}),\label{small_viability_condition}
\end{equation}
for all $\mathbf{x}\in\partial\mathcal{N}_{\gamma}(\mathcal{D}_{r_a}(\mathcal{O}_i))$.
This, combined with the fact that the control vector trajectory $\mathbf{u}(\xi)$ is continuous, when it corresponds to the \textit{obstacle-avoidance} mode, as stated in Lemma \ref{basic-conditions}, ensures that $\mathbf{x}(t, j) \in\mathcal{N}_{\gamma}(\mathcal{D}_{r_a}(\mathcal{O}_i))$ for all $(t, j)\in(I_{j_1 +1}\times j_1 +1)$.

Note that $\partial\mathcal{N}_{\gamma}(\mathcal{D}_{r_a}(\mathcal{O}_i)) = \partial\mathcal{D}_{r_a}(\mathcal{O}_i)\cup\partial\mathcal{D}_{r_a + \gamma}(\mathcal{O}_i)$. For all $\mathbf{x}\in\partial\mathcal{D}_{r_a}(\mathcal{O}_i)$, one has $\mathcal{H}_{\geq}(\mathbf{0}, \mathbf{x}_{\pi})=\mathbf{T}_{\mathcal{N}_{\gamma}(\mathcal{D}_{r_a}(\mathcal{O}_i))}(\mathbf{x})$, where $\mathbf{x}_{\pi} = \mathbf{x} - \Pi(\mathbf{x}, \mathcal{O}_{\mathcal{W}})$ with $\Pi(\mathbf{x}, \mathcal{O}_{\mathcal{W}}) = \Pi(\mathbf{x}, \mathcal{O}_i)$. Also, since the control input corresponds to the \textit{obstacle-avoidance} mode, for $\mathbf{x}\in\partial\mathcal{D}_{r_a}(\mathcal{O}_i)$, the control vector \eqref{control_u} is given by $\mathbf{u}(\xi) = \kappa_r\mathbf{P}(\hat{\mathbf{h}}, \mathbf{a})\mathbf{x}_{\pi}, \kappa_r > 0$. 
Since for any two orthonormal vectors $\mathbf{q}, \mathbf{s}\in\mathbb{S}^{n-1}$, the matrix $\mathbf{P}(\mathbf{q}, \mathbf{s})$ is positive semidefinite, one has $\mathbf{p}^\top\mathbf{P}(\mathbf{q}, \mathbf{s})\mathbf{p}\geq 0$ for all $\mathbf{p}\in\mathbb{R}^n$.
Therefore, for all $\mathbf{x}\in\partial\mathcal{D}_{r_a}(\mathcal{O}_i)$, one has $\mathbf{x}_{\pi}\mathbf{P}(\hat{\mathbf{h}}, \mathbf{a})\mathbf{x}_{\pi}\geq 0$. This implies that $\mathbf{u}(\xi)\in\mathcal{H}_{\geq}(\mathbf{0}, \mathbf{x}_{\pi})=\mathbf{T}_{\mathcal{N}_{\gamma}(\mathcal{D}_{r_a}(\mathcal{O}_i))}(\mathbf{x})$ for $\mathbf{x}\in\partial\mathcal{D}_{r_a}(\mathcal{O}_i)$, and condition \eqref{small_viability_condition} holds true.

Next, for $\mathbf{x}\in\partial\mathcal{D}_{r_a + \gamma}(\mathcal{O}_i)$, one has $\mathcal{H}_{\leq}(\mathbf{0}, \mathbf{x}_{\pi})=\mathbf{T}_{\mathcal{N}_{\gamma}(\mathcal{D}_{r_a}(\mathcal{O}_i))}(\mathbf{x})$. Also, since the control input corresponds to the \textit{obstacle-avoidance} mode, for $\mathbf{x}\in\partial\mathcal{D}_{r_a + \gamma}(\mathcal{O}_i)$, the control vector \eqref{control_u} is given by $\mathbf{u}(\xi) = -\kappa_r\mathbf{P}(\hat{\mathbf{h}}, \mathbf{a})\mathbf{x}_{\pi}, \kappa_r > 0$. 
As mentioned earlier, for all $\mathbf{x}\in\partial\mathcal{D}_{r_a+ \gamma}(\mathcal{O}_i)$, one has $\mathbf{x}_{\pi}\mathbf{P}(\hat{\mathbf{h}}, \mathbf{a})\mathbf{x}_{\pi}\geq 0$. Therefore, $\mathbf{u}(\xi)\in\mathcal{H}_{\leq}(\mathbf{0}, \mathbf{x}_{\pi})=\mathbf{T}_{\mathcal{N}_{\gamma}(\mathcal{D}_{r_a}(\mathcal{O}_i))}(\mathbf{x})$ for $\mathbf{x}\in\partial\mathcal{D}_{r_a + \gamma}(\mathcal{O}_i)$, and condition \eqref{small_viability_condition} holds true. As a result, since $\mathbf{x}(t_1, j_1 +1)\in\mathcal{N}_{\gamma}(\mathcal{D}_{r_a}(\mathcal{O}_i))$, one can conclude that
\begin{equation}
    \mathbf{x}(t, j)\in\mathcal{N}_{\gamma}(\mathcal{D}_{r_a}(\mathcal{O}_i)),\label{state_belongs_to_gamma_neighborhood}
\end{equation}
for all $(t, j)\in(I_{j_1+1}\times j_1 +1)$.
 
Next, we show that when the control input corresponds to the \textit{obstacle-avoidance} mode,  $\mathbf{x}(t, j)\in\mathcal{N}_{\gamma}(\mathcal{D}_{r_a}(\mathcal{O}_i))\cap\mathcal{P}(\mathbf{h}, \mathbf{a})$ for all $(t, j)\in(I_{j_1 +1}\times j_1+1)$. 
When the control input corresponds to the \textit{obstacle-avoidance} mode, it is given by $\mathbf{u}(\xi) = \kappa_r\mathbf{v}(\mathbf{x}, \mathbf{h}, \mathbf{a})$, which, as per \eqref{n_dimensional_obstacle-avoidance_vector}, can be expressed as a linear combination of the vectors $\mathbf{P}(\hat{\mathbf{h}}, \mathbf{a})\mathbf{x}_{\pi}$ and $\mathbf{R}(\hat{\mathbf{h}}, \mathbf{a})\mathbf{P}(\hat{\mathbf{h}}, \mathbf{a})\mathbf{x}_{\pi}$.
Note that, according to \eqref{updatelaw_part1}, one has $\mathbf{h}^\top\mathbf{a} = 0$. Therefore, for all $\mathbf{x}\in\mathcal{N}_{\gamma}(\mathcal{D}_{r_a}(\mathcal{O}_i))\cap\mathcal{P}(\mathbf{h}, \mathbf{a})$, as per \eqref{parallel_projection_operator}, one has $\mathbf{P}(\hat{\mathbf{h}}, \mathbf{a})\mathbf{x}_{\pi}\in\mathcal{P}(\mathbf{h}, \mathbf{a})$. Additionally, using \eqref{rotational_operator}, one can show that $\mathbf{R}(\hat{\mathbf{h}}, \mathbf{a})\mathbf{P}(\hat{\mathbf{h}}, \mathbf{a})\mathbf{x}_{\pi}\in\mathcal{P}(\mathbf{h}, \mathbf{a})$. Therefore, for all $\mathbf{x}\in\mathcal{N}_{\gamma}(\mathcal{D}_{r_a}(\mathcal{O}_i))\cap\mathcal{P}(\mathbf{h}, \mathbf{a})$, one has $\mathbf{v}(\mathbf{x}, \mathbf{h}, \mathbf{a})\in\mathcal{P}(\mathbf{h}, \mathbf{a})$. As a result, since $\mathbf{x}(t_1, j_1 +1)\in\mathcal{N}_{\gamma}(\mathcal{D}_{r_a}(\mathcal{O}_i))\cap\mathcal{P}(\mathbf{h}, \mathbf{a})$, using \eqref{state_belongs_to_gamma_neighborhood}, one can conclude that 
\[
\mathbf{x}(t, j)\in\mathcal{N}_{\gamma}(\mathcal{D}_{r_a}(\mathcal{O}_i))\cap\mathcal{P}(\mathbf{h}, \mathbf{a}), 
\]
for all $(t, j)\in(I_{j_1 +1}\times j_1 +1)$ and claim 1 in Lemma \ref{lemma:always_enter_move_to_target_mode} is satisfied.

Next, we proceed to prove claim 2 in Lemma \ref{lemma:always_enter_move_to_target_mode} which states that when $\xi(t_1, j_1 + 1)\in\mathcal{F}_1$ for some $(t_1, j_1 + 1)\in\text{ dom }\xi$, the control input steers the state $\xi$ to the jump set of the \textit{obstacle-avoidance} mode $\mathcal{J}_1$ in finite time $(t_2, j_1 + 1)\succ(t_1, j_1 + 1)$ with $t_2<\infty.$

Let us define the set $\mathcal{O}_i^{\mathcal{S}} = \mathcal{D}_{r_a}(\mathcal{O}_i)\cap\mathcal{P}(\mathbf{h}, \mathbf{a})$, as shown in Fig. \ref{figure_proof_neighborhood_partition}. Since obstacle $\mathcal{O}_i$ is convex, the set $\mathcal{O}_i^{\mathcal{S}}$ is also convex. As a result, the target location has a unique closest point on the set $\mathcal{O}_i^{\mathcal{S}}$, represented by $\Pi(\mathbf{0}, \mathcal{O}_i^{\mathcal{S}}).$ 
We define a set $\mathcal{LS}$ as follows:
\begin{equation}
    \mathcal{LS}:= \mathcal{L}(\mathbf{0}, \mathbf{0}_{\pi}^{\mathcal{S}})\cap\mathcal{N}_{\gamma}(\mathcal{D}_{r_a}(\mathcal{O}_i))\cap\mathcal{H}_{\geq}(\mathbf{0}_{\pi}^{\mathcal{S}}, -\mathbf{0}_{\pi}^{\mathcal{S}}),\label{proof_line_segment}
    \end{equation}
where $\mathbf{0}_{\pi}^{\mathcal{S}} = \Pi(\mathbf{0}, \mathcal{O}_i^s)$.
Since $\mathbf{0}\in\mathcal{P}(\mathbf{h}, \mathbf{a}),$ the line segment $\mathcal{LS}$ belongs to the plane $\mathcal{P}(\mathbf{h}, \mathbf{a}).$ Since $\mathcal{LS}\cap\mathcal{D}_{r_a}^{\circ}(\mathcal{O}_i) = \emptyset$,  the line segment $\mathcal{LS}$ also belongs to the \textit{exit} region $\mathcal{R}_e$ \eqref{exit_region}. Since the \textit{hit point} $\mathbf{h}$ belongs to $\mathcal{R}_a^i$, the target location $\mathbf{0}$ is closer to the location $\Pi(\mathbf{0}, \mathcal{O}_i^{\mathcal{S}})$ than to the \textit{hit point} $\mathbf{h}$. Hence, if $\mathbf{0}\notin\mathcal{D}_{r_a + \gamma}(\mathcal{O}_i)$, then for a sufficiently small value of $\bar{\epsilon}$, used in \eqref{partition_rm}, one can ensure that the set $\mathcal{LS}$ belongs to the set $\mathcal{J}_{1}^{\mathcal{W}}$ \eqref{jumpset_obstacleavoidance}. On the other hand, if $\mathbf{0}\in\mathcal{N}_{\gamma}(\mathcal{D}_{r_a}(\mathcal{O}_i))$, it is straightforward to verify that $\mathcal{LS}\subset\mathcal{S}_0$, which, according to \eqref{jumpset_obstacleavoidance}, implies that $\mathcal{LS}\subset\mathcal{J}_1^{\mathcal{W}}$.

Now, if one ensures that the state $\mathbf{x}$, which belongs to the set $\mathcal{N}_{\gamma}(\mathcal{D}_{r_a}(\mathcal{O}_i))\cap\mathcal{P}(\mathbf{h}, \mathbf{a})$ after time $(t_1, j_1)$, in the \textit{obstacle-avoidance} mode around obstacle $\mathcal{O}_i$, eventually intersects the set $\mathcal{LS}$ at some finite time $(t_2, j_1 +1)\succ(t_1, j_1+1)$, then it will imply that $\xi(t_2, j_1+1)\in\mathcal{J}_1$, and claim 2 in Lemma \ref{lemma:always_enter_move_to_target_mode} will be proven. To that end, let us divide the set $\mathcal{N}_{\gamma}(\mathcal{D}_{r_a}(\mathcal{O}_i))\cap\mathcal{P}(\mathbf{h}, \mathbf{a})$, as shown in Fig. \ref{figure_proof_neighborhood_partition}, into 3 separate subsets as follows:
\begin{equation}
\begin{aligned}
    &\mathcal{N}_{\gamma}(\mathcal{D}_{r_a}(\mathcal{O}_i))\cap\mathcal{P}(\mathbf{h}, \mathbf{a}) = \mathcal{S}_1\cup\mathcal{S}_2\cup\mathcal{S}_3,
    \end{aligned}
\end{equation}
where the sets $\mathcal{S}_1, \mathcal{S}_2$ and $\mathcal{S}_3$ are defined as follows:

    \begin{equation}
\begin{aligned}
    &\mathcal{S}_1 = \mathcal{N}_{\gamma_a}(\mathcal{D}_{r_a}(\mathcal{O}_i))\cap\mathcal{P}(\mathbf{h}, \mathbf{a}),\\
    &\mathcal{S}_2 = \left(\mathcal{N}_{\gamma_s-\gamma_a}(\mathcal{D}_{r_a + \gamma_a}(\mathcal{O}_i))\right)^{\circ}\cap\mathcal{P}(\mathbf{h}, \mathbf{a}),\\
    &\mathcal{S}_3 = \mathcal{N}_{\gamma - \gamma_s}(\mathcal{D}_{r_a + \gamma_s}(\mathcal{O}_i))\cap\mathcal{P}(\mathbf{h}, \mathbf{a}),
\end{aligned}
\end{equation}
where $0<\gamma_a<\gamma_s<\gamma.$

\begin{figure}
    \centering
    \includegraphics[width = 1\linewidth]{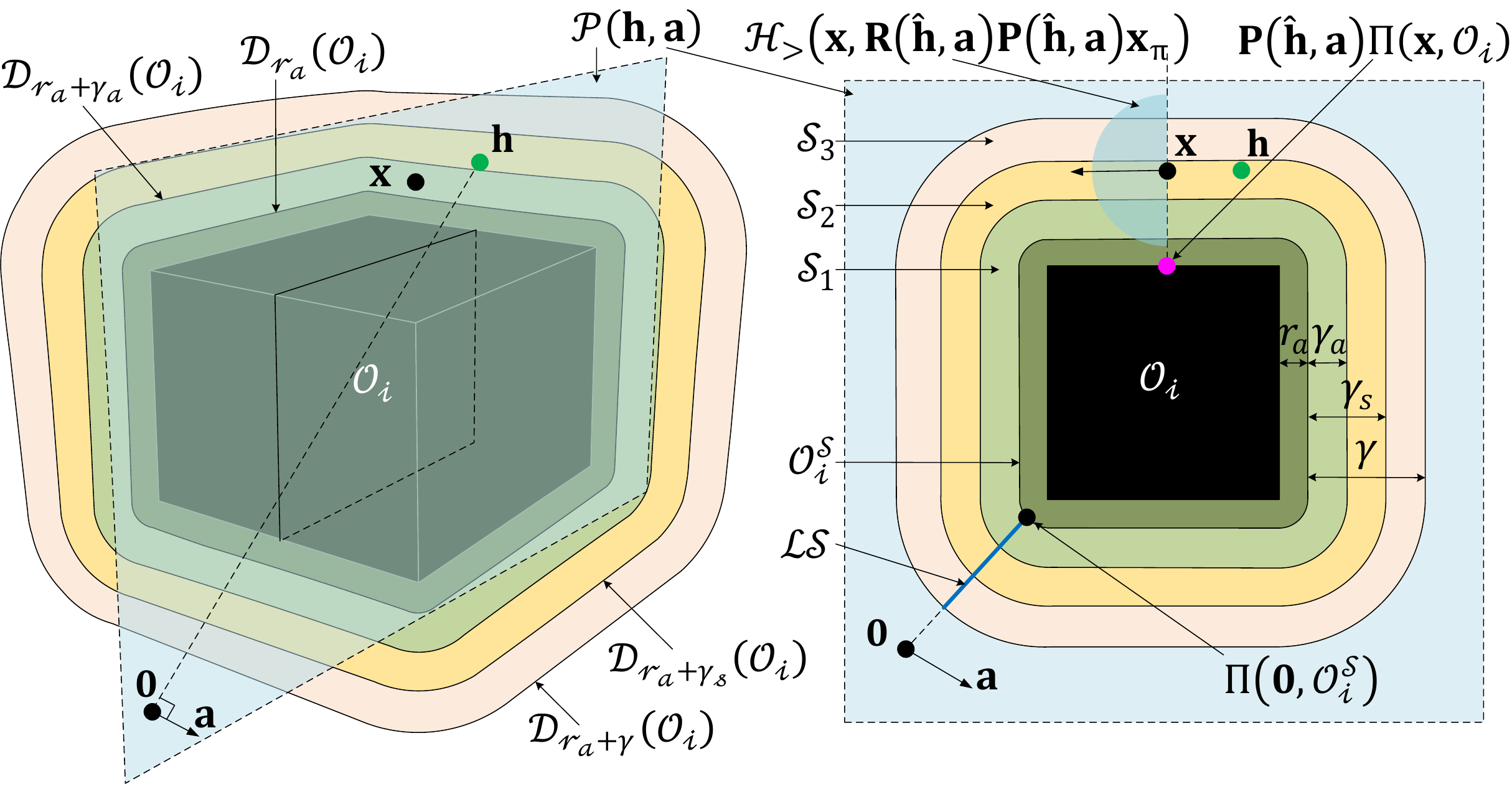}
    \caption{The partition of the set $\mathcal{N}_{\gamma}(\mathcal{D}_{r_a}(\mathcal{O}_i))\cap\mathcal{P}(\mathbf{h}, \mathbf{a})$.}
\label{figure_proof_neighborhood_partition}
\end{figure}

We show that when the control input corresponds to the \textit{obstacle-avoidance} mode and the state $\mathbf{x}$ belongs either to the set $\mathcal{S}_1$ or to the set $\mathcal{S}_3$, the control eventually steers the state $\mathbf{x}$ to the set $\mathcal{S}_2.$ 
Then, we show that for all $\mathbf{x}\in\mathcal{S}_2$, the control vector $\mathbf{u}(\xi)$ belongs to the open positive half-space $\mathcal{P}_{>}(\mathbf{0}, \mathbf{R}(\hat{\mathbf{h}}, \mathbf{a})\mathbf{P}(\hat{\mathbf{h}}, \mathbf{a})\mathbf{x}_{\pi})$.
This implies that the state $\mathbf{x}$, which belongs to the set $\mathcal{N}_{\gamma}(\mathcal{D}_{r_a}(\mathcal{O}_i))\cap\mathcal{P}(\mathbf{h}, \mathbf{a})$ after time $(t_1, j_1)$, in the \textit{obstacle-avoidance} mode around obstacle $\mathcal{O}_i$, is always steered to the open positive half-space $\mathcal{P}_{>}(\mathbf{0}, \mathbf{R}(\hat{\mathbf{h}}, \mathbf{a})\mathbf{P}(\hat{\mathbf{h}}, \mathbf{a})\mathbf{x}_{\pi})$ and will eventually reach the set $\mathcal{LS}$ at some finite time $(t_2, j_1 +1)\succ(t_1, j_1 +1).$ 

First, we show that when the control input corresponds to the \textit{obstacle-avoidance} mode and the state $\mathbf{x}$ is either in the set $\mathcal{S}_1$ or in the set $\mathcal{S}_3$, the control will eventually steer the state $\mathbf{x}$ to the set $\mathcal{S}_2.$ 

When the control input corresponds to the \textit{obstacle-avoidance} mode and $\mathbf{x}$ belongs to the set $\mathcal{S}_1$, the control vector $\mathbf{u}(\xi)$ in \eqref{control_u} becomes
\begin{equation}
    \mathbf{u}(\xi) = \kappa_r\mathbf{P}(\hat{\mathbf{h}}, \mathbf{a})\mathbf{x}_{\pi}, \kappa_r > 0.\label{control_law_S1}
\end{equation}
Let $\mathbf{x}\in\partial\mathcal{D}_{r_a+ \beta}(\mathcal{O}_i)\cap\mathcal{S}_1$ for some $\beta\in[0, \gamma_a]$. We know that for $\mathbf{x}\in\partial\mathcal{D}_{r_a + \beta}(\mathcal{O}_i)\cap\mathcal{S}_1$, the tangent cone to the set $\mathcal{N}_{\gamma - \beta}(\mathcal{D}_{r_a + \beta}(\mathcal{O}_i))$ at $\mathbf{x}$ is given by
\[
    \mathbf{T}_{\mathcal{N}_{\gamma - \beta}(\mathcal{D}_{r_a + \beta}(\mathcal{O}_i))}(\mathbf{x}) = \mathcal{H}_{\geq}(\mathbf{0}, \mathbf{x}_{\pi}).
\]
If we show that for all $\mathbf{x}\in\partial\mathcal{D}_{r_a + \beta}(\mathcal{O}_i)\cap\mathcal{S}_1$, one has $\mathbf{x}_{\pi}\mathbf{P}(\hat{\mathbf{h}}, \mathbf{a})\mathbf{x}_{\pi} > 0$, then it implies that the control input vector \eqref{control_law_S1} steers $\mathbf{x}$ to the interior of the set $\mathcal{N}_{\gamma - \beta}(\mathcal{D}_{r_a + \beta}(\mathcal{O}_i))$. This, combined with the fact that $\mathbf{x}(t, j)\in\mathcal{P}(\mathbf{h}, \mathbf{a})$, for all $(t, j)\in(I_{j_1 + 1}\times j_1 + 1)$, as per claim 1 in Lemma \ref{lemma:always_enter_move_to_target_mode}, ensures that the control input vector \eqref{control_law_S1} steers $\mathbf{x}$ to the interior of the set $(\mathcal{S}_1\cup\mathcal{S}_2)\setminus\mathcal{D}_{r_a+\beta}^{\circ}(\mathcal{O}_i)$ and eventually  $\mathbf{x}$ will enter in the set $\mathcal{S}_2$. To proceed with the proof we require the following fact:

\textbf{Fact 1}: Consider a plane $\mathcal{P}(\mathbf{p}, \mathbf{q})$, where $\mathbf{p}\in\mathcal{J}_0^{\mathcal{W}}\cap\mathcal{N}_{\gamma}(\mathcal{D}_{r_a}(\mathcal{O}_i))$, for some $i\in\mathbb{I}$ and $\mathbf{q}\in \mathbf{A}(\mathbf{p})$, where the mapping $\mathbf{A}$ is defined in \eqref{update_law_for_vector_a}. Then, for all $\mathbf{x}\in\mathcal{N}_{\gamma}(\mathcal{D}_{r_a}(\mathcal{O}_i))\cap\mathcal{P}(\mathbf{p}, \mathbf{q})$, one has $\mathbf{x}_{\pi}^\top\mathbf{P}(\hat{\mathbf{p}}, \mathbf{q})\mathbf{x}_{\pi} > 0$, where $\mathbf{x}_{\pi} = \mathbf{x} - \Pi(\mathbf{x}, \mathcal{O}_{\mathcal{W}})$. \label{fact1}

\begin{proof}
    This proof is by contradiction. First, note that for any vector $\mathbf{s}\in\mathbb{R}^n$, one has $\mathbf{s}^\top\mathbf{P}(\hat{\mathbf{p}}, \mathbf{q})\mathbf{s}\geq 0$. Let us assume that there exists $\mathbf{x}\in\mathcal{N}_{\gamma}(\mathcal{D}_{r_a}(\mathcal{O}_i))\cap\mathcal{P}(\mathbf{p}, \mathbf{q})$ such that $\mathbf{x}_{\pi}\mathbf{P}(\hat{\mathbf{p}}, \mathbf{q})\mathbf{x}_{\pi} = 0$. Since $\mathbf{x}_{\pi}\ne\mathbf{0}$, one has $\mathbf{P}(\hat{\mathbf{p}}, \mathbf{q})\mathbf{x}_{\pi} = \mathbf{0}$. Therefore, the vector $\mathbf{x}_{\pi}$ is normal to the plane $\mathcal{P}(\mathbf{p}, \mathbf{q})$. This implies that the plane $\mathcal{P}(\mathbf{p}, \mathbf{q})$ is a supporting hyperplane \cite[Section 2.5.2]{boyd2004convex} to the convex set $\mathcal{D}_{r_a + \beta}(\mathcal{O}_i)$ at $\mathbf{x}$, where $\beta = d(\mathbf{x},\mathcal{O}_i) - r_a\in[0, \gamma]$. Therefore, the set $(\mathcal{D}_{r_a}(\mathcal{O}_i))^{\circ}\cap\mathcal{P}(\mathbf{p}, \mathbf{q})$ is an empty set.

    However, since $\mathbf{q} \in \mathbf{A}(\mathbf{p})$, according to \eqref{update_law_for_vector_a}, one has $\mathbf{p}_{\pi}\in\mathcal{P}(\mathbf{p}, \mathbf{q})$. Therefore, $\mathcal{L}(\mathbf{p}, \Pi(\mathbf{p}, \mathcal{O}_i))\subset\mathcal{P}(\mathbf{p}, \mathbf{q})$. As a result, $\mathcal{L}(\mathbf{p}, \Pi(\mathbf{p}, \mathcal{O}_i))\cap\mathcal{D}_{r_a}^{\circ}(\mathcal{O}_i)\ne\emptyset$. This implies that $\mathcal{D}_{r_a}^{\circ}(\mathcal{O}_i)\cap\mathcal{P}(\mathbf{p}, \mathbf{q})\ne\emptyset$, which is a contradiction.
\end{proof}

According to Fact 1, for all $\mathbf{x}\in\partial\mathcal{D}_{r_a + \beta}(\mathcal{O}_i)\cap\mathcal{S}_1$, where $\beta\in[0, \gamma_a]$, one has $\mathbf{x}_{\pi}\mathbf{P}(\hat{\mathbf{h}}, \mathbf{a})\mathbf{x}_{\pi} > 0$. Therefore, as discussed earlier,  the control input vector \eqref{control_law_S1} steers $\mathbf{x}$ to the interior of the set $(\mathcal{S}_1\cup\mathcal{S}_2)\setminus\mathcal{D}_{r_a+\beta}^{\circ}(\mathcal{O}_i)$ and eventually  $\mathbf{x}$ will enter in the set $\mathcal{S}_2$.

Similarly, when the control input corresponds to the \textit{obstacle-avoidance} mode and $\mathbf{x}$ belongs to the set $\mathcal{S}_3$, the control vector $\mathbf{u}(\xi)$ in \eqref{control_u} is given by
{
\begin{equation}
    \mathbf{u}(\xi) = -\kappa_r{\mathbf{P}(\hat{\mathbf{h}}, \mathbf{a})\mathbf{x}_{\pi}}, \kappa_r > 0.\label{control_law_S3}
\end{equation}}
Let $\mathbf{x}\in\partial\mathcal{D}_{r_a+\beta}(\mathcal{O}_i)\cap\mathcal{S}_3$ for some $\beta\in[\gamma_s, \gamma]$. We know that for $\mathbf{x}\in\partial\mathcal{D}_{r_a + \beta}(\mathcal{O}_i)\cap\mathcal{S}_3$, the tangent cone to the set $\mathcal{D}_{r_a + \beta}(\mathcal{O}_i)$ at $\mathbf{x}$ is given by
\[
    \mathbf{T}_{\mathcal{D}_{r_a + \beta}(\mathcal{O}_i)}(\mathbf{x}) = \mathcal{H}_{\leq}(\mathbf{0}, \mathbf{x}_{\pi}).
\]
According to Fact 1, for all $\mathbf{x}\in\partial\mathcal{D}_{r_a + \beta}(\mathcal{O}_i)\cap\mathcal{S}_3$, where $\beta\in[\gamma_s, \gamma]$, one has $\mathbf{x}_{\pi}\mathbf{P}(\hat{\mathbf{h}}, \mathbf{a})\mathbf{x}_{\pi} > 0$. 
This implies that the control input vector \eqref{control_law_S3} steers $\mathbf{x}$ to the interior of the set $\mathcal{D}_{r_a + \beta}(\mathcal{O}_i)$. This, combined with the fact that $\mathbf{x}(t, j)\in\mathcal{P}(\mathbf{h}, \mathbf{a})$, for all $(t, j)\in(I_{j_1 + 1}\times j_1 + 1)$, as per claim 1 in Lemma \ref{lemma:always_enter_move_to_target_mode}, ensures that the control input vector \eqref{control_law_S3} steers $\mathbf{x}$ to the interior of the set $(\mathcal{S}_3\cup\mathcal{S}_2)\cap\mathcal{D}_{r_a+\beta}^{\circ}(\mathcal{O}_i)$ and eventually  $\mathbf{x}$ will enter in the set $\mathcal{S}_2$.

Finally, we show that when the control input corresponds to the \textit{obstacle-avoidance} mode and the state $\mathbf{x}$ belongs to the set $\mathcal{S}_2$, the control vector $\mathbf{u}(\xi)$ belongs to the open positive half-space $\mathcal{H}_{>}(\mathbf{0}, \mathbf{R}(\hat{\mathbf{h}},\mathbf{a})\mathbf{P}(\hat{\mathbf{h}}, \mathbf{a})\mathbf{x}_{\pi})$. 
When the control law corresponds to the \textit{obstacle-avoidance} mode and the state $\mathbf{x}\in\mathcal{S}_2$, according to \eqref{hybrid_control_law}, one has $\mathbf{u}(\xi) = \kappa_r\mathbf{v}(\mathbf{x}, \mathbf{h}, \mathbf{a}), \kappa_r > 0.$ Note that for all $\mathbf{x}\in\mathcal{S}_2$, one has $\eta(\mathbf{x})\in(-1, 1)$. Therefore, for every $\mathbf{x}\in\mathcal{S}_2$, the vector $\mathbf{v}(\mathbf{x}, \mathbf{h}, \mathbf{a})$ can be expressed as a linear combination of the vectors $\mathbf{P}(\hat{\mathbf{h}}, \mathbf{a})\mathbf{x}_{\pi}$ and $\mathbf{R}(\hat{\mathbf{h}}, \mathbf{a})\mathbf{P}(\hat{\mathbf{h}}, \mathbf{a})\mathbf{x}_{\pi}$ given by
\begin{equation}
    \mathbf{v}(\mathbf{x}, \mathbf{h}, \mathbf{a}) = k_1\mathbf{P}(\hat{\mathbf{h}}, \mathbf{a})\mathbf{x}_{\pi} + k_2\mathbf{R}(\hat{\mathbf{h}}, \mathbf{a})\mathbf{P}(\hat{\mathbf{h}}, \mathbf{a})\mathbf{x}_{\pi},
\end{equation}
where $k_1\in\mathbb{R}$ and $k_2> 0$. Additionally, according to Fact 1, for all $\mathbf{x}\in\mathcal{S}_2$, one has $\mathbf{P}(\hat{\mathbf{h}}, \mathbf{a})\mathbf{x}_{\pi} \ne \mathbf{0}.$ As a result, it can be confirmed that $\mathbf{v}(\mathbf{x}, \mathbf{h}, \mathbf{a})^\intercal\mathbf{R}(\hat{\mathbf{h}}, \mathbf{a})\mathbf{P}(\hat{\mathbf{h}}, \mathbf{a})\mathbf{x}_{\pi} > 0$, when the state $\mathbf{x}$ belongs to the set $\mathcal{S}_2,$ and the proof is complete.

\subsection{Proof of Theorem \ref{theorem:global_stability}}
\label{proof_of_theorem}

\textbf{Forward invariance and stability:} The forward invariance of the obstacle-free set $\mathcal{K}$, for the hybrid closed-loop system \eqref{hybrid_closed_loop_system}, is immediate from Lemma \ref{lemma:set_invariance}. 
We next prove the stability of $\mathcal{A}$ using \cite[Definition 3.1]{sanfelice2021hybrid}.

Since $\mathbf{0}\in(\mathcal{W}_{r_a})^{\circ}$, there exists $\mu_1 > 0$ such that
$\mathcal{B}_{\mu_1}(\mathbf{0})\cap(\mathcal{D}_{r_a}(\mathcal{O}_{\mathcal{W}}))^{\circ} = \emptyset.$ According to  \eqref{jumpset_movetotarget}, there exists $\mu_2 > 0$ such that $\mathcal{B}_{\mu_2}(\mathbf{0})\cap\mathcal{J}_0^{\mathcal{W}}=\emptyset$. 
Additionally, as per \eqref{jumpset_obstacleavoidance} and \eqref{line_of_sight_to_target}, there exists $\mu_3>0$ such that $\mathcal{B}_{\mu_3}(\mathbf{0})\subset\mathcal{J}_1^{\mathcal{W}}$.
We define the set $\mathcal{S}_{\mu}:= \{\xi\in\mathcal{K}|\mathbf{x}\in\mathcal{B}_{\mu}(\mathbf{0})\},$
where $\mu\in(0, \min\{\mu_1, \mu_2, \mu_3\}).$ Notice that for all initial conditions $\xi(0, 0)\in\mathcal{S}_{\mu}$, the control input, after at most one jump corresponds to the \textit{move-to-target} mode and steers $\mathbf{x}$ towards the origin with the control input vector $\mathbf{u}(\xi) = -\kappa_s\mathbf{x}, \kappa_s> 0$. Hence, for each $\mu\in(0, \min\{\mu_1, \mu_2, \mu_3\}),$ the set $\mathcal{S}_{\mu}$ is forward invariant for the hybrid closed-loop system \eqref{hybrid_closed_loop_system}.
 
Consequently, for every $\rho>0$, one can choose $\sigma\in(0, \min\{\mu_1, \mu_2, \mu_3, \rho\})$ such that for all initial conditions $\xi(0, 0)$ with $d(\xi(0, 0), \mathcal{A})\leq\sigma$, one has $d(\xi(t, j), \mathcal{A})\leq \rho$ for all $(t, j)\in\text{ dom }\xi,$ where $d(\xi, \mathcal{A})^2 = \underset{(\mathbf{0}, \bar{\mathbf{h}}, \bar{\mathbf{a}}, \bar{m}, \bar{s})\in\mathcal{A}}{\inf}(\norm{\mathbf{x}}^2 + \norm{\mathbf{h} - \bar{\mathbf{h}}}^2 + \norm{\mathbf{a} - \bar{\mathbf{a}}}^2+(m - \bar{m})^2 + (s - \bar{s})^2) = \norm{\mathbf{x}}^2.$ 
Hence, according to \cite[Definition 3.1]{sanfelice2021hybrid}, the target set $\mathcal{A}$ is stable for the hybrid closed-loop system \eqref{hybrid_closed_loop_system}.
Next, we proceed to establish the convergence properties of the set $\mathcal{A}$.

\textbf{Attractivity:} We aim to show that for the proposed hybrid closed-loop system \eqref{hybrid_closed_loop_system}, the target set $\mathcal{A}$ is globally attractive in the set $\mathcal{K}$ using \cite[Defintion 3.1 and Remark 3.5]{sanfelice2021hybrid}. In other words, we prove that for all initial conditions $\xi(0, 0)\in\overline{\mathcal{F}}\cup\mathcal{J} = \mathcal{K}$, every maximal solution $\xi$ to the hybrid closed-loop system is complete and satisfies 
\begin{equation}
    \underset{(t, j)\in\text{ dom }\xi, t + j \to \infty}{\lim} d(\xi(t, j), \mathcal{A}) = \norm{\mathbf{x}(t, j)} = 0.\label{condition_for_attractivity}
\end{equation}

The completeness of all maximal solutions to the hybrid closed-loop system \eqref{hybrid_closed_loop_system} follows from Lemma \ref{lemma:set_invariance}.
Next, we prove that for all initial condition $\xi(0, 0)\in\mathcal{K}$, every complete solution $\xi$ to the hybrid closed-loop system \eqref{hybrid_closed_loop_system}, satisfies \eqref{condition_for_attractivity}.
We consider two cases based on the initial value of the mode indicator variable $m(0, 0)$. 

\textbf{Case 1}: $m(0, 0) = 0$. For the hybrid closed-loop system \eqref{hybrid_closed_loop_system}, consider a solution $\xi$ initialized in the \textit{move-to-target} mode. Let us assume $\xi(t_0, j_0)\in\mathcal{F}_0$ for some $(t_0, j_0)\in\text{ dom }\xi, (t_0, j_0)\succeq(0, 0).$ If $\xi(t, j)\notin\mathcal{J}_0$, for all $(t, j)\succeq(t_0, j_0)$, then the control input $\mathbf{u}(\mathbf{x}, \mathbf{h}, \mathbf{a}, 0, s) = -\kappa_s\mathbf{x}$ with $\kappa_s > 0$ will steer the state $\mathbf{x}$ straight towards the origin. On the other hand, assume that there exists $(t_1, j_1)\succeq(t_0, j_0)$ such that $\xi(t_1, j_1)\in\mathcal{J}_0$. Then, according to \eqref{updatelaw_part1}, the control law switches to the \textit{obstacle-avoidance} mode.
As per \eqref{jumpset_movetotarget}, it is clear that $\mathbf{x}(t_1, j_1)\in\mathcal{J}_0^{\mathcal{W}}\cap\mathcal{N}_{\gamma_s}(\mathcal{D}_{r_a}(\mathcal{O}_i)),$ for some $i\in\mathbb{I}$.
At this instance, according to \eqref{updatelaw_part1}, the proposed navigation algorithm updates the values of the state variables $\mathbf{h}(t_1, j_1 + 1) = \mathbf{x}(t_1, j_1)$, $\mathbf{a}(t_1, j_1 + 1) \in\mathbf{A}(\mathbf{x}(t_1, j_1))$, $m(t_1, j_1 +1) = 1$ and $s(t_1, j_1 + 1) = s(t_1, j_1) + 1.$ 
According to \eqref{hybrid_closed_loop_system}, $\mathbf{h}(t_1, j_1 +1) = \mathbf{h}(t, j)$, $\mathbf{a}(t_1, j_1 +1) = \mathbf{a}(t, j)$ and $m(t_1, j_1 +1) = m(t, j)$ for all $(t, j)\in(I_{j_1 +1}\times j_1 +1)$, where $I_{j_1+1} = \{t|(t, j_1 +1)\in\text{ dom }\xi\}$.

According to Lemma \ref{lemma:always_enter_move_to_target_mode}, there exists $(t_2, j_1 + 1)\succ(t_1, j_1 + 1)$ with $t_2<\infty$ such that $\xi(t_2, j_1 + 1)\in\mathcal{J}_1$. Notice that, according to \eqref{jumpset_obstacleavoidance} and \eqref{partition_rm}, one has $\|\mathbf{x}(t_2, j_1 + 1)\| < \|\mathbf{x}(t_1, j_1 + 1)\|$. In other words, according to Lemma \ref{lemma:always_enter_move_to_target_mode}, the proposed navigation algorithm ensures that, at the instance where the control switches from the \textit{obstacle-avoidance} mode to the \textit{move-to-target} mode, the origin is closer to the point $\mathbf{x}$ than to the last point where the control switched to the \textit{obstacle-avoidance} mode. 
Furthermore, when the control input corresponds to the \textit{move-to-target} mode, it steers the state $\mathbf{x}$ towards the origin under the influence of control $\mathbf{u}(\xi) = -\kappa_s\mathbf{x}, \kappa_s > 0.$ Consequently, given that the workspace $\mathcal{W}$ and the obstacles $\mathcal{O}_i, i\in\mathbb{I}\setminus\{0\}$, are compact, it can be concluded that the solution $\xi(t, j)$ will have a finite number of jumps and will satisfy \eqref{condition_for_attractivity}.

\textbf{Case 2}: $m(0, 0) = 1.$ For the hybrid closed-loop system \eqref{hybrid_closed_loop_system}, consider a solution $\xi$ initialized in the \textit{obstacle-avoidance} mode. Since $m(0, 0)= 1$, according to \eqref{oa_jumpset} and \eqref{J_s_set_definition}, $\xi(0, 0)\in\mathcal{J}_1$. Therefore, according to \eqref{updatelaw_part2}, the control input switches to the \textit{move-to-target} mode and $m(0, 1) = 0.$ One can now use arguments similar to the ones used for case 1 to show that the solution $\xi(t, j)$ will have a finite number of jumps and will satisfy \eqref{condition_for_attractivity}.

Hence, the target set $\mathcal{A}$ is globally attractive in the set $\mathcal{K}$ for the proposed hybrid closed-loop system \eqref{hybrid_closed_loop_system}. In addition, since the set $\mathcal{A}$ is stable for the hybrid closed-loop system \eqref{hybrid_closed_loop_system}, it is globally asymptotically stable in the set $\mathcal{K}$ for the hybrid closed-loop system \eqref{hybrid_closed_loop_system} as per \cite[Remark 3.5]{sanfelice2021hybrid}.

\subsection{Proof of Theorem \ref{theorem:sphere_global_stability}}

\label{proof:sphere_world}
\textbf{Forward invariance (Safety):} Note that, according to Lemma \ref{lemma:hybrid_basic_conditions_for_sphere}, the hybrid closed-loop system \eqref{modified_hybrid_closed_loop_system} satisfies the hybrid basic conditions. Furthermore, the modified control input vector $\mathbf{u}_{s}(\xi)$, as defined in \eqref{modified_control_u}, is obtained by replacing the avoidance control vector $\mathbf{v}(\mathbf{x}, \mathbf{h}, \mathbf{a})$ \eqref{n_dimensional_obstacle-avoidance_vector} with $\mathbf{v}_s(\mathbf{x}, \mathbf{h}, \mathbf{a})$ \eqref{modified_obstacle-avoidance_control} in \eqref{control_u}. Hence, demonstrating that the hybrid closed-loop system \eqref{modified_hybrid_closed_loop_system} satisfies the viability condition, as mentioned in \eqref{viability_condition}, for $m = 1$, allows one to employ similar arguments from the proof of Lemma \ref{lemma:set_invariance} to establish the forward invariance of the set $\mathcal{K}$ for the hybrid closed-loop system \eqref{modified_hybrid_closed_loop_system}.
In other words, we want to show that for all $\xi\in\mathcal{F}_1\setminus \mathcal{J}_1$ 
\begin{equation}
\mathbf{F}_s(\xi)\cap\mathbf{T}_{\mathcal{F}_1}(\xi) \ne\emptyset.\label{modified_viability_condition}
\end{equation}

For all $\xi\in\mathcal{K}$ such that $\mathbf{x}\in(\mathcal{F}_1^{\mathcal{W}})^{\circ}\setminus\mathcal{J}_1$ and $m = 1$, the tangent cone $\mathbf{T}_{\mathcal{F}_1}(\xi) = \mathbb{R}^n\times\mathbf{T}_{\mathcal{W}_{r_a}}(\mathbf{h})\times\mathcal{H}(\mathbf{0}, \mathbf{a})\times\{0\}\times\mathbf{T}_{\mathbb{R}_{\geq 0}}(s)$, where the sets $\mathbf{T}_{\mathcal{W}_{r_a}}(\mathbf{h})$ and $\mathbf{T}_{\mathbb{R}_{\geq 0}}(s)$ are defined in \eqref{definition_hp} and \eqref{definition_sp}, respectively.
Since, according to \eqref{modified_hybrid_closed_loop_system}, $\dot{\mathbf{h}}= \mathbf{0}\in\mathbf{T}_{\mathcal{W}_{r_a}}(\mathbf{h})$, $\dot{s} = 1 \in\mathbf{T}_{\mathbb{R}_{\geq 0}}(s)$ and $\dot{\mathbf{a}} = \mathbf{0}\in\mathcal{H}(\mathbf{0}, \mathbf{a})$, the viability condition in \eqref{modified_viability_condition} holds true for all $\xi\in\mathcal{K}$ such that $\mathbf{x}\in(\mathcal{F}_1^{\mathcal{W}})^{\circ}\setminus\mathcal{J}_1$ and $m = 1$. 

Finally, for all $\xi\in\mathcal{K}$ with $\mathbf{x}\in\partial\mathcal{F}_1\setminus\mathcal{J}_1$ and $m = 1,$ the tangent cone $\mathbf{T}_{\mathcal{F}_1}(\xi)$ is given by
\begin{equation}
    \mathbf{T}_{\mathcal{F}}(\xi) = \mathcal{H}_{\geq}(\mathbf{0}, \mathbf{x}_{\pi})\times\mathbf{T}_{\mathcal{W}_{r_a}}(\mathbf{h})\times\mathcal{H}(\mathbf{0},\mathbf{a})\times\{0\}\times\mathbf{T}_{\mathbb{R}_{\geq 0}}(s),
\end{equation}
where the sets $\mathbf{T}_{\mathcal{W}_{r_a}}(\mathbf{h})$ and $\mathbf{T}_{\mathbb{R}_{\geq 0}}(s)$ are defined in \eqref{definition_hp} and \eqref{definition_sp}, respectively. According to \eqref{modified_control_u}, for $m = 1$, $\mathbf{u}_s(\xi) = \kappa_r\mathbf{R}(\hat{\mathbf{h}}, \mathbf{a})\mathbf{P}(\hat{\mathbf{h}}, \mathbf{a})\mathbf{x}_{\pi}, \kappa_r > 0.$ 
Note that, for all $\mathbf{x}\in\mathcal{N}_{\gamma}(\mathcal{D}_{r_a}(\mathcal{O}_i)),$ for each $i\in\mathbb{I}\setminus\{0\}$, one has $\mathbf{x}_{\pi}\in\mathcal{H}(\mathbf{0}, \mathbf{R}(\hat{\mathbf{h}}, \mathbf{a})\mathbf{P}(\hat{\mathbf{h}}, \mathbf{a})\mathbf{x}_{\pi})$, where $\mathbf{a}\in\mathbb{S}^{n-1}$, and consequently $\mathbf{u}_s(\xi)^\top\mathbf{x}_{\pi} = 0. $
Additionally, according to \eqref{modified_hybrid_closed_loop_system}, $\dot{\mathbf{h}}= \mathbf{0}\in\mathbf{T}_{\mathcal{W}_{r_a}}(\mathbf{h})$, $\dot{s} = 1\in\mathbf{T}_{\mathbb{R}_{\geq 0}}(s)$ and $\dot{\mathbf{a}} = \mathbf{0}\in\mathcal{H}(\mathbf{0}, \mathbf{a})$. Hence, the viability condition in \eqref{modified_viability_condition} holds true for all $\xi\in\mathcal{K}$ such that $\mathbf{x}\in\partial\mathcal{F}_1\setminus\mathcal{J}_1$ and $m = 1$, and as such it holds true for all $\xi\in\mathcal{F}_1\setminus\mathcal{J}_1.$

\textbf{Stability:} 
When the mode indicator variable $m = 0$, one has $\mathbf{u}_s(\xi) = \mathbf{u}(\xi) = -\kappa_s\mathbf{x}, \kappa_s > 0.$ Additionally, the target location at the origin $\mathbf{0}$ belongs to $\mathcal{W}_{r_a}^{\circ}$, and the definitions of the flow set $\mathcal{F}$ and the jump set $\mathcal{J}$ are the same for \eqref{hybrid_closed_loop_system} and \eqref{modified_hybrid_closed_loop_system}. Hence, one can use similar arguments from the proof of Theorem \ref{theorem:global_stability} to prove the stability of the target set $\mathcal{A}$.


\textbf{Attractivity:} If we prove that all solutions $\xi$ to the hybrid closed-loop system \eqref{modified_hybrid_closed_loop_system} satisfy Lemma \ref{lemma:always_enter_move_to_target_mode}, then one can use arguments similar to the ones in the proof of Theorem \ref{theorem:global_stability} to prove the attractivity of the target set $\mathcal{A}$ from any point in $\mathcal{K}$ for the hybrid closed-loop system \eqref{modified_hybrid_closed_loop_system}. Consequently, we proceed to prove that every solution $\xi$ to the hybrid closed-loop system \eqref{modified_hybrid_closed_loop_system} satisfies Lemma \ref{lemma:always_enter_move_to_target_mode}.


Since $\xi(t_1, j_1)\in\mathcal{J}_0$, one has $\mathbf{h}\in\mathcal{J}_0^{\mathcal{W}}\cap\partial\mathcal{D}_{r_a + \beta}(\mathcal{O}_i)$ for some $i\in\mathbb{I}\setminus\{0\}$ and $\beta\in[0, \gamma]$, where $\mathbf{h} = \mathbf{h}(t_1, j_1 +1) = \mathbf{h}(t, j)$ for all $(t, j)\in(I_{j_1 +1}\times j_1 + 1)$. First, we show that $\mathbf{x}(t, j)\in\mathcal{N}_{\gamma}(\mathcal{D}_{r_a}(\mathcal{O}_i))$ for all $(t, j)\in(I_{j_1 +1}\times j_1 + 1)$. For all $\mathbf{x}\in\partial\mathcal{D}_{r_a + \beta}(\mathcal{O}_i)$, the tangent cone to the set $\partial\mathcal{D}_{r_a + \beta}(\mathcal{O}_i)$ at $\mathbf{x}$ is given by
\begin{equation}
    \mathbf{T}_{\partial\mathcal{D}_{r_a + \beta}(\mathcal{O}_i)}(\mathbf{x}) = \mathcal{H}(\mathbf{0}, \mathbf{x}_{\pi}),
\end{equation}
where $\mathbf{x}_{\pi} = \mathbf{x} - \Pi(\mathbf{x}, \mathcal{O}_{\mathcal{W}})$. When the control input vector \eqref{modified_control_u} corresponds to the \textit{obstacle-avoidance} mode, for all $\mathbf{x}\in\partial\mathcal{D}_{r_a + \beta}(\mathcal{O}_i)$, one has $\mathbf{u}_s(\xi) = \kappa_r\mathbf{v}_s(\mathbf{x}, \mathbf{h}, \mathbf{a})$, $\kappa_r > 0$ and $\mathbf{a} = \mathbf{a}(t_1 , j_1 + 1) = \mathbf{a}(t, j)$ for all $(t, j)\in(I_{j_1 +1}\times j_1 + 1)$. Now, using \eqref{modified_obstacle-avoidance_control}, one can conclude that for all $\mathbf{x}\in\partial\mathcal{D}_{r_a + \beta}(\mathcal{O}_i)$, $\mathbf{v}_s(\mathbf{x}, \mathbf{h}, \mathbf{a})^\top\mathbf{x}_{\pi} = 0$ and $\mathbf{u}_s(\xi)\in\mathbf{T}_{\partial\mathcal{D}_{r_a + \beta}(\mathcal{O}_i)}(\mathbf{x})$. Additionally, as per Lemma \ref{lemma:hybrid_basic_conditions_for_sphere}, the control input trajectory $\mathbf{u}(\xi(t, j))$ is continuous when it corresponds to the \textit{obstacle-avoidance} mode. Therefore, using Nagumo's theorem \cite[Theorem 4.7]{blanchini2008set}, one can conclude that for all $(t, j)\in(I_{j_1 + 1}\times j_1 +1)$
\begin{equation}\mathbf{x}(t, j)\in\partial\mathcal{D}_{r_a + \beta}(\mathcal{O}_i),\label{sphere3D_gamma_neighbourhood_forward_invariance}
\end{equation}
 where $\beta\in[0, \gamma]$. Therefore, $\mathbf{x}(t, j)\in\mathcal{N}_{\gamma}(\mathcal{D}_{r_a}(\mathcal{O}_i))$ for all $(t, j)\in(I_{j_1 + 1}\times j_1 + 1)$.

Next, we show that $\mathbf{x}(t, j)\in\mathcal{N}_{\gamma}(\mathcal{D}_{r_a}(\mathcal{O}_i))\cap\mathcal{P}(\mathbf{h}, \mathbf{a})$ for all $(t, j)\in(I_{j_1 + 1}\times j_1 +1)$. Since $\mathbf{a} \in \mathbf{A}_s(\mathbf{h})$ and obstacle $\mathcal{O}_i$ is a sphere, the plane $\mathcal{P}(\mathbf{h}, \mathbf{a})$ passes through the origin and the center $\mathbf{c}_i$ of obstacle $\mathcal{O}_i$. As a result, for all $\mathbf{x}\in\mathcal{N}_{\gamma}(\mathcal{D}_{r_a}(\mathcal{O}_i))\cap\mathcal{P}(\mathbf{h}, \mathbf{a})$, one has $\mathbf{x}_{\pi}\in\mathcal{P}(\mathbf{h}, \mathbf{a})$. In the \textit{obstacle-avoidance} mode, the control input is given by $\mathbf{u}_s(\xi) = \kappa_r\mathbf{v}_s(\mathbf{x}, \mathbf{h}, \mathbf{a})$, $\kappa_r > 0$. According to \eqref{modified_obstacle-avoidance_control}, it is clear that for all $\mathbf{x}\in\mathcal{N}_{\gamma}(\mathcal{D}_{r_a}(\mathcal{O}_i))\cap\mathcal{P}(\mathbf{h}, \mathbf{a})$, $\mathbf{v}_s(\mathbf{x}, \mathbf{h},\mathbf{a})\in\mathcal{P}(\mathbf{h}, \mathbf{a})$. As a result, since $\mathbf{x}(t_1, j_1 +1)\in\mathcal{N}_{\gamma}(\mathcal{D}_{r_a}(\mathcal{O}_i))\cap\mathcal{P}(\mathbf{h}, \mathbf{a})$, using \eqref{sphere3D_gamma_neighbourhood_forward_invariance}, one can conclude that 
\begin{equation}
    \mathbf{x}(t, j)\in\mathcal{N}_{\gamma}(\mathcal{D}_{r_a}(\mathcal{O}_i))\cap\mathcal{P}(\mathbf{h}, \mathbf{a}),\label{sphere3D_claim1_proved}
\end{equation}
for all $(t, j)\in(I_{j_1 + 1}\times j_1 + 1)$ and claim 1 in Lemma \ref{lemma:always_enter_move_to_target_mode} is satisfied.

We proceed to prove claim 2 in Lemma \ref{lemma:always_enter_move_to_target_mode}. We know that $\mathbf{h}\in\partial\mathcal{D}_{r_a + \beta}(\mathcal{O}_i)$ for some $i\in\mathbb{I}\setminus\{0\}$ and $\beta\in[0, \gamma]$. Additionally, according to \eqref{sphere3D_gamma_neighbourhood_forward_invariance} and \eqref{sphere3D_claim1_proved}, it is clear that for all $(t, j)\in(I_{j_1 +1}\times j_1 +1)$, $\mathbf{x}(t, j)\in\partial\mathcal{D}_{r_a + \beta}(\mathcal{O}_i)\cap\mathcal{P}(\mathbf{h}, \mathbf{a})$, where $\mathbf{a} \in \mathbf{A}_s(\mathbf{h})$. Moreover, for all $(t, j)\in(I_{j_1 +1}\times j_1 +1)$, one has $\mathbf{u}_s(\xi(t, j))\in\mathcal{H}_{>}(\mathbf{0}, \mathbf{R}(\hat{\mathbf{h}}, \mathbf{a})\mathbf{x}_{\pi})$, where $\mathbf{u}_s(\xi(t, j)) = \kappa_r\mathbf{v}_s(\mathbf{x}(t, j), \mathbf{h}, \mathbf{a}), \kappa_r > 0$. Therefore, since obstacle $\mathcal{O}_i$ is compact, there exists $t_2 < \infty$ and $t_2 > t_1$ such that  $\mathbf{x}(t_2, j_1 + 1) = \Pi(\mathbf{0}, \partial\mathcal{D}_{r_a + \beta}(\mathcal{O}_i))$. 
Since obstacle $\mathcal{O}_i$ is a sphere, it is straightforward to verify that $\Pi(\mathbf{0}, \partial\mathcal{D}_{r_a + \beta}(\mathcal{O}_i))\in\mathcal{LS}$, where the set $\mathcal{LS}$ is defined according to \eqref{proof_line_segment}.
If $\mathbf{0}\notin\mathcal{D}_{r_a + \gamma}(\mathcal{O}_i)$, then for a sufficiently small value of $\bar{\epsilon}$, used in \eqref{partition_rm}, one can ensure that the set $\mathcal{LS}$ belongs to the set $\mathcal{J}_{1}^{\mathcal{W}}$ \eqref{jumpset_obstacleavoidance}. On the other hand, if $\mathbf{0}\in\mathcal{N}_{\gamma}(\mathcal{D}_{r_a}(\mathcal{O}_i))$, it is straightforward to verify that $\mathcal{LS}\subset\mathcal{S}_0$, which, according to \eqref{jumpset_obstacleavoidance}, implies that $\mathcal{LS}\subset\mathcal{J}_1^{\mathcal{W}}$.
This, according to \eqref{oa_jumpset} and \eqref{jumpset_obstacleavoidance}, implies that there exists $t_2\in I_{j_1 +1}$ such that $t_2 < \infty$ and $\xi(t_2, j_1 +1)\in\mathcal{J}_1,$ and claim 2 in Lemma \ref{lemma:always_enter_move_to_target_mode} holds true.

\textbf{Monotonic decrease of the distance $\|{\mathbf{x}}\|$:} Monotonic decrease of $\|{\mathbf{x}}\|$ is trivial in the \textit{move-to-target} mode, thus, we focus on proving the monotonic decrease in the \textit{obstacle-avoidance} mode.

Consider a solution $\xi$ to the hybrid closed-loop system \eqref{modified_hybrid_closed_loop_system}. Let us assume that there exists $(t_1, j_1)\in\text{ dom }\xi$ such that $\xi(t_1, j_1)\in\mathcal{J}_0$. Therefore, according to \eqref{mtt_jumpset}, \eqref{jumpset_movetotarget}, and \eqref{update_law}, one has $\mathbf{h}(t_1, j_1 +1)\in\mathcal{J}_0^{\mathcal{W}}\cap\partial\mathcal{D}_{r_a + \beta}(\mathcal{O}_i)$, for some $i\in\mathbb{I}\setminus\{0\}$ and $\beta\in[0, \gamma_s]$, and $\mathbf{a}(t_1, j_1 +1) \in \mathbf{A}_s(\mathbf{h}(t_1, j_1 +1))$. According to \eqref{modified_hybrid_closed_loop_system}, $\mathbf{h}(t, j) = \mathbf{h}(t_1, j_1 + 1)$ and $\mathbf{a}(t, j) = \mathbf{a}(t_1, j_1 + 1)$ for all $(t, j)\in(I_{j_1 +1}\times j_1 + 1)$. Let $\mathbf{h} = \mathbf{h}(t_1, j_1 +1) = \mathbf{h}(t, j)$ and $\mathbf{a} = \mathbf{a}(t_1, j_1 +1) = \mathbf{a}(t, j)$ for all $(t, j)\in(I_{j_1 +1}\times j_1 +1)$. 
According to Lemma \ref{lemma:always_enter_move_to_target_mode}, under the control input $\mathbf{u}_s(\xi(t, j))$,  the state $\mathbf{x}(t, j)$ belongs to the set $\mathcal{N}_{\gamma}(\mathcal{D}_{r_a}(\mathcal{O}_i))\cap\mathcal{P}(\mathbf{h}, \mathbf{a})$ for all $(t, j)\in(I_{j_1+1}\times j_1 +1).$ Moreover, $\xi(t_2, j_1 +1)\in\mathcal{J}_1$
, where $t_2 = \underset{t\in I_{j_1 +1}}{\sup} t.$

If one shows that for all $(t, j)\in([t_1, t_2]\times j_1 +1), \frac{d}{dt}(\frac{1}{2}\mathbf{x}(t, j)^\top\mathbf{x}(t, j))\leq 0$, then it will imply that the control input vector $\mathbf{u}_s(\xi(t, j))$ guarantees a monotonic decrease in the distance $\|\mathbf{x}\|$ as the solution $\xi(t, j_1+1)$ flows during the interval $I_{j_1 +1}.$ Notice that for all $(t, j)\in(I_{j_1 + 1}\times j_1 + 1)$, one has $\frac{d}{dt}(\frac{1}{2}\mathbf{x}(t, j)^\top\mathbf{x}(t, j)) = \mathbf{x}(t, j)^\top\mathbf{u}_s(\xi(t, j))$, where, for $m = 1$, $\mathbf{u}_s(\xi(t, j)) = \kappa_r \mathbf{v}_s(\mathbf{x}(t, j), \mathbf{h}, \mathbf{a})$. Therefore, we proceed to show that $\mathbf{x}(t, j)^\top\mathbf{v}_s(\mathbf{x}(t, j), \mathbf{h}, \mathbf{a})\leq 0$ for all $(t, j)\in(I_{j_1 + 1}\times j_1 + 1)$.

Let us divide the set $\mathcal{N}_{\gamma}(\mathcal{D}_{r_a}(\mathcal{O}_i))\cap\mathcal{P}(\mathbf{h}, \mathbf{a})$ into two mutually exclusive subsets as follows:
\begin{equation}
    \mathcal{N}_{\gamma}(\mathcal{D}_{r_a}(\mathcal{O}_i))\cap\mathcal{P}(\mathbf{h}, \mathbf{a})= \mathcal{U}_1 \cup \mathcal{U}_2,
\end{equation}
where
\begin{equation}
\begin{aligned}
    \mathcal{U}_1&:= \mathcal{N}_{\gamma}(\mathcal{D}_{r_a}(\mathcal{O}_i))\cap\mathcal{P}(\mathbf{h}, \mathbf{a})\cap\mathcal{H}_{\geq}(\mathbf{0}, \mathbf{R}(\hat{\mathbf{h}}, \mathbf{a})\mathbf{c}_i),\\
    \mathcal{U}_2&:= \mathcal{N}_{\gamma}(\mathcal{D}_{r_a}(\mathcal{O}_i))\cap\mathcal{P}(\mathbf{h}, \mathbf{a})\cap\mathcal{H}_{<}(\mathbf{0}, \mathbf{R}(\hat{\mathbf{h}}, \mathbf{a})\mathbf{c}_i),\label{partition_for_monotonicity}
\end{aligned}
\end{equation}
with $\mathbf{c}_i$ being the center of obstacle $\mathcal{O}_i.$

\begin{figure}
    \centering
    \includegraphics[width=1\linewidth]{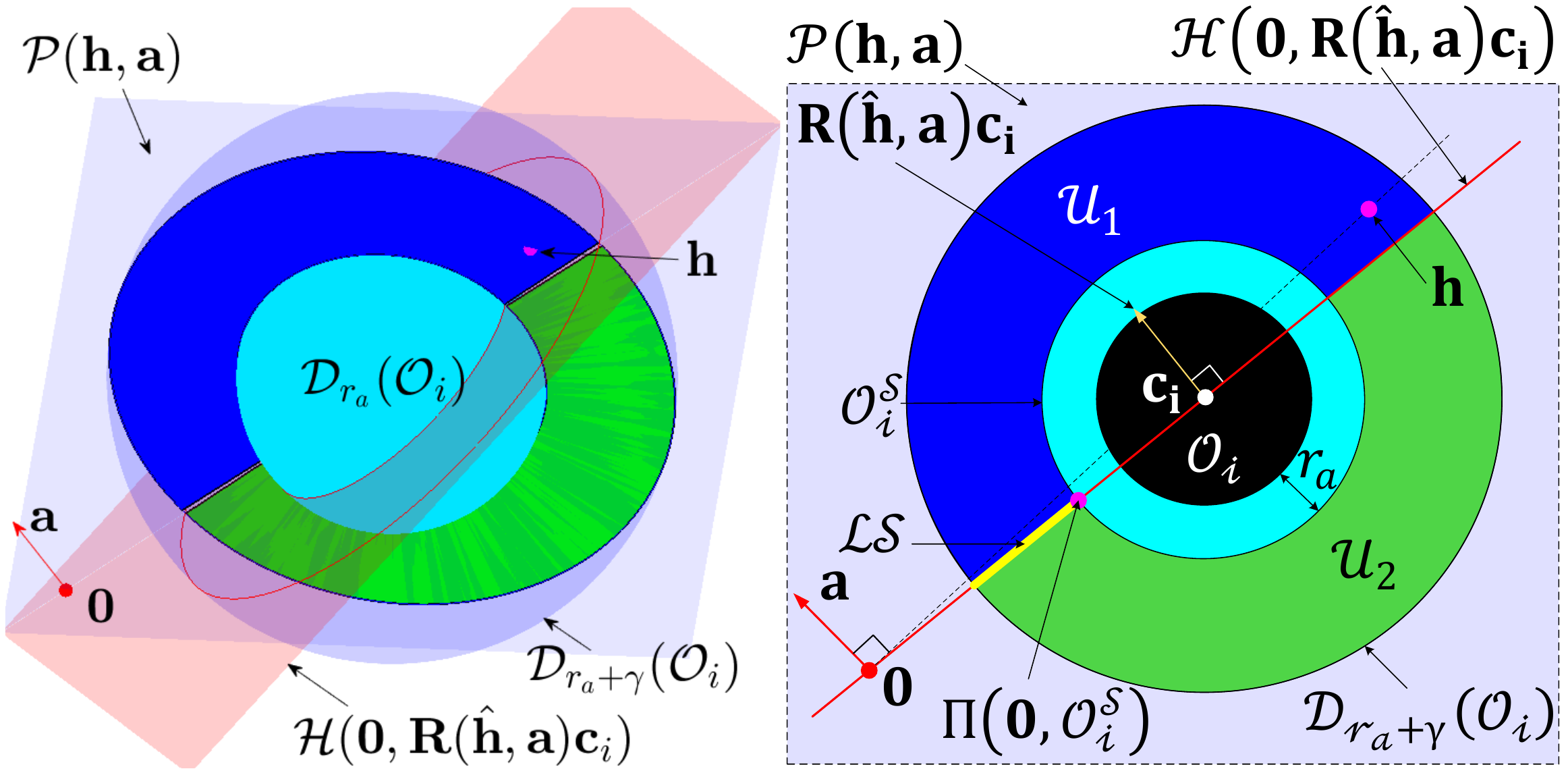}
    \caption{Geometric representation of the sets $\mathcal{U}_1$ and $\mathcal{U}_2$, defined in \eqref{partition_for_monotonicity}, for a three-dimensional spherical obstacle $\mathcal{O}_i$.}
    \label{sphere_proof_final_image}
\end{figure}

Since $\mathbf{a}$ and $\mathbf{h}$ are chosen as per \eqref{update_law}, and obstacle $\mathcal{O}_i$ is a sphere, the plane $\mathcal{P}(\mathbf{h}, \mathbf{a})$ intersects both the center $\mathbf{c}_i$ of obstacle $\mathcal{O}_i$ and the target location at the origin, as shown in Fig. \ref{sphere_proof_final_image}. As a result, for all $\mathbf{x}\in\mathcal{N}_{\gamma}(\mathcal{D}_{r_a}(\mathcal{O}_i))\cap\mathcal{P}(\mathbf{h}, \mathbf{a})$, one has $\mathbf{x}_{\pi}\in\mathcal{P}(\mathbf{h}, \mathbf{a})$ and $\mathbf{R}(\hat{\mathbf{h}}, \mathbf{a})\mathbf{x}_{\pi}\in\mathcal{P}(\mathbf{h}, \mathbf{a})$. 
Moreover, since obstacle $\mathcal{O}_i$ is a sphere, for all $\mathbf{x}\in\mathcal{U}_1$, it is true that $\mathbf{x}_{\pi} \in\mathcal{P}(\mathbf{h}, \mathbf{a})\cap\mathcal{H}_{\geq}(\mathbf{0}, \mathbf{R}(\hat{\mathbf{h}}, \mathbf{a})\mathbf{x})$. 
Therefore, for all $\mathbf{x}\in\mathcal{U}_1$, one has $\mathbf{R}(\hat{\mathbf{h}}, \mathbf{a})\mathbf{x}_{\pi}\in\mathcal{P}(\mathbf{h}, \mathbf{a})\cap\mathcal{H}_{\leq}(\mathbf{0}, \mathbf{x})$, as illustrated in Fig. \ref{monotonic_decrease_reason}. 
As a result, for all $\mathbf{x}\in\mathcal{U}_1$, one can conclude that $\mathbf{x}^\top\mathbf{v}_s(\mathbf{x}, \mathbf{h}, \mathbf{a}) \leq 0$.
Now, if one shows that $\mathbf{x}(t, j)\in\mathcal{U}_1$, for all $(t, j)\in (I_{j_1 +1}\times j_1 +1)$, then it implies that $\mathbf{x}(t, j)^\top\mathbf{u}_{s}(\xi(t, j))\leq 0$ for all $(t, j)\in(I_{j_1+1}\times j_1 +1)$ 

\begin{figure}
    \centering
    \includegraphics[width=0.95\linewidth]{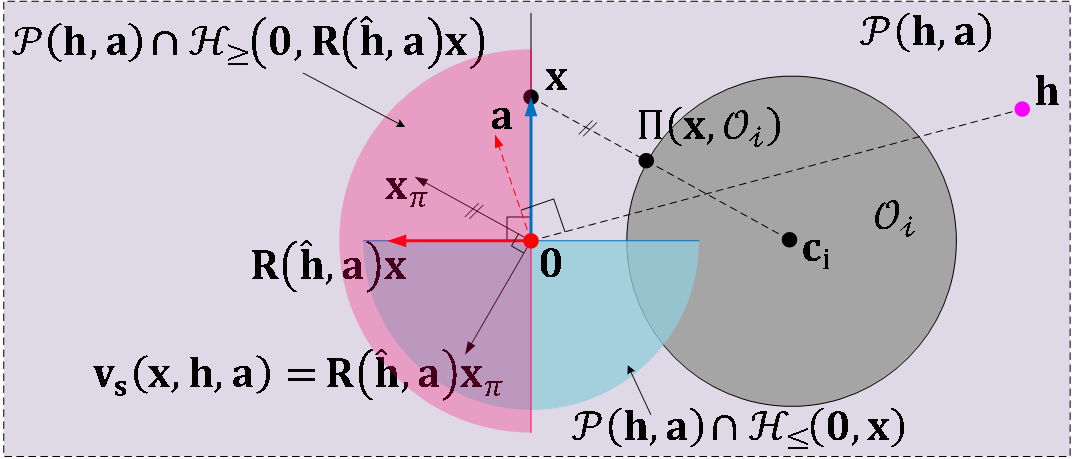}
    \caption{Workspace scenario illustrating $\mathbf{x}_{\pi}\in\mathcal{P}(\mathbf{h}, \mathbf{a})\cap\mathcal{H}_{\geq}(\mathbf{0}, \mathbf{R}(\hat{\mathbf{h}}, \mathbf{a})\mathbf{x})\implies \mathbf{v}_s(\mathbf{x}, \mathbf{h}, \mathbf{a})\in\mathcal{P}(\mathbf{h}, \mathbf{a})\cap\mathcal{H}_{\leq}(\mathbf{0}, \mathbf{x})$.}
    \label{monotonic_decrease_reason}
\end{figure}

Note that, when $\xi(t_1, j_1)\in\mathcal{J}_0$ with $\mathbf{x}(t_1, j_1 +1)\in\mathcal{J}_{0}^{\mathcal{W}}\cap\mathcal{N}_{\gamma}(\mathcal{D}_{r_a}(\mathcal{O}_i))$, the choice of the unit vector $\mathbf{a}$, as per \eqref{update_law_for_vector_a}, ensures that $\mathbf{x}(t_1, j_1 +1)\in\mathcal{U}_1.$
We know that, under the control input $\mathbf{u}_s(\xi(t, j))$,  the state $\mathbf{x}(t, j)$ belongs to the set $\mathcal{N}_{\gamma}(\mathcal{D}_{r_a}(\mathcal{O}_i))\cap\mathcal{P}(\mathbf{h}, \mathbf{a})$ for all $(t, j)\in(I_{j_1+1}\times j_1 +1)$, where $\mathbf{h}\in\mathcal{J}_0^{\mathcal{W}}\cap\partial\mathcal{D}_{r_a + \beta}(\mathcal{O}_i),$ for some $\beta\in[0, \gamma_s]$.
Since $\mathbf{x}(t_1, j_1 +1)\in\mathcal{U}_1$ and $\mathbf{v}_s(\mathbf{x}, \mathbf{h}, \mathbf{a})^\top\mathbf{x}\leq 0, \forall \mathbf{x}\in\mathcal{U}_1$, under the control input $\mathbf{u}_s(\xi) = \kappa_r\mathbf{v}_s(\mathbf{x}, \mathbf{h}, \mathbf{a})$, $\mathbf{x}$ can enter in the set $\mathcal{U}_2$ only from the set $\mathcal{LS}$, where $\mathcal{LS}$ is defined in \eqref{proof_line_segment} and is depicted in Fig. \ref{sphere_proof_final_image}.  
Therefore, there exists $(t_2, j_1 +1)\succ(t_1, j_1 +1)$ such that $\mathbf{x}(t_2, j_1 +1) = \mathcal{LS}\subset\mathcal{U}_1$. Moreover, as proved earlier, for a sufficiently small value of $\bar{\epsilon}$, used in \eqref{partition_rm}, one can guarantee that $\mathcal{LS}\subset\mathcal{J}_1^{\mathcal{W}}$, where the set $\mathcal{J}_1^{\mathcal{W}}$ is defined in \eqref{jumpset_obstacleavoidance}. Hence, at $(t_2, j_1 +1)$, one has $\xi(t_2, j_1 +1)\in\mathcal{J}_1$, which implies that $t_2 = \underset{t\in I_{j_1 +1}}{\sup} t.$ As a result, for all time $(t, j)\in( I_{j_1 +1}\times j_1 +1 )$, one has $\mathbf{x}(t, j)\in\mathcal{U}_1$, and the proof is complete.

\end{appendix}

\bibliographystyle{IEEEtran}
\bibliography{reference}

\end{document}